\documentclass{article}

\usepackage[margin=1in]{geometry}
\usepackage[T1]{fontenc}    
\usepackage{hyperref}       
\usepackage{url}            
\usepackage{booktabs}       
\usepackage{amsfonts}       
\usepackage{nicefrac}       
\usepackage{microtype}      
\usepackage{xcolor}         
\usepackage{graphicx}
\usepackage{wrapfig}
\usepackage{capt-of}
\usepackage{subcaption}

\usepackage{multirow}           
\usepackage{duckuments}         

\usepackage{amsmath}
\usepackage{mathtools}
\usepackage{enumitem}
\usepackage{amsthm}
\usepackage{pifont}
\usepackage[linesnumbered,vlined,ruled,commentsnumbered]{algorithm2e}
\usepackage{tcolorbox}
\usepackage{tikz}
\usetikzlibrary{automata, positioning}
\usepackage{wrapfig}
\usepackage[export]{adjustbox}
\usepackage[square, numbers]{natbib}

\usepackage{thmtools}
\usepackage{thm-restate}

\usepackage{pifont}
\newcommand{\cmark}{\ding{51}}%
\newcommand{\xmark}{\ding{55}}%
\newcommand\xtw[1]{\textcolor{blue}{[xw: #1]}}

\newtheorem{theorem}{Theorem}[section]
\newtheorem{proposition}[theorem]{Proposition}
\newtheorem{definition}[theorem]{Definition}
\newtheorem{assumption}[theorem]{Assumption}
\newtheorem{lemma}[theorem]{Lemma}
\newtheorem{property}[theorem]{Property}

\usepackage{float}

\newtheorem{example}{Example}[section]
\newtheorem{remark}{Remark}[section]
\newcommand{\yuhong}[1]{{\textcolor{red}{Y: #1}}}

\newcommand{\proj}{WALLA\xspace}
\newcommand{\PP}{\mathcal{P}}
\newcommand{\Q}{\mathcal{Q}}

\title{Decentralized Aggregation of LLM Predictions  \\ via Wagering Mechanisms}

\author{Yuhong Luo \thanks{Rutgers University, \texttt{y.luo@rutgers.edu}} \and David M. Pennock \thanks{DIMACS, Rutgers University, \texttt{dpennock@dimacs.rutgers.edu}} \and Xintong Wang \thanks{Rutgers University, \texttt{xintong.wang@rutgers.edu}}}

\begin{document}

\maketitle
\begin{abstract}
It is increasingly common to aggregate predictions from multiple LLMs, each with domain expertise or access to private tools and data, to improve collective prediction performance. 
In decentralized settings, aggregation weights need to be determined without access to models' private information and should remain robust to strategic reporting. 
We propose a family of \textit{advantage-aligned wagering mechanisms for LLM aggregation} (WALLA), in which each model reports a prediction and a learned wager, and predictions are aggregated using wagers as weights.
WALLA introduces a leave-one-out baseline into the net payout function, yielding three desirable properties: (1) dominant-strategy incentive compatibility of prediction under arbitrary belief structure, (2) advantage--wager alignment, where the optimal wager is proportional to the model's expected score advantage, and (3) prediction-agnostic wager optimization, enabling decentralized learning of wager policies without requiring optimal predictions. 
We further instantiate two mechanism variants that trade off normality and no-arbitrage while maintaining a bounded worst-case deficit for the mechanism.
Experiments on question-answering and forecasting benchmarks across heterogeneous models and private-information settings show that WALLA matches centralized aggregation methods in predictive performance, while simultaneously achieving decentralized learning, advantage-aligned aggregation weights, uncertainty awareness, and incentive-compatible prediction.\footnote{The source code is provided in \url{https://github.com/chailab-rutgers/WALLA}.}

\end{abstract}

\section{Introduction}
Having different large language models---each with distinct capabilities, domain expertise, or access to private tools and data---collaborating on complex tasks is becoming increasingly common. 
When no single model dominates across domains, a natural objective is to combine their complementary strengths.
For prediction tasks such as forecasting and question answering, this amounts to aggregating their probabilistic predictions.

A central challenge in aggregating LLM predictions is determining how much influence or weight each model should have on each question. 
In centralized systems, question-dependent aggregation weights can be learned directly from models' hidden representations or private contexts~\citep{lu2023routing, chen2024routerdc, ong2025routellm, song2025irt, WOLPERT1992241}. 
In decentralized deployments, however, models' private signals are unavailable to the aggregator. 
Existing decentralized approaches either assign fixed weights~\citep{wang2022selfconsistency, du2023improving, choi2025debate} or derive weights from reported confidence~\citep{kang2025scalable} or input perplexity~\citep{mavromatis2024pack}.  
However, aggregation weights based on these heuristic signals may be unreliable when models are miscalibrated and may also be strategically exaggerated by independent LLM services seeking greater influence or rewards. 
We propose designing decentralized mechanisms whose incentives align aggregation weights with models' expected contribution to the aggregated prediction. Ideally, for each domain-specific question, models that expect to outperform the rest of the pool should participate and receive greater aggregation weights, while models that expect to be outperformed should receive little or no weight in the aggregation.

We propose to address this challenge
through \textit{wagering
mechanisms}~\citep{Lambert2008Self,LAMBERT2015axiom,Chen2014Removing,Freeman2017Double}, in which each participant reports a prediction and places a wager---a stake reflecting its confidence in that prediction. 
The mechanism then evaluates predictions using proper scoring rules and allocates payouts based on these scores relative to others, with rewards scaled by the size of the wagers. 
By interpreting wagers as aggregation weights, the mechanism simultaneously incentivizes truthful predictions and elicits how much influence each model should have in the final aggregated prediction. 




The canonical weighted-score wagering mechanism (WSWM)~\citep{LAMBERT2015axiom} satisfies key incentive properties, but two aspects limit its applicability to learning-based agents. 
First, its incentive-compatibility guarantee relies on \textit{immutable beliefs}---the assumption that agents would not revise their posteriors upon observing others' reports. 
Second, because of the form of WSWM's net payout, a rational agent's optimal wager always lies at the boundaries, producing a binary signal that is insufficiently informative for weighted aggregation or for learning wagering strategies that reflect an agent's comparative advantage on domain-specific questions.


\paragraph{Our contribution.}
We propose a family of wagering mechanisms referred to as \textit{advantage-aligned wagering mechanisms for LLM aggregation} (\proj) that resolve both limitations through a single structural modification to the net payout function.
This modification yields three properties that make the mechanism well-suited to learning-based agents:
(1)~\textit{dominant-strategy incentive compatibility under arbitrary belief structures}: truthful prediction remains optimal regardless of opponents' strategies or how agents' beliefs are formed; (2)~\textit{advantage--wager alignment}: the best-response wager is proportional to the agent's expected score advantage; and (3)~\textit{prediction-agnostic wager optimization}: for any fixed prediction, the best-response wager remains well defined, enabling agents to learn wagering strategies without requiring optimal predictions.
Unlike heuristic confidence measures, the equilibrium wager has a precise interpretation as an agent's expected score advantage, arising as the equilibrium outcome of a net payout function that aligns wagers with comparative advantage.

We further instantiate the proposed net payout function with two variants,
where the first satisfies \textit{normality} but admits \textit{arbitrage}~\citep{Chen2014Removing}, while the second sacrifices normality but guarantees no-arbitrage. 
We prove that the mechanism's worst-case deficit is a tunable design choice independent of both the number of participants and the amount wagered.



Empirically, we conduct experiments on decentralized prediction aggregation in which LLM agents use their model hidden states to learn a wager network.
We evaluate \proj across three scenarios: homogeneous models with private contexts, heterogeneous models, and heterogeneous models with private contexts, demonstrating that the wager network learns to identify comparative advantage from diverse sources: context assignment, domain expertise, and their combination.
\proj matches centralized aggregation baselines (stacked generalization, learned routers) on QA benchmarks (MMLU, MedMCQA, ARC-Challenge, PubMedQA) and a forecasting benchmark (BayesX), while being the only method that simultaneously achieves weighted aggregation, uncertainty awareness, and fully decentralized learning with incentive-compatibility guarantees.

\subsection{Related Work}\label{sec:related_work}

\paragraph{Aggregating LLM predictions.} 
The aggregation of LLM predictions has been studied extensively~\citep{chen2025harnessing}. 
Equal-weighted aggregation, such as self-consistency~\citep{wang2022selfconsistency}, majority vote~\citep{li2024more,yu-etal-2024-breaking, choi2025debate}, or multi-agent debate~\citep{du2023improving} is shown to be effective. 
However, these methods implicitly assume that models have comparable capabilities across tasks, an assumption that may fail when tasks require specialized knowledge. 
%
Recent work shows that LLM hidden representations encode uncertainty and confidence~\citep{kadavath2022language, azaria2023internal, Xiao_Dou_Xiong_Chen_Chen_2026}, motivating training-free methods that estimate aggregation weights from model uncertainty~\citep{mavromatis2024pack,kang2025scalable}. 
While computationally efficient, these methods rely on accurate confidence estimates and can be vulnerable to miscalibration or strategic manipulation.

More general approaches learn aggregation weights in a centralized supervised manner. 
A stream of work studies pre-inference routing, where a centrally trained routing model selects which LLM should answer a query, without accessing model predictions or hidden representations at inference time~\citep{lu2023routing,shnitzer2024large, chen2024routerdc, ong2025routellm, song2025irt}. 
Another line adopts stacked generalization~\citep{WOLPERT1992241}, which learns aggregation weights from model hidden representations.
Finally, several methods train external reward models to evaluate generated outputs and determine aggregation weights~\citep{uesato2022solvingmathwordproblems, jiang2023llm,lightman2023letsverifystepstep}, which can be computationally expensive. 
In contrast, our method supports fully decentralized learning, where aggregation weights emerge as equilibrium wagers with a precise interpretation as agents' expected score advantages.

A series of recent works studies mechanism designs for LLM aggregation~\citep{dutting2024Mechanism, Dubey2024Auctions, Bergemann2025Data, soumalias2026truthful}. Rather than improving the aggregated prediction performance, these mechanisms primarily elicit truthful reports of preferences (e.g., bids for tokens) for applications such as online advertising. 
Our work here uses mechanism design to elicit aggregation weights (reflecting prediction advantage) for decentralized prediction aggregation.

\paragraph{LLM forecasting.} 
Forecasting has emerged as an important benchmark for evaluating LLMs' world knowledge and uncertainty-estimation capabilities~\citep{chenghaozhu-etal-2025-llm,dai2025are}. 
Existing work has focused on improving the forecasting performance of language-model-based forecasting systems~\citep{ZouForecasting2022, halawi2024approaching} and on designing benchmarks for LLM forecasters~\citep{jin-etal-2021-forecastqa, zhang-etal-2024-analyzing, karger2025forecastbench, wang-etal-2025-openforecast, zeng2025futurexadvancedlivebenchmark, Guan2026open,yang2026llmasaprophet}. \citet{Schoenegger2024Wisdom} showed that aggregating forecasts from multiple LLMs can rival the accuracy of human crowds. 


\paragraph{Wagering mechanisms and prediction markets.}
Wagering mechanisms are designed to elicit probabilistic predictions from strategic forecasters~\citep{Lambert2008Self,Chen2014Removing, LAMBERT2015axiom, Freeman2020no, RAJA2024142}. 
These mechanisms are proposed to satisfy desirable properties, including incentive compatibility, budget balance, individual rationality, normality, and no-arbitrage.  Our work builds on this literature by designing wagering mechanisms for decentralized aggregation of LLM predictions.

A closely related line of work studies prediction markets~\citep{Abernethy2011Collaborative,Beygelzimer2012Learning,Frongillo2012Interpreting, Abernethy2013Efficient,Hu2014Multi}, which aggregate beliefs through market prices in dynamic environments where participants continuously update their positions as new information arrives. In contrast, we consider a static information setting in which each agent reports a prediction and a wager only once.


\section{Preliminaries}
This section introduces preliminaries underlying our proposed mechanism and analysis.
We first review proper scoring rules and the canonical \textit{weighted-score wagering mechanism} on which we build, then discuss desirable properties, and motivate additional ones for extending the framework to LLM agents.\looseness=-1

%
\paragraph{Scoring rules and wagering mechanisms.}
We consider a probabilistic prediction setting, where a mechanism elicits predictions on uncertain outcomes from multiple agents and evaluates them based on the quality of their predictions.
Let $\mathcal{M} = \{1, 2, \dots, M\}$ denote a set of $M$ agents. 
A \textit{wagering mechanism} operates in two stages.
In the first stage, given a question $x \in \mathcal{X}$ with an unknown outcome $y \in \mathcal{Y}$, each agent $i$ submits a prediction $p_i \in \Delta(\mathcal{Y})$ and wagers some non-negative amount $w_i \in \mathbb{R}_{\ge 0}$ on her prediction.
We denote the resulting \textit{prediction} and \textit{wager profiles} by $\mathbf{p} \coloneqq (p_1, \ldots, p_M)$ and $\mathbf{w} \coloneqq (w_1, \ldots, w_M)$, and write $\mathbf{p}_{-i}$ and $\mathbf{w}_{-i}$ for the corresponding profiles excluding agent $i$. 
In the second stage, after the outcome is revealed, the mechanism determines the \textit{net payout} to each agent, denoted $\pi_i(\mathbf{p}, \mathbf{w}, y) \in \mathbb{R}$, based on the quality of her prediction and the amount wagered.

The quality of a prediction is measured by \textit{a scoring rule}: $s \colon \Delta(\mathcal{Y}) \times \mathcal{Y} \to \mathbb{R}$, which assigns a numerical score $s(p, y)$ to a probabilistic forecast $p$ 
upon realization of outcome $y$.
A \textit{strictly proper scoring rule} incentivizes a \textit{risk-neutral} agent to predict according to her true belief.
%
\begin{definition}[Strictly Proper Scoring Rule]
A scoring rule $s$ is \emph{strictly proper} if and only if for all $p,q \in \Delta(\mathcal{Y})$ and $p \neq q$,
\[
  \mathbb{E}_{Y \sim q}[s(q, Y)] 
  >
  \mathbb{E}_{Y \sim q}[s(p, Y)]. 
\]
A prediction of $p = q$ is \emph{truthful}, as the agent simply reports her belief $q$.
\end{definition}
\begin{example} [Brier Score~\citep{Brier1950,Gneiting01032007}]\label{def:brier}
$s(p,y)=c_1-c_2\|p-\delta_y\|_2^2$, where $c_1 \in \mathbb{R}$, $c_2 \in \mathbb{R}_{+}$, and $\delta_y\in\Delta(\mathcal{Y})$ denotes the Dirac measure at $y$, the indicator vector of the realized outcome $y$.
\end{example}
The Brier score, also known as the quadratic score, is a simple and
widely used strictly proper scoring rule; it is bounded and has the property that for an agent with belief $q$, the expected score $\mathbb{E}_{Y \sim q}[s(p, Y)]$ is monotonically decreasing in $\|p-q\|_2^2$, the squared distance between the prediction $p$ and the belief $q$.
%
A scoring rule is then used in the mechanism's payout function $\pi: (\Delta(\mathcal{Y}) \times \mathbb{R}_{\ge 0})^M \times \mathcal{Y} \rightarrow \mathbb{R}^M$.
\begin{definition}[Weighted-Score Wagering Mechanism (WSWM)~\citep{LAMBERT2015axiom}]
Given a scoring rule $s$, the WSWM specifies the net payout to agent $i$ as the difference between her score and the wager-weighted average of all agents' scores, scaled by her own wager, i.e.,
\begin{equation}
    \small
    \label{eq:wswm_payout}
    \pi^{\text{WSWM}}_i(\mathbf{p}, \mathbf{w}, y)
    = w_i\left(s(p_i,y) - \frac{\sum_{j=1}^M w_j s(p_j, y)}{\sum_{j=1}^M w_j}\right).
\end{equation}
\end{definition}
The WSWM based on a Brier score is also referred to as the \textit{Brier betting} mechanism.
%

\paragraph{Desirable properties of wagering mechanisms.}
Following~\citet{LAMBERT2015axiom} and~\citet{Chen2014Removing}, we introduce a set of properties that wagering mechanism designers typically aim to satisfy and that are particularly relevant in our setting: (a) \textit{incentive compatibility} (Def.~\ref{def:ic}), 
(b) \textit{individual rationality} (Def.~\ref{def:ir}), 
(c) \textit{weak budget balance} (Def.~\ref{def:wbb}), 
(d) \textit{normality} (Def.~\ref{def:normal_old}), and 
(e) \textit{no arbitrage} (Def.~\ref{def:noarb}).

\begin{definition} [Incentive Compatibility (IC)]\label{def:ic}
A wagering mechanism is incentive-compatible if all players maximize their expected net payout when predicting their true beliefs. Formally, $\forall i\in \mathcal{M}$, $q_i\in  \Delta(\mathcal{Y})$ which reflects agent $i$'s true belief, $p_i\in  \Delta(\mathcal{Y})$ where $q_i \ne p_i$, $\mathbf{p}_{-i} \in  \Delta(\mathcal{Y})^{M-1}$, and $\mathbf{w} \in \mathbb{R}_{\ge 0}^M$ with $w_i > 0$,
\[\mathbb{E}_{Y\sim q_i}[\pi_i((q_i, \mathbf{p}_{-i}), \mathbf{w}, Y)] > \mathbb{E}_{Y\sim q_i}[\pi_i((p_i, \mathbf{p}_{-i}), \mathbf{w}, Y)].\] 
The mechanism is \textit{dominant-strategy incentive-compatible (DSIC)} if truthful prediction is optimal regardless of other agents' reports, for any $w_i > 0$.
\end{definition}

\begin{definition} [Individual Rationality (IR)]\label{def:ir} 
Every agent prefers participating to not participating in the mechanism (i.e., the expected net payoff of participating is nonnegative). 
Formally, $\forall i\in \mathcal{M}$, $q_i\in  \Delta(\mathcal{Y})$ which reflects agent $i$'s true belief and $w_i \in \mathbb{R}_{\ge0}$, there exists some $p_i \in  \Delta(\mathcal{Y})$ such that for all $\mathbf{p}_{-i} \in  \Delta(\mathcal{Y})^{M-1}$ and 
$\mathbf{w}_{-i} \in \mathbb{R}_{\ge 0}^{M-1}$,%
\[\mathbb{E}_{Y\sim q_i}[\pi_i((p_i, \mathbf{p}_{-i}), (w_i, \mathbf{w}_{-i}), Y)] \ge 0.\]
\end{definition}

 \begin{definition} [Weak Budget Balance]\label{def:wbb} The mechanism does not lose money but can potentially make a profit. Formally, $\forall \mathbf{p}\in \Delta(\mathcal{Y})^M, \mathbf{w} \in \mathbb{R}_{\ge 0}^M $ and $y\in \mathcal{Y}$,
\[\sum_{i=1}^M \pi_i (\mathbf{p}, \mathbf{w}, y) \le 0.\]
\end{definition}

\begin{definition}[Normality]\label{def:normal_old} 
    If from agent $i$'s perspective the prediction of another agent $j$ improves, then agent $i$'s expected net payoff decreases. Formally, for every $i \ne j \in \mathcal{M}$, $q_i\in \Delta (\mathcal{Y})$ which reflects agent $i$'s true belief, $\mathbf{w}\in \mathbb{R}^M_{\ge 0}$ and $\mathbf{p} \in \Delta(\mathcal{Y})^M$, defining $\mathbf{p}'$ by $p'_j= q_i$ and $p'_k = p_k$ for $k \ne j$,
    \[\mathbb{E}_{Y\sim q_i}[\pi_i(\mathbf{p}, \mathbf{w}, Y)] \ge \mathbb{E}_{Y\sim q_i}[\pi_i(\mathbf{p}', \mathbf{w}, Y)],\]
    with strict inequality whenever $p_j\ne q_i$, $w_i > 0$ and $w_j > 0$.
\end{definition}

\begin{definition}[No-Arbitrage~\citep{Chen2014Removing}]
\label{def:noarb}
 No agent can guarantee a strictly positive payout regardless of the realized outcome.
 Formally, a mechanism is arbitrage-free if there exists no agent $i \in \mathcal{M}$, prediction profile $\mathbf{p}\in \Delta(\mathcal{Y})^M$, and $\mathbf{w}\in \mathbb{R}_{\geq 0}^M$ such that $\pi_i(\mathbf{p}, \mathbf{w}, y) > 0$ for all $y \in \mathcal{Y}$.
\end{definition}

WSWM is well-suited for human forecasters, satisfying the above properties except for no-arbitrage, but can fall short for learning-based agents. 
First, a rational agent's best-response wager is extreme: zero when she expects to underperform the pool average, and as large as permitted otherwise.
It yields only coarse, binary feedback for weighted aggregation and learning wagering strategies. 
Second, its IC relies on immutable beliefs~\citep{LAMBERT2015axiom}; when agents update beliefs after observing others’ reports, truthful prediction may no longer be optimal. 
We aim to close both gaps by designing a mechanism suitable for both human and learning-based participants (e.g., LLMs): the best-response wager reflects expected score advantage, and DSIC holds under an arbitrary belief structure.

\paragraph{Warm-up: WSWM with non-strategic agents.}
\label{sec:warmup}
%
We analyze the wagering mechanism in a simplified, non-strategic setting where agents report truthful predictions and reinvest their cumulative payoffs as wagers. 
We consider a variant of WSWM with a scaled linear scoring rule and linear pooling, and show that the resulting dynamics admit connections to gradient-based learning, including equivalence to multiplicative weights updates, and that the induced aggregation recovers Bayesian model averaging~\citep{hoeting1999bayesian} with improved regret guarantees over standard WSWM formulations~\citep{Freeman2020no}. 
The formal statements and proofs are deferred to Appendix~\ref{apx:warmup}.
\section{Design of Wagering Mechanisms for LLM Aggregation}\label{sec:mechanism}
%
This section presents our proposed \textit{advantage-aligned wagering mechanism for LLM aggregation (WALLA)}.
We start by analyzing our mechanism under the \emph{one-shot, single-question} setting (Sections~\ref{sec:oneshot_setting}--\ref{sec:our_mech}), where each agent holds a belief and acts to maximize her expected net payoff.
We characterize the equilibrium prediction and wager, and establish incentive and structural properties of the mechanism. 
Section~\ref{sec:learn_wager} then extends to the \textit{multiple-question} setting, formalizing how LLM agents can learn wager policies from observed payouts as a function of their hidden states, which encode the question and any context.

%
\subsection{Setting, Information Model, and Strategies}
\label{sec:oneshot_setting}
We now formalize the setting, information model, and agent strategy.
Let $(\Omega, \mathcal{F})$ be a measurable space, where $\Omega$ is the set of possible states of the world and $\mathcal{F}$ is the $\sigma$-algebra of measurable events.
Each state $\omega \in \Omega$ encodes the question $X$, the outcome $Y \in \mathcal{Y}$, and all agents' private signals $(S_1, \ldots, S_M)$. 
The private signal $S_i$ represents additional information available to agent $i$ beyond the question itself, 
such as outputs from tool calls or extra context. 
The information available to agent $i$ is captured by $\mathcal{F}_i = \sigma(X, S_i) \subseteq \mathcal{F}$: all agents observe the question $X$, but may differ in their private signals $S_i$.
%
Each agent $i$ holds a prior belief $\Q_i$, a probability measure on $(\Omega, \mathcal{F})$ that may differ from the data-generating probability measure $\PP$.
When agents share a common prior, we write $\Q_i = \Q$ for all $i$.

A strategy for agent $i$ is a pair $(p_i, w_i)$, where $p_i\colon \Omega \to \Delta(\mathcal{Y})$ is a probabilistic prediction and $w_i\colon \Omega \to \mathbb{R}_{\geq 0}$ is a wager. 
Both are required to be $\mathcal{F}_i$-measurable: they depend on $\omega$ only through the question $X$ and agent $i$'s private signal $S_i$.%
\footnote{For LLMs, $\mathcal{F}_i$-measurability is automatic: for fixed model weights (e.g., during inference), the model's output depends only on input $(X, S_i)$.}
All agents are assumed \textit{risk-neutral}, optimizing expected net payout on each individual question.
For notational simplicity, we suppress the explicit dependence on $\omega \in \Omega$ where the context is clear, using $p_i$, $w_i$, and $y$ to denote the prediction $p_i(\omega)$, wager $w_i(\omega)$, and realized outcome $Y(\omega)$ respectively.

%

\subsection{Advantage-Aligned Wagering Mechanisms}
\label{sec:our_mech}

We discuss three components of our design of mechanisms: (1)~a scoring rule $s$ that evaluates individual prediction, (2)~a net payout function that determines each agent's \textit{net} payoff, and (3)~an aggregation rule that combines predictions into a single forecast.
The net payout function is the main object of design.\looseness=-1

%

\begin{definition}[Net Payout Function]\label{def:mechanism}
Given a strictly proper scoring rule $s$ and a constant $c_3 > 0$, the \textit{net} payout to agent $i$ upon the realization of outcome $y \in \mathcal{Y}$ is
\begin{equation}
\label{eq:ourpayout}
\pi_i(\mathbf{p}, \mathbf{w}, y) := w_i\!\left(s(p_i, y) - b_{-i}(\mathbf{p}_{-i}, \mathbf{w}_{-i}, y) - c_3 w_i\right),
\end{equation}
where $b_{-i}\colon \Delta(\mathcal{Y})^{M-1} \times \mathbb{R}_{\geq 0}^{M-1} \times \mathcal{Y} \to \mathbb{R}$ is a baseline score
: it depends only on other agents' predictions and wagers, not on $(p_i, w_i)$. We refer to $b_{-i}$ as the \emph{leave-one-out baseline} and $s(p_i, y) - b_{-i}$ as the comparative advantage. 
\end{definition}

The specific form of $b_{-i}(\cdot)$ is left free; Section~\ref{sec:variants} instantiates two concrete choices that trade off normality with no-arbitrage.
Since $w_i$ is $\mathcal{F}_i$-measurable, it factors out of $\mathbb{E}[\cdot \mid \mathcal{F}_i]$ under any probability measure.
The expected net payout decomposes as
\begin{equation}\label{eq:expected-payout}
\mathbb{E}_{\Q_i}[\pi_i \mid \mathcal{F}_i] 
= w_i \bigg(\underbrace{\,\mathbb{E}_{\Q_i}[s(p_i, Y) \mid \mathcal{F}_i]}_{\text{self-score}} 
- \underbrace{\,\mathbb{E}_{\Q_i}[b_{-i} \mid \mathcal{F}_i]}_{\text{leave-one-out baseline}}\bigg) 
- \underbrace{c_3\,w_i^2}_{\text{regularization}}.
\end{equation}



We show the other two main properties of this net payout function: the first establishes the optimal prediction for \textit{any} wager; the second establishes that the best-response wager elicits the agent's \emph{expected score advantage}.

\begin{theorem}[DSIC of Prediction]
\label{thm:dsic}
Under the net payout function~\eqref{eq:ourpayout}, truthful prediction is a dominant strategy: each agent maximizes her expected net payoff by reporting her own belief about $Y$, regardless of opponents' strategies and her own wager. 
Formally, for any belief $\Q_i$ held by agent $i$, any wager profile $\mathbf{w} \in \mathbb{R}_{\geq 0}^M$, and any $\mathbf{p}_{-i} \in \Delta(\mathcal{Y})^{M-1}$,
\[
p^*_i := \Q_i(Y = \cdot \mid \mathcal{F}_i) \;\in\; \arg\max_{p_i \in \Delta(\mathcal{Y})} \mathbb{E}_{\Q_i}\!\left[\pi_i\!\left((p_i, \mathbf{p}_{-i}), \mathbf{w}, Y\right) \mid \mathcal{F}_i\right],
\]
and the maximizer is unique when $w_i > 0$. 
\end{theorem}

\begin{proof}
From the decomposition~\eqref{eq:expected-payout}, since $b_{-i}$ does not depend on $p_i$, when $w_i > 0$, the optimization reduces to
\[
p_i^* \in \arg\max_{{p}_i \in \Delta(\mathcal{Y})} \mathbb{E}_{\Q_i}[s({p}_i, Y) \mid \mathcal{F}_i].
\]
By strict properness of $s$, this is uniquely attained at $p^*_i := \mathbb{E}_{\Q_i}[\delta_Y \mid \mathcal{F}_i] = \Q_i(Y = \cdot \mid \mathcal{F}_i)$, the agent's belief about $Y$ given her observation of the question $X$ and private signal $S_i$.
When $w_i = 0$, the payout is identically zero regardless of the prediction.
\end{proof}

Two features deserve emphasis. \textit{First}, our DSIC holds under any belief structure, including mutable beliefs: if an agent revises $\Q_i$ after observing others' reports, the updated posterior remains optimal. This contrasts with WSWM, where the scaling factor $\sum_{j\ne i} w_j/\sum_{j} w_j$ correlates with $Y$ under mutable beliefs, producing an $O(1/M)$ bias~\citep{LAMBERT2015axiom}. \textit{Second}, the optimal prediction does not depend on the agent's own wager ($w_i > 0$), so a learned (or even suboptimal) wager does not distort the prediction incentive. Together with the property that the best-response wager is well defined for any fixed prediction (Theorem~\ref{thm:swa}), this decoupling enables the learning pipeline of Section~\ref{sec:learn_wager}. 

\begin{theorem}[Advantage--Wager Alignment]
\label{thm:swa}
Under the net payout function~\eqref{eq:ourpayout}, the best-response wager is proportional to the agent's expected score advantage.
%
Formally, for any belief $\Q_i$ held by agent $i$, any $\mathbf{p} \in \Delta(\mathcal{Y})^M$, and any $\mathbf{w}_{-i} \in \mathbb{R}_{\geq 0}^{M-1}$,
the best-response wager for agent $i$ is
\begin{equation}
\label{eq:opt-wager}
w_i^*(p_i) = \left(\frac{A_i}{2c_3}\right)^+ ; \quad A_i : = \mathbb{E}_{\Q_i}[s(p_i, Y) - b_{-i} \mid \mathcal{F}_i],
\end{equation}
where $(x)^+ := \max(x, 0)$ and $A_i$ denotes agent $i$'s \emph{expected score advantage} over the baseline.


\begin{proof}
For any fixed $p_i$, $\mathbf{p}_{-i}$, and $\mathbf{w}_{-i}$, the expected net payout~\eqref{eq:expected-payout} as a function of $w_i \in \mathbb{R}_{\ge 0}$ is
$
\mathbb{E}_{\Q_i}[\pi_i \mid \mathcal{F}_i] = w_i\,A_i - c_3\,w_i^2.
$

Since $\frac{\partial^2}{\partial w_i^2}\mathbb{E}_{\Q_i}[\pi_i \mid \mathcal{F}_i] = -2c_3 < 0$, this is strictly concave. The unique unconstrained maximizer is $w_i = A_i/(2c_3)$, and the constrained maximizer over $w_i \geq 0$ is $w_i^*(p_i) = (A_i/(2c_3))^+$. 
\end{proof}

%
\end{theorem}

Three features follow from Theorem~\ref{thm:swa}. \textit{First}, agents with $A_i \leq 0$ wager zero and are automatically excluded from aggregation. \textit{Second}, the best-response wager is well defined for any prediction $p_i$, not just the optimal one, allowing the wagers to be learned separately without altering the base model's predictions (more in Section~\ref{sec:learn_wager}). \textit{Third}, under the Brier score and truthful prediction, $\mathbb{E}_{\Q_i}[s(p_i, Y) \mid \mathcal{F}_i] = c_1 - c_2\,\mathrm{Var}_{\Q_i}(\delta_Y \mid \mathcal{F}_i)$, so higher predictive uncertainty reduces the expected self-score and hence the wager.\footnote{Overconfident models underestimate this variance and may over-wager; calibration on realized outcomes can help (Sec.~\ref{sec:additional_exp}).} 


Furthermore, our mechanism satisfies individual rationality in the interim sense: after observing the question and her private signal, an agent prefers participating to not participating.

\begin{definition}[Interim Individual Rationality]\label{def:iir} The expected net payoff of participating upon receiving the question and private signal is nonnegative. Formally, $\forall i \in \mathcal{M}$, any $ \Q_i$ and $\mathcal{F}_i$, 
there exist some $p_i \in \Delta (\mathcal{Y})$ and $w_i\in \mathbb{R}_{\ge 0}$ such that for all $\mathbf{p}_{-i} \in \Delta(\mathcal{Y})^{M-1}$ and $\mathbf{w}_{-i} \in \mathbb{R}_{\ge 0}^{M-1}$,
\[ \mathbb{E}_{\Q_i}[\pi_i((p_i, \mathbf{p}_{-i}), (w_i, \mathbf{w}_{-i}), Y) \mid \mathcal{F}_i] \ge 0. \]
    
\end{definition}

 Interim IR (Def.~\ref{def:iir}) is satisfied because for arbitrary $p_i$, $w_i^*(p_i)$ guarantees non-negative expected net payout $\mathbb{E}_{\Q_i}[\pi_i \mid \mathcal{F}_i] = A_i w_i^*(p_i)  - c_3(w_i^*(p_i))^2 =   (A^2_i/(4c_3))^+
$.

\paragraph{Equilibrium Characterization.}
Heterogeneous predictions can stem from two sources: differences in information and differences in beliefs. 
We use Theorems~\ref{thm:dsic} and~\ref{thm:swa} to characterize the equilibrium in each extreme case. 
In practice, both sources often appear together. 

\begin{example}[Homogeneous agents with private signals]
\label{ex:homo}
Consider $M$ homogeneous agents sharing a common prior $\Q$ over $(\Omega, \mathcal{F})$. 
Each agent $i$ observes a private signal $S_i$, giving rise to distinct private information $\mathcal{F}_i = \sigma(X, S_i) \neq \mathcal{F}_j$ when $S_i \neq S_j$. 
Agents differ not in their priors but in what they observe: heterogeneous predictions arise from heterogeneous information.
In LLM settings, this corresponds to $M$ copies of the same base model, each receiving a different context for the same question. \looseness=-1

\end{example}
\begin{restatable}[BNE under common prior]{corollary}{bne}
\label{cor:bne}

Suppose agents share a common prior $\Q$ (Example~\ref{ex:homo}). 
Let $p_i^* = \Q(Y = \cdot \mid \mathcal{F}_i)$ (Theorem~\ref{thm:dsic}) and $w_i^*(p_i^*) = (A_i/2c_3)^+$ with $A_i = \mathbb{E}[s(p_i^*, Y) - b_{-i} \mid \mathcal{F}_i]$ (Theorem~\ref{thm:swa}). The strategy profile $\{(p_i^*, w_i^*(p_i^*))\}_{i=1}^M$ is a Bayes--Nash equilibrium: for every agent $i$ and every alternative strategy $(\tilde{p}_i, \tilde{w}_i)$,
\[
\mathbb{E}_{\Q}\!\left[\pi_i\!\left((p_i^*, \mathbf{p}_{-i}^*), (w_i^*(p_i^*), \mathbf{w}_{-i}^*), Y\right) \mid \mathcal{F}_i\right] \;\geq\; \mathbb{E}_{\Q}\!\left[\pi_i\!\left((\tilde{p}_i, \mathbf{p}_{-i}^*), (\tilde{w}_i, \mathbf{w}_{-i}^*), Y\right) \mid \mathcal{F}_i\right].
\]
\end{restatable}

 \begin{proof} See Appendix~\ref{proof:BNE}.
 \end{proof}

\begin{example}[Heterogeneous agents]
\label{ex:hetero}
Consider $M$ agents holding beliefs $\Q_1, \ldots, \Q_M$ that may differ.
No assumptions are made regarding their relationship; the measures are permitted to be mutually singular, perfectly absolutely continuous, or anywhere in between.
Agent $i$ does not know agent $j$'s belief $\Q_j$. 
Heterogeneous predictions arise not from different information but from different perspectives on the same input.
%
In the LLM setting, this corresponds to $M$ distinct models---differing in architecture, training data, or fine-tuning procedure---each observing the same question with no additional context ($S_i = \emptyset$, so $\mathcal{F}_i = \sigma(X) = \mathcal{F}_j$). 

\end{example}

\begin{restatable}[NE under heterogeneous beliefs]{corollary}{nashe}
\label{cor:ne}
Suppose each agent $i$ holds her own belief $\Q_i$ (Example~\ref{ex:hetero}). Let $p_i^* = \Q_i(Y = \cdot \mid \mathcal{F}_i)$ (Theorem~\ref{thm:dsic}) and $w_i^*(p_i^*) = (A_i/2c_3)^+$ with $A_i = \mathbb{E}_{\Q_i}[s(p_i^*, Y) - b_{-i} \mid \mathcal{F}_i]$ (Theorem~\ref{thm:swa}). The strategy profile $\{(p_i^*, w_i^*(p_i^*))\}_{i=1}^M$ is a Nash equilibrium: for every agent $i$ and every alternative strategy $(\tilde{p}_i, \tilde{w}_i)$,
\[
\mathbb{E}_{\Q_i}\!\left[\pi_i\!\left((p_i^*, \mathbf{p}_{-i}^*), (w_i^*(p_i^*), \mathbf{w}_{-i}^*), Y\right) \mid \mathcal{F}_i\right] \;\geq\; \mathbb{E}_{\Q_i}\!\left[\pi_i\!\left((\tilde{p}_i, \mathbf{p}_{-i}^*), (\tilde{w}_i, \mathbf{w}_{-i}^*), Y\right) \mid \mathcal{F}_i\right].
\]\end{restatable}

  \begin{proof} See Appendix~\ref{proof:NE}.
  \end{proof}

\subsubsection{Two Variants of Leave-One-Out Score Baseline}
\label{sec:variants}
 
Theorems~\ref{thm:dsic} and~\ref{thm:swa} hold for any leave-one-out baseline $b_{-i}$ that does not depend on $(p_i, w_i)$.
We first instantiate two concrete choices that make different design tradeoffs, and then show that both variants admit a worst-case deficit bound independent of the number of participants and the amount wagered.
 
\begin{definition}[Weighted-Score Baseline of \proj I]\label{def:variant1}
The baseline is the wager-weighted average of other agents' scores:
\begin{equation}\label{eq:variant1}
b_{-i}^{\mathrm{I}}(\mathbf{p}_{-i}, \mathbf{w}_{-i}, y) := \frac{\sum_{j \neq i} w_j\, s(p_j, y)}{\sum_{j \neq i} w_j}.
\end{equation}
\end{definition}
 
\begin{definition}[Weighted-Prediction Baseline of \proj II]\label{def:variant2}
The baseline evaluates the scoring rule at the wager-weighted pool of other agents' predictions:
\begin{equation}\label{eq:variant2}
b_{-i}^{\mathrm{II}}(\mathbf{p}_{-i}, \mathbf{w}_{-i}, y) := s(q_{-i}, y), \qquad q_{-i} := \frac{\sum_{j \neq i} w_j\, p_j}{\sum_{j \neq i} w_j} \in \Delta(\mathcal{Y}).
\end{equation}
\end{definition}
 
In short, \proj~I averages scores, whereas \proj~II scores an average prediction. 
The two variants trade off normality (Def.~\ref{def:normality}) and no-arbitrage (Def.~\ref{def:noarb}).
 %
These two properties are in tension under known wagering mechanism designs. 
~\citet{LAMBERT2015axiom} showed that WSWM satisfies normality, while \citet{Chen2014Removing} showed it admits arbitrage and proposed the no-arbitrage wagering mechanism (NAWM), which eliminates arbitrage but sacrifices normality. 
Our two variants exhibit the same tradeoff. We first adapt the definition of normality (Def.~\ref{def:normal_old}) for our information model.

\begin{definition}[Normality]
\label{def:normality}
A mechanism satisfies normality if, for every $i \neq j \in \mathcal{M}$, any belief $\Q_i$ over the measurable space $(\Omega, \mathcal{F})$, $\mathcal{F}_i \in \mathcal{F}$, $\mathbf{w} \in \mathbb{R}_{\geq 0}^M$ and $\mathbf{p} \in \Delta(\mathcal{Y})^M$, defining $\mathbf{p}'$ by $p_j' = p^*_i := \Q_i(Y = \cdot \mid \mathcal{F}_i)$ and $p_k' = p_k$ for $k \neq j$,
\[
\mathbb{E}_{\Q_i}[\pi_i(\mathbf{p}, \mathbf{w}, Y) \mid \mathcal{F}_i] \;\geq\; \mathbb{E}_{\Q_i}[\pi_i(\mathbf{p}', \mathbf{w}, Y) \mid \mathcal{F}_i],
\]
with strict inequality whenever $p_j \neq \Q_i(Y = \cdot \mid \mathcal{F}_i), w_i > 0, \text{ and } w_j > 0$. That is, if from agent $i$'s perspective (i.e., $\Q_i$), the prediction of another agent $j$ improves, then agent $i$'s expected net payoff does not increase.
\end{definition}

We defer proofs of this section to Appendix~\ref{proof:variants}.

\begin{restatable}[Normality of \proj I]{proposition}{normality}
\label{thm:normality}
\proj~I satisfies normality. 
\end{restatable}
 

\begin{restatable}[Arbitrage in \proj I]{proposition}{arbone}
\label{prop:arb1}
\proj~I does not satisfy no-arbitrage. 
When predictions from other agents are not all identical, there exists a report $\hat{p}_i$ and a sufficiently small wager $w_i$ such that $\pi_i > 0$ for all outcomes $y \in \mathcal{Y}$. 
\end{restatable}


 

\begin{restatable}[No-Arbitrage of \proj II]{proposition}{noarb}
\label{thm:noarb}
\proj~II satisfies no-arbitrage.
\end{restatable}

 

\begin{restatable}[Normality Failure of \proj II]{proposition}{normtwo}
\label{prop:norm2}
\proj~II does not satisfy normality: by adopting $p_i^*$, agent $j$ can change the pooled baseline $q_{-i}$ in a direction that increases player $i$'s expected net payoff. 
\end{restatable}


Both variants satisfy DSIC under any belief structure (Theorem~\ref{thm:dsic}), advantage--wager alignment at best response (Theorem~\ref{thm:swa}), and prediction-agnostic wager optimization, 
and the choice between variants is a design decision, depending on the presence of strategic behavior. However, both compromise the budget balance.
We show below that the worst-case deficit of our mechanism is bounded and is independent of the number of participants $M$. For our net payout function using the Brier score (Def.~\ref{def:brier}), the worst-case deficit is $c_2^2/c_3$. 


\begin{restatable}[Bounded Worst-Case Deficit]{proposition}{budget}
\label{thm:budget}
Under the net payout function~\eqref{eq:ourpayout} with a bounded scoring rule ($s \in [\underline{s}, \overline{s}]$) for \proj~I, or a bounded and concave scoring rule for \proj~II (e.g., the Brier score), the total net payout satisfies
\begin{equation}\label{eq:budget-bound}
\sum_{i=1}^M \pi_i \;\leq\; \left(\frac{\overline{s} - \underline{s}}{W} - c_3\right) \sum_{i=1}^M w_i^2 \;\leq\; \frac{(\overline{s} - \underline{s})^2}{4c_3}
\end{equation}
where $W= \sum_{i} w_i$, for any prediction profile $\mathbf{p}\in \Delta(\mathcal{Y})^M$, any wager profile $\mathbf{w}\in \mathbb{R}_{\geq 0}^M$, and realized outcome $y \in \mathcal{Y}$.

\end{restatable}
Compared to WSWM with a global average baseline~\citep{LAMBERT2015axiom}, which achieves exact budget balance, our leave-one-out baseline inflates each agent's deviation by $W/W_{-i}$, producing the leave-one-out inflation term in~\eqref{eq:exact-budget}.
The regularization $c_3 \sum_i w_i^2$ serves to offset this inflation (e.g., by taxing wagers). 
The regularization constant $c_3$ is the designer's lever: a larger $c_3$ reduces the worst-case deficit at the cost of suppressing wagers.

\subsubsection{Aggregation}\label{apx:aggregation}
After eliciting predictions $\{p_i\}_{i=1}^M$ and wagers $\{w_i\}_{i=1}^M$, the platform aggregates individual reports into a single probabilistic forecast $\hat{p} \in \Delta(\mathcal{Y})$ using wagers as weights.
 
\paragraph{Linear pooling.} The wager-weighted arithmetic mean
\begin{equation}\label{eq:linear-pool}
\hat{p}_{\mathrm{lin}} := \frac{\sum_{i=1}^M w_i\, p_i}{W} \in \Delta(\mathcal{Y})
\end{equation}
is the unique minimizer of the wager-weighted sum of squared Euclidean distances $\sum_i w_i \|q - p_i\|_2^2$ over $q \in \Delta(\mathcal{Y})$, or equivalently, the minimizer of the wager-weighted KL divergence $\sum_i w_i\, \mathrm{KL}(p_i \| q)$. Linear pooling is computationally simple and tends to preserve disagreement among forecasters.
 
\paragraph{Logarithmic pooling.} The wager-weighted geometric mean
\begin{equation}\label{eq:log-pool}
\hat{p}_{\mathrm{log}}[y] \propto \prod_{i=1}^M p_i[y]^{w_i/W}, \qquad y \in \mathcal{Y},
\end{equation}
minimizes the wager-weighted reverse KL divergence $\sum_i w_i\, \mathrm{KL}(q \| p_i)$. It is externally Bayesian and satisfies marginalization consistency \citep{Genest1984}, but emphasizes regions of consensus among forecasters, typically yielding a sharper aggregated distribution and exhibiting sensitivity to probabilities near zero.

\subsection{Learning a Wager Policy from Payouts}
\label{sec:learn_wager}
Consider the repeated, \textit{multiple-question} setting, where questions $x^{(t)}$ arrive sequentially and net payouts $\pi^{(t)}_{i}$ are realized after each round $t$ under the true data-generating probability measure $\PP$.
By observing these net payouts, an agent can learn her comparative advantage, continuously updating her wager policy to approximate the best response.

\paragraph{Learning procedure.}
We consider LLM agents, where each model $i$ has a forward pass $f_i$ that maps input $(X^{(t)}, S^{(t)}_{i})$ to a hidden state $h^{(t)}_{i} \coloneqq f_i(X^{(t)}, S^{(t)}_{i})$ where $h^{(t)}_{i} \in \mathcal{H}$.
At each round, agent $i$ produces (1)~a prediction $p^{(t)}_{i} = \Q_i(Y^{(t)} \mid X^{(t)}, S^{(t)}_{i}) = g_{\phi_i}(h^{(t)}_{i})$ from the model's output layer $g_{\phi_i} : \mathcal{H} \rightarrow \Delta(\mathcal{Y})$ and (2)~a wager $w^{(t)}_{i} = g_{\theta_i}(h^{(t)}_{i})$ from a learned wager network $g_{\theta_i}\colon \mathcal{H} \to \mathbb{R}_{\geq 0}$ that maps the hidden state $h^{(t)}_{i}$ to a non-negative scalar. 
Specifically, the wager network is trained to minimize the negative empirical net payout:
$\mathcal{L}(\theta_i) = -\frac{1}{T}\sum_{t=1}^T \left[g_{\theta_i}(h^{(t)}_{i})\!\left(s^{(t)}_{i} - b^{(t)}_{-i}\right) - c_3\, g_{\theta_i}(h^{(t)}_{i})^2\right]$,
where $s^{(t)}_{i} = s(p^{(t)}_{i}, y^{(t)})$ and $b^{(t)}_{-i}$ is the realized leave-one-out baseline at round $t$. As agents update their wager policies simultaneously, this is a non-stationary multi-agent learning environment.
The realized score differential $s^{(t)}_{i} - b^{(t)}_{-i}$ serves as a noisy sample of 
the comparative advantage, where randomness comes from the data-generating process and opponents' wager policies. 


\paragraph{Why learn wagers rather than improve predictions?}
Improving predictions requires approximating $\PP(Y \mid X, S_i)$, a task constrained by the model's architecture, training data, and capacity.
Learning wagers is comparatively lightweight: it requires only a small network operating on the hidden state to learn when and to what extent the model possesses a comparative advantage, rather than how to generate better predictions.
This is possible precisely because of prediction-agnostic wager optimization. 
By separating wager optimization from prediction, \proj~allows smaller models to focus on domains of comparative advantage, rather than competing with larger generalist models across all tasks.

\paragraph{Will agents learn to stop participating?}
In the extreme, if one model dominates or each model perfectly identifies its best
instances, the mechanism reduces to single-model prediction or per-question routing.
In practice, non-stationarity---through model updates, evolving query distributions, and
retrieval augmentation---prevents comparative advantages from being perfectly learned, helping sustain multi-model participation and meaningful aggregation.

\section{Empirical Evaluation: Eliciting and Aggregating LLM Predictions }
\label{sec:experiment}

In this section, we first introduce the experiment setup, including participating agents, datasets, evaluation metrics, and baselines. After detailing the decentralized learning procedures of LLM participants in \proj, we present results across three scenarios: homogeneous models with private contexts, heterogeneous models, and heterogeneous models with private contexts.
\subsection{Experiment Setup: Agents, Datasets, Evaluation Metrics, and Baselines
}\label{sec:setup}
\paragraph{Participating agents.} We draw participants from a pool of four LLMs spanning general-purpose and domain-specialized capabilities: Gemma-2-9B (Gemma2) \citep{gemmateam2024gemma2}, Llama3.1-8B (Llama3.1) \citep{grattafiori2024llama3herdmodels}, Llama3-Aloe-8B (Aloe) \citep{gururajan2024aloe}, a medical domain fine-tuned model, and BioMistral-7B (BioMistral) \citep{labrak-etal-2024-BioMistral}, pre-trained on biomedical corpora. Each model produces a probability distribution over the outcome space from its output token logits, following standard practice for multiple-choice and forecasting evaluations.

\paragraph{Datasets.} We evaluate on five datasets covering question answering and event forecasting: PubMedQA \citep{jin-etal-2019-PubMedQA}, MedMCQA \citep{MedMCQA}, MMLU \citep{hendrycks2021measuring}, ARC-Challenge \citep{clark2018think}, and BayesX, a synthetic Bayesian forecasting dataset we construct. 
%
Existing QA datasets primarily evaluate whether additional context or expert knowledge improves prediction accuracy.
Forecasting presents a different challenge: additional information may reduce epistemic uncertainty about the world while simultaneously revealing greater aleatoric uncertainty in the outcome-generating process.
BayesX is designed to capture this phenomenon. Its priors can be highly concentrated, but private context often shifts the posterior toward a more uniform distribution, reflecting a better-informed yet less confident prediction. As a result, BayesX tests whether an aggregation method can recognize the informed model even when additional information makes its forecast appear less certain.


PubMedQA and BayesX include question-related contexts that may be selectively revealed to a subset of participants, enabling experiments with heterogeneous information.
Details about datasets are provided in Appendix~\ref{apx:datasets}.
%

\paragraph{Metrics.}
Beyond standard predictive performance metrics---area under the ROC curve (AUC) and accuracy (ACC)---and expected calibration error (ECE), we evaluate whether the aggregation weights identify high-performing experts on each question. Since \proj uses the Brier score as its scoring rule, we additionally evaluate whether the aggregation weights align with per-question Brier-score performance using three metrics: 
\begin{itemize}
    \item Kendall's Tau (K-Tau): Rank correlation between the ordering of models by their wagers and by their per-question Brier scores; higher values indicate the wagers correctly order models by quality.
    
    \item Mean Reciprocal Rank (MRR): Average inverse rank of the best-performing model under the ordering induced by aggregation weights; MRR $= 1$ means the best model always receives the highest weight.
    
    \item Dynamic Regret (D-Regret): Per-question Brier-score gap between the aggregated prediction and the best individual model; lower values indicate the aggregation is closer to oracle per-question model selection.
\end{itemize}
Formal definitions are in Appendix~\ref{apx:metrics}.
All reported metrics in tables are scaled by $\times 100$, and all reported intervals ($\pm$) are $95\%$ confidence intervals over five independent runs.

\begin{table}[ht]
\centering
\caption{Categorization of baselines. See Sec.~\ref{sec:setup} for discussion.}
\label{tab:baseline}
 \resizebox{1.0\linewidth}{!}{%
\begin{tabular}{l
                    *{6}{c}}
\toprule
\textbf{Methods}  & Uncertainty-Aware & Decentralized &  Learning  & IC / Manipulation-Resistant  \\
\midrule
Pre-inference Router~\citep{lu2023routing, chen2024routerdc, ong2025routellm, song2025irt}  & \xmark  & \xmark & \cmark & --\\
Stacked Generalization~\citep{WOLPERT1992241}  & \cmark  & \xmark & \cmark& --\\
Uniform Averaging~\citep{wang2022selfconsistency, li2024more, yu-etal-2024-breaking} & \xmark  & \cmark & \xmark & \xmark\\
Weight by Model Confidence~\citep{kang2025scalable}  & \cmark  & \cmark & \xmark & \xmark\\
Weight by Context Perplexity~\citep{mavromatis2024pack}  & \cmark  & \cmark & \xmark& \xmark\\
\proj (Ours)  & \cmark  & \cmark & \cmark& \cmark\\ 
\bottomrule

\end{tabular}%
}
\end{table}

\paragraph{Baselines.}
We consider seven baselines spanning pre-inference routers~\citep{chen2024routerdc, song2025irt, ong2025routellm}, learning-free ensembling~\citep{mavromatis2024pack, kang2025scalable, wang2022selfconsistency}, and learning-based ensembling with LLM hidden states~\citep{WOLPERT1992241}. Table~\ref{tab:baseline} categorizes all methods along four axes: uncertainty-awareness, decentralized aggregation, learned weights, and incentive compatibility.

Among \textit{learning-free methods}, UniformAvg performs equally weighted aggregation. 
SelfCertainty~\citep{kang2025scalable} uses the predicted distribution to estimate confidence and treats it as the aggregation weight. 
PackLLM~\citep{mavromatis2024pack} uses the inverse perplexity on input tokens as weights, requiring access to model parameters. 
These methods can be considered decentralized, and some are uncertainty-aware, but they do not learn from data, may be sensitive to miscalibrated confidence estimates, and are not incentive-compatible---participants can misreport confidence or perplexity to gain disproportionate weight.

\textit{Learning-based methods}, including pre-inference routers~\citep{chen2024routerdc, ong2025routellm, song2025irt} and StackedGen~\citep{WOLPERT1992241}, require centralized control to train routing or aggregation networks. 
Pre-inference routers learn weights conditional on input features but do not capture each model's internal uncertainty on the specific question. 
StackedGen trains on concatenated hidden states and is therefore uncertainty-aware, but requires centralized access to all models' internals.
Since StackedGen has access to all models' hidden states, it serves a strong centralized reference point, with more information than any decentralized method. 

\proj simultaneously achieves all four properties (Table~\ref{tab:baseline}). Detailed descriptions of baselines are in Appendix~\ref{apx:baselines}.
Other experiment setup details are provided in Appendix~\ref{apx:exp_setup}.

\subsection{Decentralized Learning of Wager Policies for LLMs}
\label{sec:learn_wager_policy}

In \proj, LLM participants produce predictions via their forward passes and learn wager policies as described in Section~\ref{sec:learn_wager}. 
The base model parameters are frozen while only the wager network is trained. Specifically, we extract the last-layer hidden state after processing input tokens (i.e., question $X$ and context $S$) and pass it to a two-layer perceptron (MLP) to output a non-negative scalar wager. 
The goal is for each agent to learn her comparative advantage: \textit{on which or what type of questions it is likely to outperform the pool, and by how much}. 
We allow all participants' wager networks to be trained simultaneously without assuming that opponents' strategies are fixed. This procedure simulates how, in a multi-agent system, each LLM participant would independently learn to wager in a deployed wagering mechanism, using only her own model's hidden states and observed payouts.
A key property of the payout function~\eqref{eq:ourpayout} is that each agent can recover the hindsight optimal wager from her own payout alone: since $\pi_i = w_i(s_i - b_{-i} - c_3 w_i)$, the realized comparative advantage $s_i - b_{-i} = \pi_i / w_i + c_3 w_i$ and thus the hindsight optimal wager is computable from the agent's observed payout and wager, without access to other agents' predictions, wagers, or scores. The wager network is trained to approximate the resulting target $w^* = ((s_i - b_{-i}) / 2c_3)^+$, so that the learned wager approximates 
the best-response wager to the data-generating distribution~$\PP$ and opponents' wager policies
conditioned on the model's hidden state.

To preserve the feasible range of $w^*$ while improving numerical stability, we apply a sigmoid transformation to the output of the wager head.  Training minimizes the mean squared error (MSE) between the predicted wager $w$ and the hindsight target $w^*$.\footnote{This objective is equivalent to minimizing the hindsight regret of failing to play $w^*$ in each round when $w^*>0$. When $w^*=0$, the MSE objective yields a less aggressive update toward zero than direct regret minimization, as the penalty is quadratic rather than linear, thereby encouraging continued participation.}
The mechanism parameters and learning hyperparameters are reported in Table~\ref{tab:params}.



\begin{table*}[ht]
\centering
\begin{minipage}{\textwidth}
\caption{Homogeneous models (Aloe) on PubMedQA with private contexts. For each question, one of $M \in \{4, 8, 12\}$ models is randomly selected to receive a relevant abstract; the rest receive none.\protect\footnotemark 
}
\label{tab:homo}
 \resizebox{1.0\linewidth}{!}{%
\begin{tabular}{l
                    *{5}{c} 
                   *{5}{c} *{5}{c} }
\toprule
& \multicolumn{5}{c}{\textbf{4}} & \multicolumn{5}{c}{\textbf{8} } & \multicolumn{5}{c}{\textbf{12}} \\
\cmidrule(lr){2-6} \cmidrule(lr){7-11} \cmidrule(lr){12-16}
\textbf{Method} &AUC $\uparrow$ & ACC $\uparrow$ & ECE $\downarrow$  & MRR $\uparrow$ & DRegret $\downarrow$ &AUC $\uparrow$ & ACC $\uparrow$ & ECE $\downarrow$  & MRR $\uparrow$ & DRegret $\downarrow$ &AUC $\uparrow$ & ACC $\uparrow$ & ECE $\downarrow$  & MRR $\uparrow$ & DRegret $\downarrow$\\
\midrule
 UniformAvg &  70.29 &  71.66 &  2.27  &-  &  20.80 &  64.85  & 67.20 & 3.12 & - &  25.35 &  63.10 &  66.17  & 4.49 &  - & 27.30 \\
SelfCertainty & 71.78 & 79.23 & 2.67 & 86.76  & 13.62 &  69.31 & 76.92&3.33&84.73& 16.50 &  68.05  & 75.66 & 3.65 & 84.07  & 18.12\\
PackLLM & 78.23 &  85.22  &  10.23  & 88.64  & 8.01 &  75.65 &  81.71  &9.38 &  87.85 &11.99& 73.65  &  78.97  & 7.44 & 87.67 & 14.64\\

\midrule
RouterDC w/ context &  69.36\tiny{$\pm$13.24} & 72.12\tiny{$\pm$25.21}  &2.58\tiny{$\pm$20.64} & 61.24\tiny{$\pm$351.39}  & 20.82\tiny{$\pm$20.95}  & 65.17\tiny{$\pm$3.02}  &67.51\tiny{$\pm$1.47} & 3.10\tiny{$\pm$1.97}&  47.36\tiny{$\pm$125.47}& 25.46\tiny{$\pm$2.19} & 63.35\tiny{$\pm$4.16}  &66.17\tiny{$\pm$0.42}& 4.75\tiny{$\pm$2.33}  &41.58\tiny{$\pm$174.74}&  27.20\tiny{$\pm$1.31}   \\
NIRTRouter w/ context &  70.19\tiny{$\pm$1.32} & 71.71\tiny{$\pm$0.63} &2.15\tiny{$\pm$1.17} & 52.26\tiny{$\pm$14.97}& 20.80\tiny{$\pm$0.01}&  64.85\tiny{$\pm$0.00} &67.53\tiny{$\pm$0.00}&  3.18\tiny{$\pm$0.00}&  33.38\tiny{$\pm$24.19} &  25.57\tiny{$\pm$0.00} & 63.10\tiny{$\pm$0.01}&  66.17\tiny{$\pm$0.00}&  4.91\tiny{$\pm$0.19}&  24.29\tiny{$\pm$1.11}  & 27.30\tiny{$\pm$0.00}   \\
RouteLLM w/ context & 82.32\tiny{$\pm$0.07} &  87.48\tiny{$\pm$0.06} &  7.24\tiny{$\pm$0.07} & 87.81\tiny{$\pm$0.83} &   3.51\tiny{$\pm$0.08}  & 82.22\tiny{$\pm$0.04} &  87.50\tiny{$\pm$0.03} & 7.21\tiny{$\pm$0.02}  & 85.77\tiny{$\pm$1.05}  &  3.54\tiny{$\pm$0.02}  & 82.52\tiny{$\pm$0.05} &  87.39\tiny{$\pm$0.03} & 7.17\tiny{$\pm$0.02}  & 84.15\tiny{$\pm$0.99}  & 3.55\tiny{$\pm$0.02} \\
RouteLLM w/o context & 69.49\tiny{$\pm$0.39}  &  72.01\tiny{$\pm$0.19} &  2.43\tiny{$\pm$0.70}  & 52.07\tiny{$\pm$0.11} &  20.64\tiny{$\pm$0.00}  & 65.08\tiny{$\pm$0.38} &  67.44\tiny{$\pm$0.19} & 3.24\tiny{$\pm$0.49}  & 33.29\tiny{$\pm$1.65}  & 25.63\tiny{$\pm$0.01}  & 63.85\tiny{$\pm$0.35}  &  65.80\tiny{$\pm$0.17} &4.72\tiny{$\pm$0.14}   & 28.31\tiny{$\pm$2.59}    & 27.34\tiny{$\pm$0.02}  \\
StackedGen & 81.99\tiny{$\pm$0.28} &  87.76\tiny{$\pm$0.11} & 7.53\tiny{$\pm$0.11} & 88.01\tiny{$\pm$0.25}  & 3.29\tiny{$\pm$0.13}  &   82.65\tiny{$\pm$0.74} &87.44\tiny{$\pm$0.18} &7.48\tiny{$\pm$0.38} &  85.25\tiny{$\pm$1.20}   &  3.49\tiny{$\pm$0.25}  &  82.26\tiny{$\pm$0.32}  & 87.36\tiny{$\pm$0.29}   &7.29\tiny{$\pm$0.09} &  83.76\tiny{$\pm$1.05}    &  3.66\tiny{$\pm$0.31}  \\
\midrule
\proj~I & 81.07\tiny{$\pm$0.58}  & 87.00\tiny{$\pm$0.47} & 9.76\tiny{$\pm$0.33}  & 86.97\tiny{$\pm$1.89}  & 5.19\tiny{$\pm$0.64} &  84.47\tiny{$\pm$0.28}  & 87.09\tiny{$\pm$0.14}  & 8.35\tiny{$\pm$0.48} &  81.90\tiny{$\pm$0.86} &  9.32\tiny{$\pm$0.22}  & 82.10\tiny{$\pm$0.35}   & 86.66\tiny{$\pm$0.32}  & 10.89\tiny{$\pm$0.80} & 77.14\tiny{$\pm$1.94} & 12.06\tiny{$\pm$0.62} \\
\proj~II & 81.07\tiny{$\pm$0.19}  &86.68\tiny{$\pm$0.11} & 9.83\tiny{$\pm$0.14}  & 88.46\tiny{$\pm$0.22} &  5.49\tiny{$\pm$0.05}  & 79.85\tiny{$\pm$0.34}  & 85.95\tiny{$\pm$0.39}  & 10.45\tiny{$\pm$0.29}  & 85.07\tiny{$\pm$1.39}  & 7.15\tiny{$\pm$0.33} &   78.65\tiny{$\pm$0.60}    & 85.17\tiny{$\pm$0.45}  & 10.61\tiny{$\pm$0.37} & 83.84\tiny{$\pm$0.57} & 8.46\tiny{$\pm$0.26}
\\
\bottomrule

\end{tabular}%
}
\end{minipage}
\end{table*}

\footnotetext{Aloe achieves AUC=82.39, ACC=87.57, ECE=7.38 with abstract and AUC=60.65, ACC=63.10, ECE=7.83 without.}

\begin{table*}[ht]
\centering
\begin{minipage}{\textwidth}
\caption{Homogeneous models (Gemma2) on BayesX with private contexts.  For each question, one of $M \in \{4,8,12\}$ models is randomly selected to receive evidence.\protect\footnotemark}\label{tab:bayex}
 \resizebox{1.0\linewidth}{!}{%
\begin{tabular}{l
                    *{4}{c} 
                   *{4}{c} *{4}{c} }
\toprule
& \multicolumn{4}{c}{\textbf{4}} & \multicolumn{4}{c}{\textbf{8}} & \multicolumn{4}{c}{\textbf{12}} \\
\cmidrule(lr){2-5} \cmidrule(lr){6-9} \cmidrule(lr){10-13} 
\textbf{Method} & KLD $\downarrow$ &  TVD $\downarrow$ & MRR $\uparrow$ & DRegret $\downarrow$ & KLD $\downarrow$ &  TVD $\downarrow$  & MRR $\uparrow$ &  DRegret $\downarrow$ & KLD $\downarrow$ &  TVD $\downarrow$   & MRR $\uparrow$ &  DRegret $\downarrow$ \\
\midrule
UniformAvg & 25.64 & 32.96  & - &    19.65  & 41.09  & 38.93  & - & 28.79   & 48.67 & 40.76 &  -  & 32.02  \\
SelfCertainty &   61.44  & 43.99  &25.00   & 37.37 &  68.46 &  44.58 &  12.50  &  38.74  &  70.80 &  44.59 &  8.33  & 38.93  \\
PackLLM & 21.58 & 30.70  & 100.00 &  16.66 & 37.36 & 37.75 & 100.00 &  26.75 & 45.32 & 39.93&  100.00  & 30.60  \\
\midrule
RouterDC w/ context & 22.88\tiny{$\pm$0.26}  & 30.00\tiny{$\pm$0.38}  & 65.55\tiny{$\pm$12.27}  &  16.77\tiny{$\pm$0.18}  & 39.96\tiny{$\pm$14.38} & 38.38\tiny{$\pm$6.81} & 39.01\tiny{$\pm$64.03} & 28.02\tiny{$\pm$9.66} &  48.96\tiny{$\pm$0.60} &  40.73\tiny{$\pm$0.12} & 26.15\tiny{$\pm$3.64} &   32.00\tiny{$\pm$0.22}  \\
NIRTRouter w/ context & 22.49\tiny{$\pm$17.12} & 30.72\tiny{$\pm$16.17} & 52.11\tiny{$\pm$0.32}   & 17.11\tiny{$\pm$15.72} & 40.50\tiny{$\pm$0.49}  & 38.76\tiny{$\pm$0.03}  & 33.97\tiny{$\pm$0.00}  & -0.35\tiny{$\pm$1.46}& 45.18\tiny{$\pm$0.11} & 38.97\tiny{$\pm$0.02} & 25.91\tiny{$\pm$0.11} &  29.74\tiny{$\pm$0.11} \\
RouteLLM w/ context & 0.64\tiny{$\pm$0.07} & 3.75\tiny{$\pm$0.92} & 100.00\tiny{$\pm$0.00}  & -2.31\tiny{$\pm$0.07}  & 0.83\tiny{$\pm$2.32} & 4.67\tiny{$\pm$13.61}  &100.00\tiny{$\pm$0.00}  &  -2.12\tiny{$\pm$2.25} &  0.93\tiny{$\pm$2.46}  & 5.27\tiny{$\pm$10.68} & 100.00\tiny{$\pm$0.00}  &  -2.02\tiny{$\pm$2.40}  \\
RouteLLM w/o context &   25.67\tiny{$\pm$0.22}  & 
32.96\tiny{$\pm$0.01} & 52.08\tiny{$\pm$0.00}   & 19.66\tiny{$\pm$0.08} &  41.11\tiny{$\pm$0.20} &  38.92\tiny{$\pm$0.00}&  33.97\tiny{$\pm$0.00}&     28.78\tiny{$\pm$0.01}&   48.68\tiny{$\pm$0.15}&  40.76\tiny{$\pm$0.06} &  25.91\tiny{$\pm$0.05}  & 32.01\tiny{$\pm$0.10}   \\
StackedGen & 1.27\tiny{$\pm$0.04}  & 6.45\tiny{$\pm$0.11} &  100.00\tiny{$\pm$0.00} & -1.69\tiny{$\pm$0.04} & 1.51\tiny{$\pm$0.03} &  7.06\tiny{$\pm$0.08} &  100.00\tiny{$\pm$0.00} &  -1.46\tiny{$\pm$0.03} &   1.82\tiny{$\pm$0.04}  & 7.93\tiny{$\pm$0.10} & 100.00\tiny{$\pm$0.00} &    -1.13\tiny{$\pm$0.04} \\
\midrule
 \proj~I &  2.22\tiny{$\pm$0.11} & 8.69\tiny{$\pm$0.22}  & 100.00\tiny{$\pm$0.00} &  -0.74\tiny{$\pm$0.11} &  5.87\tiny{$\pm$0.50} & 15.97\tiny{$\pm$0.82} &  100.00\tiny{$\pm$0.00} &  2.86\tiny{$\pm$0.49} &  12.59\tiny{$\pm$0.98} &   24.01\tiny{$\pm$0.90} &   100.00\tiny{$\pm$0.00} &     9.19\tiny{$\pm$0.88} \\
\proj~II & 2.22\tiny{$\pm$0.07}   & 8.68\tiny{$\pm$0.14}  & 100.00\tiny{$\pm$0.00} & -0.74\tiny{$\pm$0.07} &  5.91\tiny{$\pm$0.08} & 16.05\tiny{$\pm$0.12} & 100.00\tiny{$\pm$0.00}  & 2.90\tiny{$\pm$0.08} & 13.20\tiny{$\pm$1.26} & 24.57\tiny{$\pm$1.13} & 100.00\tiny{$\pm$0.00}  &  9.74\tiny{$\pm$1.12}  \\
\bottomrule

\end{tabular}%
}
\end{minipage}
\end{table*}
\footnotetext{Gemma2 achieves KLD=3.24 and TVD=10.27 with evidence, and KLD=78.30 and TVD=45.46 without.}

\begin{table*}[ht]
\centering
\begin{minipage}{\textwidth}
\caption{Heterogeneous models without context on three datasets. MedMCQA and MMLU are in-distribution; ARC-Challenge is out-of-distribution, with MMLU as the training set.
}\label{tab:4models}
 \resizebox{1.0\linewidth}{!}{%
\begin{tabular}{l
                    *{5}{c} 
                   *{5}{c} *{5}{c}}
\toprule
& \multicolumn{5}{c}{\textbf{MedMCQA (i.d.)}} & \multicolumn{5}{c}{\textbf{MMLU (i.d.)}} & \multicolumn{5}{c}{\textbf{ARC-Challenge (o.o.d.)}} \\
\cmidrule(lr){2-6} \cmidrule(lr){7-11} \cmidrule(lr){12-16}
\textbf{Method} & ACC $\uparrow$ & ECE $\downarrow$  & MRR $\uparrow$ & K-Tau $\uparrow$ & DRegret $\downarrow$ & ACC $\uparrow$ & ECE $\downarrow$  & MRR $\uparrow$ & K-Tau $\uparrow$ & DRegret $\downarrow$ & ACC $\uparrow$ & ECE $\downarrow$  & MRR $\uparrow$ & K-Tau $\uparrow$ & DRegret $\downarrow$ \\
\midrule
Gemma2 & 62.90  &25.28 &  64.99 &  5.93  & 47.28  & 87.18  & 8.72  & 82.80  & 32.61  & 12.94   &90.43  & 7.55  & 89.18  & 38.43  & 10.67  \\
Aloe &  79.77  &  3.18   & 59.49  & 14.42  &  15.77 &   80.23 &    8.77 &   53.42   & 3.68  &  21.34  &  80.35  &     8.58  &  51.17  &  -0.43  &  23.43  \\
Llama3.1 &   78.97   & 2.64  &  51.56  & 4.47   & 15.97  &  81.99 &   2.52   & 40.88  &  -8.01  &  17.61  &  81.30   & 3.06  &  37.43  &  -12.20 &   21.01    \\
BioMistral & 47.47  &   17.10  &   37.71  &   -24.82  &   56.24   &  66.05  &   10.92  &   32.59   &  -28.27  &   39.95 &    66.61   &  10.95  &   30.68  &   -25.80   &  40.65    \\ 

\midrule

UniformAvg & 76.45 & 8.10  & - & 0.00 & 21.05 & 85.15 & 5.10 & - &0.00 &13.87& 87.39 &7.40 &-& 0.00 &14.23  \\

SelfCertainty & 74.42 &  4.58 &  79.78 &  43.61  &  22.08 &    86.58 &   1.36  &  88.75 &   63.21  &   11.86  &  89.22  &  2.70   & 91.59  & 66.96 &  10.71 \\

PackLLM & 79.00 &  9.93  &  50.77  &  8.85  & 18.93  &   84.94   & 5.13   & 37.79  &  -27.81  &  14.28  &  86.26  &  6.87  &  36.38 &   -21.80  &  14.93   \\
\midrule
RouterDC & 78.95\tiny{$\pm$5.46}  &  8.77\tiny{$\pm$0.36}  & 55.65\tiny{$\pm$1.86}  &  27.65\tiny{$\pm$1.30}   & 18.35\tiny{$\pm$8.02}  & 85.28\tiny{$\pm$0.22}   & 4.75\tiny{$\pm$0.24} & 61.71\tiny{$\pm$22.90} &  30.16\tiny{$\pm$20.00}  &  13.61\tiny{$\pm$0.15} &  87.48\tiny{$\pm$0.19} &  7.22\tiny{$\pm$0.22}  & 63.21\tiny{$\pm$30.19}  & 27.50\tiny{$\pm$25.39}  & 13.98\tiny{$\pm$0.18}  \\
NIRTRouter &   81.15\tiny{$\pm$0.83}  &  6.99\tiny{$\pm$0.22}  &  55.64\tiny{$\pm$1.77} &  27.67\tiny{$\pm$1.17}  & 14.56\tiny{$\pm$0.14}  &  86.91\tiny{$\pm$5.79}  &   2.15\tiny{$\pm$0.22} &  82.72\tiny{$\pm$0.93} &   43.72\tiny{$\pm$4.81}  &  11.51\tiny{$\pm$5.78}  &  90.17\tiny{$\pm$2.21}  &  5.12\tiny{$\pm$1.75}  &  89.18\tiny{$\pm$0.00}   & 46.38\tiny{$\pm$0.00}  &  10.65\tiny{$\pm$3.17} \\
RouteLLM &  81.70\tiny{$\pm$1.56}  &  3.62\tiny{$\pm$3.15}  &  55.64\tiny{$\pm$1.77}  & 27.67\tiny{$\pm$1.17}  &  12.82\tiny{$\pm$1.79}  &  86.85\tiny{$\pm$6.55}   & 1.96\tiny{$\pm$2.21}  &  82.72\tiny{$\pm$0.93}  &  43.72\tiny{$\pm$4.81}   & 11.03\tiny{$\pm$5.91}  &  90.61\tiny{$\pm$0.00}  &  2.37\tiny{$\pm$0.52}  &  89.18\tiny{$\pm$0.00}  &  46.38\tiny{$\pm$0.00}  &  9.35\tiny{$\pm$1.86}  \\

StackedGen &  82.66\tiny{$\pm$0.25}  &  2.21\tiny{$\pm$0.49}  &  59.06\tiny{$\pm$0.38}  & 30.91\tiny{$\pm$0.50}   & 11.56\tiny{$\pm$0.18} &  87.09\tiny{$\pm$0.69} &   1.99\tiny{$\pm$0.54}  &  82.24\tiny{$\pm$0.38}  &  42.17\tiny{$\pm$1.58}   & 10.60\tiny{$\pm$0.79}  &  90.35\tiny{$\pm$0.11}  &  2.19\tiny{$\pm$0.87}  &  89.04\tiny{$\pm$0.18}  &  46.05\tiny{$\pm$0.26}  &  8.85\tiny{$\pm$0.05}  \\
\midrule
 \proj~I & 81.17\tiny{$\pm$0.29} &  7.59\tiny{$\pm$0.36}  &  55.26\tiny{$\pm$0.55}  & 24.17\tiny{$\pm$1.16} &  14.86\tiny{$\pm$0.47}  &  86.78\tiny{$\pm$0.49}  &  3.34\tiny{$\pm$0.35} &  82.60\tiny{$\pm$0.62} &  45.11\tiny{$\pm$4.05}   & 11.89\tiny{$\pm$0.43}  &  89.89\tiny{$\pm$0.31} &   6.24\tiny{$\pm$0.39}  &  86.76\tiny{$\pm$0.97}  &  44.90\tiny{$\pm$1.38}   & 11.06\tiny{$\pm$0.14}   \\
\proj~II &    81.00\tiny{$\pm$0.34}  &   7.85\tiny{$\pm$0.61}  & 55.30\tiny{$\pm$0.75}   & 21.95\tiny{$\pm$1.94}  &  15.36\tiny{$\pm$0.51}  &  86.98\tiny{$\pm$0.54}  &  3.01\tiny{$\pm$0.38}   & 82.52\tiny{$\pm$0.52}  &  40.42\tiny{$\pm$7.92}  &  11.70\tiny{$\pm$0.49}   & 90.28\tiny{$\pm$0.24} &   6.17\tiny{$\pm$0.81}  &  88.07\tiny{$\pm$0.75}  &  42.23\tiny{$\pm$7.02}  &  10.59\tiny{$\pm$0.11}  \\

\proj~I (log-pool) &   80.32\tiny{$\pm$0.10}  & 2.07\tiny{$\pm$0.10}  & 55.01\tiny{$\pm$0.22} & 23.15\tiny{$\pm$0.24} & 14.58\tiny{$\pm$0.11} & 87.36\tiny{$\pm$0.05} & 2.88\tiny{$\pm$0.33}  & 72.32\tiny{$\pm$1.99} &  25.55\tiny{$\pm$4.73} &  10.48\tiny{$\pm$0.01}&  90.65\tiny{$\pm$0.08} &  2.96\tiny{$\pm$0.19} & 79.94\tiny{$\pm$1.26} & 30.41\tiny{$\pm$3.84} &  9.16\tiny{$\pm$0.03}  \\
\proj~I (calibrated) &  82.53\tiny{$\pm$0.05} &  10.56\tiny{$\pm$0.13} &  65.98\tiny{$\pm$0.78} &  38.92\tiny{$\pm$0.94}&  11.69\tiny{$\pm$0.09}  &  87.68\tiny{$\pm$0.08}  & 8.21\tiny{$\pm$0.12}  &70.71\tiny{$\pm$0.18}  &41.02\tiny{$\pm$0.60}   &9.99\tiny{$\pm$0.07}  & 89.59\tiny{$\pm$0.07}   &8.64\tiny{$\pm$0.37}  &  78.77\tiny{$\pm$0.20}  & 50.24\tiny{$\pm$0.32} &  9.84\tiny{$\pm$0.04}   \\

\bottomrule

\end{tabular}%
}
\end{minipage}
\end{table*}

\subsection{Scenario~I: Homogeneous Models with Private Contexts}
In this setting, all participants are copies of the same LLM; heterogeneous predictions arise from private information (mirroring Example~\ref{ex:homo}). 
For each question, one randomly selected model receives relevant context; the rest receive none. 
We use Aloe on PubMedQA and Gemma2 on BayesX, scaling participants from 4 to 12.

On PubMedQA (Table~\ref{tab:homo} and~\ref{tab:homo_subset}), \proj achieves AUC and ACC comparable to StackedGen and RouteLLM (with context)---both of which require centralized access to models' hidden states or private inputs---while outperforming all other baselines. The wager network learns from payout signals and hidden states to identify when the model has received relevant context, wagering high when it holds a comparative advantage and near zero otherwise (Fig.~\ref{fig:alignment} and~\ref{fig:perf_overtime} (appendix)). As the pool grows from 4 to 12, learning-free methods degrade substantially (e.g., UniformAvg AUC drops from 70.3 to 63.1), while \proj remains robust.
On BayesX (Table~\ref{tab:bayex}), \proj outperforms all baselines except RouteLLM (with context) and StackedGen, which easily recognize context in inputs. Without access to contexts, RouteLLM drops to uniform averaging. SelfCertainty performs poorly because observing the signal makes the posterior more faithful but closer to uniform, inverting its confidence estimate.

\subsection{Scenario~II: Heterogeneous Models without Context}
\label{sec:hetero_no_context}
In this setting, we use four heterogeneous LLMs without additional context, so that predictions differ due to distinct model weights (mirroring Example~\ref{ex:hetero}). 
We train a wager network for each LLM on the MedMCQA and MMLU datasets, and evaluate in-distribution on MedMCQA and MMLU, and out-of-distribution on ARC-Challenge. 
Results with four participants are in Table~\ref{tab:4models} and~\ref{tab:4models_subset}; two-participant results are in Tables~\ref{tab:2models}--\ref{tab:2models_subset} in the appendix.

Table~\ref{tab:4models} shows that \proj performs comparably to the centralized StackedGen across all four datasets. 
Aggregation can improve over individual models---especially on MedMCQA---and on MMLU and ARC-Challenge, the aggregated prediction matches the performance of the strongest individual model (Gemma2). 
The wager network here learns a different mapping than in Scenario~I: rather than learning context assignment, it must learn each model's domain-dependent comparative advantage from the hidden state and payout signals.

Learning-free methods are notably less robust to pool composition. 
On MedMCQA, weaker models (Gemma2, BioMistral) cause performance drops for UniformAvg, PackLLM, and SelfCertainty, because miscalibrated confidence in weaker models can inflate their aggregation weight. 
This illustrates another limitation of confidence-based aggregation that holds even without strategic manipulation.


\begin{table*}[ht]
\centering
\begin{minipage}{\textwidth}
\caption{Four heterogeneous models with private contexts on the PubMedQA dataset. For each question, we randomly assign one model with a relevant context (left columns) or two models, each with a relevant and an irrelevant context (right columns).}\label{tab:hetero}
 \resizebox{1.0\linewidth}{!}{%
\begin{tabular}{l
                    *{6}{c} 
                   *{6}{c}}
\toprule
& \multicolumn{6}{c}{\textbf{One Relevant Context}} & \multicolumn{6}{c}{\textbf{One Relevant Context \& One Irrelevant Context}} \\
\cmidrule(lr){2-7} \cmidrule(lr){8-13} 
\textbf{Method} & AUC $\uparrow$ & ACC $\uparrow$ & ECE $\downarrow$  & MRR $\uparrow$ & K-Tau $\uparrow$ & DRegret $\downarrow$ & AUC $\uparrow$ & ACC $\uparrow$ & ECE $\downarrow$  & MRR $\uparrow$ & K-Tau $\uparrow$ & DRegret $\downarrow$ \\
\midrule
UniformAvg &  73.43 &  76.46 &  11.51  &  -    & 0.00 &   24.28   &  72.76  &  74.60  &  11.32  &  -  &  0.00  &  27.66 \\
SelfCertainty &  69.84  &  77.25  &  3.13  &  86.11  &  46.30  &  21.60  &  66.58  &  71.40  &  10.89  &  83.61  &  39.68  &  31.20 \\
 
PackLLM & 77.85&83.60&15.38 & 66.42 & 12.32  &  18.05  & 71.95  & 77.02  & 12.86 &  58.23  & 5.39  & 26.41\\
\midrule
RouterDC w/ context & 73.33\tiny{$\pm$0.63} & 76.55\tiny{$\pm$1.26} & 11.58\tiny{$\pm$1.91} & 54.66\tiny{$\pm$37.92} & 10.64\tiny{$\pm$64.71} &  24.23\tiny{$\pm$0.34} &  72.53\tiny{$\pm$1.44}  & 74.90\tiny{$\pm$2.94}  &  11.74\tiny{$\pm$3.55} & 48.77\tiny{$\pm$6.51}   & -2.19\tiny{$\pm$11.98}  &  27.65\tiny{$\pm$0.52} \\
NIRTRouter w/ context &  73.20\tiny{$\pm$0.58}  &  76.92\tiny{$\pm$0.84}   & 12.04\tiny{$\pm$7.80}   & 43.90\tiny{$\pm$50.11}  &  -7.29\tiny{$\pm$152.52}   & 24.15\tiny{$\pm$3.73}  &  72.53\tiny{$\pm$0.34}  & 75.40\tiny{$\pm$0.00}  & 12.27\tiny{$\pm$4.23}  & 43.15\tiny{$\pm$69.64}  & -7.18\tiny{$\pm$130.40}  & 27.48\tiny{$\pm$2.74} \\

RouteLLM w/ context & 74.46\tiny{$\pm$0.57} &  79.11\tiny{$\pm$0.59} &  5.66\tiny{$\pm$1.40}  & 60.84\tiny{$\pm$2.83} & 14.24\tiny{$\pm$2.72} & 19.18\tiny{$\pm$0.32} & 73.19\tiny{$\pm$0.72} & 
77.89\tiny{$\pm$0.64}& 
5.97\tiny{$\pm$1.35}& 
62.18\tiny{$\pm$2.94}& 
15.00\tiny{$\pm$2.44}& 
22.28\tiny{$\pm$0.32}
\\
RouteLLM w/o context &  72.15\tiny{$\pm$5.90}  & 76.27\tiny{$\pm$0.21} &  8.26\tiny{$\pm$1.51}  & 55.46\tiny{$\pm$56.78}&  11.11\tiny{$\pm$36.00}  & 23.09\tiny{$\pm$3.05}&    68.97\tiny{$\pm$12.13}  & 75.81\tiny{$\pm$3.57} &   5.91\tiny{$\pm$3.73}   & 62.59\tiny{$\pm$7.84}  &  15.84\tiny{$\pm$1.12} &  25.20\tiny{$\pm$5.89}  \\
StackedGen &  81.17\tiny{$\pm$0.30}  &  84.82\tiny{$\pm$0.18}  &  7.55\tiny{$\pm$0.43} &  71.74\tiny{$\pm$0.14}  &  26.28\tiny{$\pm$0.84}  &  12.34\tiny{$\pm$0.06}  &    80.30\tiny{$\pm$0.56}  &   84.38\tiny{$\pm$1.72}   &  6.57\tiny{$\pm$1.81}   &  71.52\tiny{$\pm$0.81}   &  24.06\tiny{$\pm$0.39}  & 14.07\tiny{$\pm$0.93} \\
\midrule
\proj~I &   79.76\tiny{$\pm$0.65}  &   82.11\tiny{$\pm$0.18}  &  11.58\tiny{$\pm$0.52}    &  72.59\tiny{$\pm$0.75}   &   28.06\tiny{$\pm$0.59}  &    16.95\tiny{$\pm$0.05}   &  77.98\tiny{$\pm$0.37}    &  81.15\tiny{$\pm$0.18}   &  11.88\tiny{$\pm$0.35}   &  74.98\tiny{$\pm$0.81}   &   28.38\tiny{$\pm$0.32}  &  20.09\tiny{$\pm$0.09}   \\
\proj~II &      79.73\tiny{$\pm$0.23}  & 82.19\tiny{$\pm$0.31}  &    11.30\tiny{$\pm$0.58}  &    72.31\tiny{$\pm$0.47}  &   27.02\tiny{$\pm$0.59}   &   16.72\tiny{$\pm$0.09}  &   77.60\tiny{$\pm$0.05}   &   80.84\tiny{$\pm$0.24}  &   10.81\tiny{$\pm$0.43}  &   75.36\tiny{$\pm$0.74}  &   26.29\tiny{$\pm$0.78}   &  19.93\tiny{$\pm$0.05} \\
\proj~I (log-pooling) &  78.28\tiny{$\pm$0.65}  &  81.60\tiny{$\pm$0.16} & 
 7.51\tiny{$\pm$0.64} & 72.51\tiny{$\pm$0.60} &  28.10\tiny{$\pm$0.33} &  16.07\tiny{$\pm$0.20} & 75.02\tiny{$\pm$0.35} &  78.83\tiny{$\pm$0.18} &  4.07\tiny{$\pm$0.21} & 75.18\tiny{$\pm$0.62} &  28.43\tiny{$\pm$0.72} &  20.26\tiny{$\pm$0.22} \\
\proj~I (calibrated) & 87.17\tiny{$\pm$0.25}&   86.18\tiny{$\pm$0.14} & 11.06\tiny{$\pm$0.18}  & 78.31\tiny{$\pm$0.09} &  35.19\tiny{$\pm$0.41}  & 12.11\tiny{$\pm$0.06} &   85.93\tiny{$\pm$0.23} & 87.19\tiny{$\pm$0.05} &  11.86\tiny{$\pm$0.13}  & 81.41\tiny{$\pm$0.17} &  37.21\tiny{$\pm$0.44}  & 11.82\tiny{$\pm$0.09} \\
\bottomrule

\end{tabular}%
}
\end{minipage}
\end{table*}

\subsection{Scenario~III: Heterogeneous Models with Private Contexts}
This setting combines Scenario~I and Scenario~II. 
We use all four LLMs for evaluation on PubMedQA. 
We consider two settings: (1) one randomly selected model receives the relevant context, as in Scenario~I; and (2) one randomly selected model receives the relevant context while another randomly selected model receives an irrelevant context (an abstract randomly drawn from the test set). 
The results are provided in Table~\ref{tab:hetero} and~\ref{tab:hetero_subset}.

StackedGen and \proj stand out as the best-performing methods. RouteLLM (with context) and PackLLM are competitive when only one relevant context is present---context availability is easy to detect from input tokens---but their advantage diminishes when an irrelevant context is also introduced, as the method must now distinguish relevant from misleading context. 


The wager network in this scenario faces the most demanding learning among the three: it must jointly assess the model's intrinsic competence, the relevance of any received context, and how much that context improves the model's comparative advantage over the other participants in the pool.
\subsection{Robustness to Variants, Advantage--Wager Alignment, and Calibration}\label{sec:additional_exp}
\paragraph{Robustness to mechanism variants.}
We evaluate \proj~II and log pooling as alternative variants. 
Both lead to results similar to \proj~I with linear pooling. 
This indicates that the performance of \proj is not driven by a particular choice of baseline or aggregation rule. 
The similarity between \proj~I and \proj~II is also consistent with theory: both variants preserve advantage--wager alignment, differing mainly in the normality/no-arbitrage tradeoff (Section~\ref{sec:variants}). Likewise, the similarity between linear pooling and log pooling suggests that the learned wagers, rather than the aggregation rule, are the primary driver of the aggregation quality. 

\begin{figure}[ht]
    \centering
    \begin{subfigure}[b]{0.49\textwidth}
        \centering
        \includegraphics[width=\textwidth]{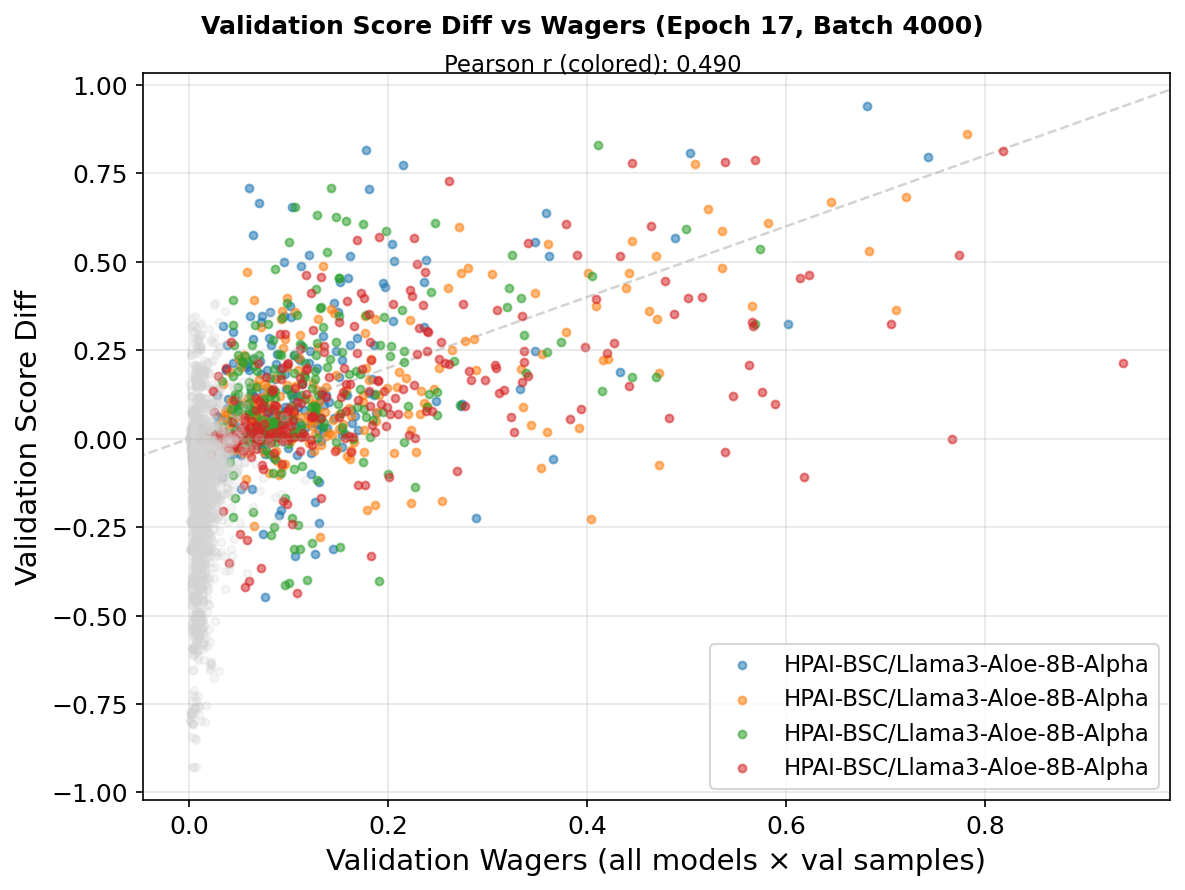}
    \end{subfigure}
    \hfill 
    \begin{subfigure}[b]{0.49\textwidth}
        \centering
        \includegraphics[width=\textwidth]{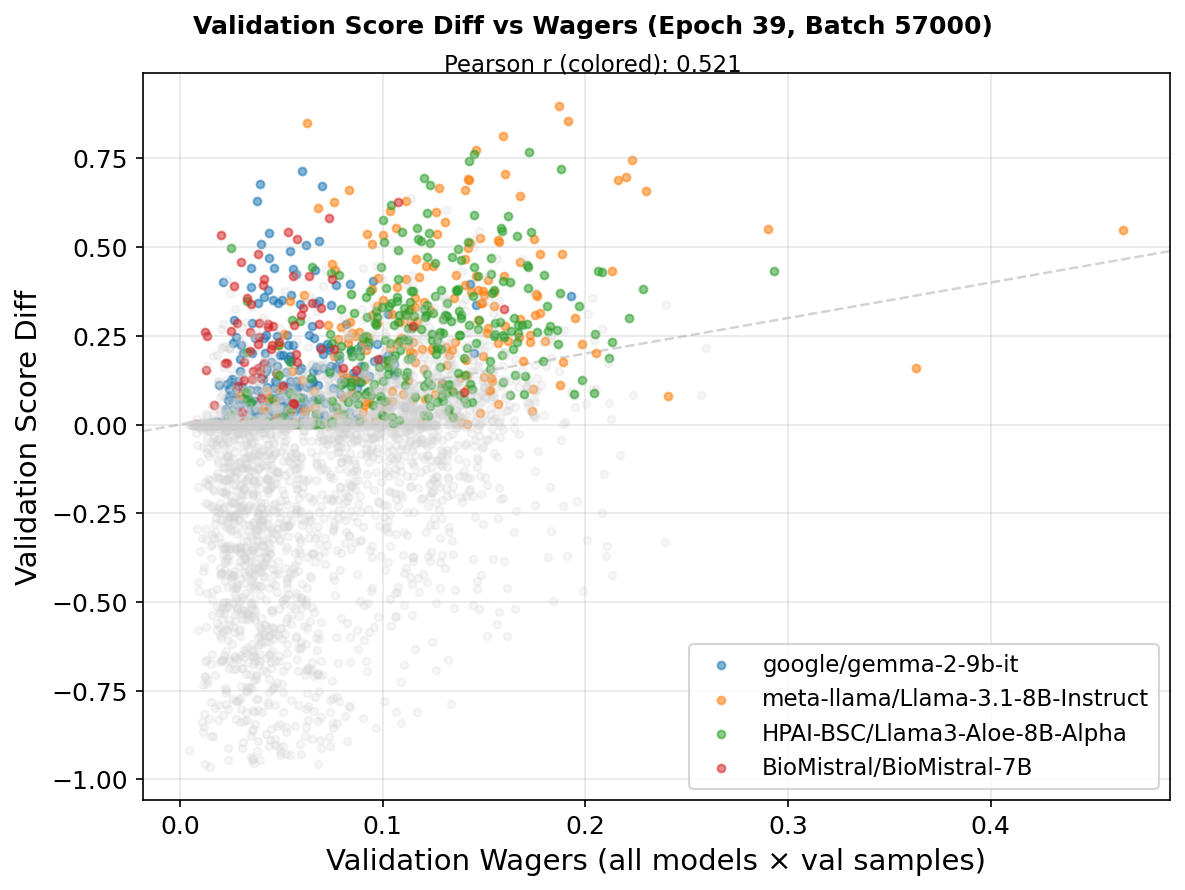}
    \end{subfigure}
    \hfill
    \begin{subfigure}[b]{0.49\textwidth}
        \centering
        \includegraphics[width=\textwidth]{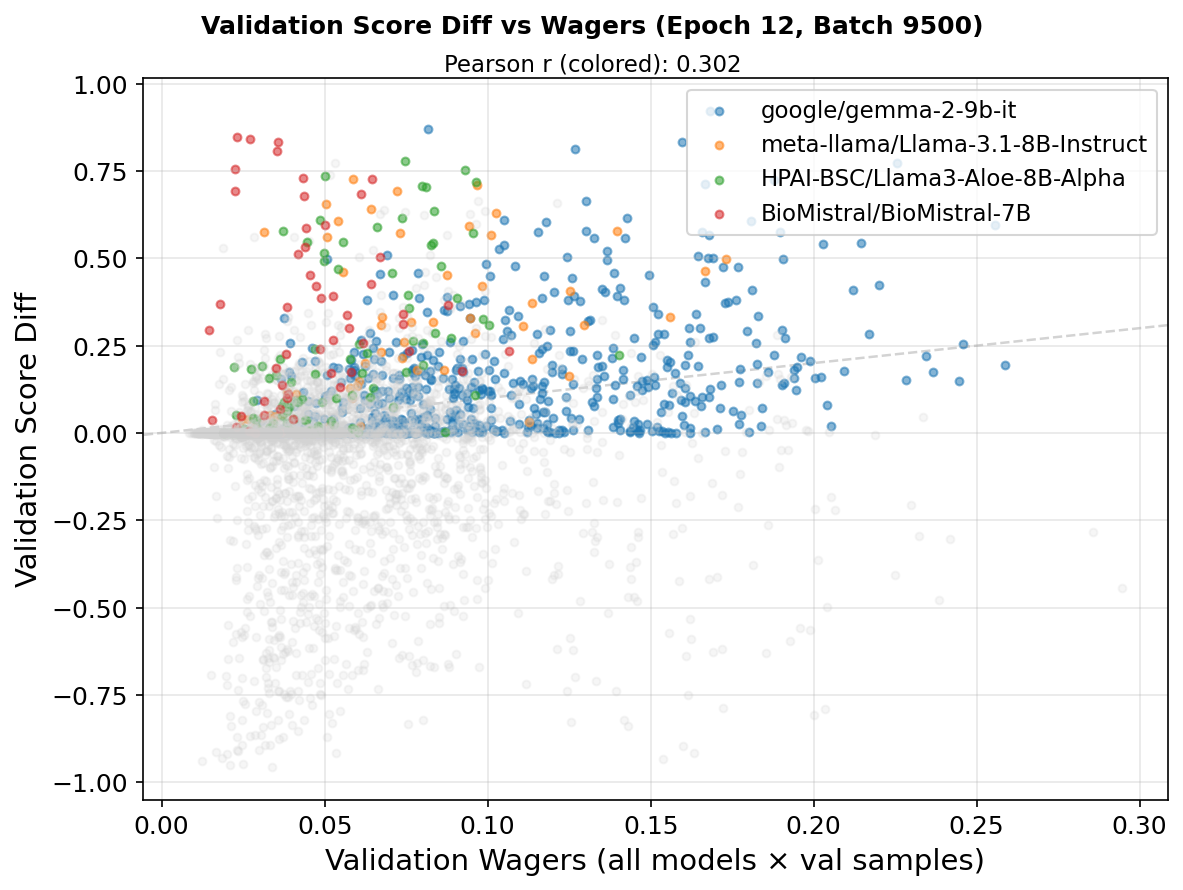}
    \end{subfigure}
    \hfill
    \begin{subfigure}[b]{0.49\textwidth}
        \centering
        \includegraphics[width=\textwidth]{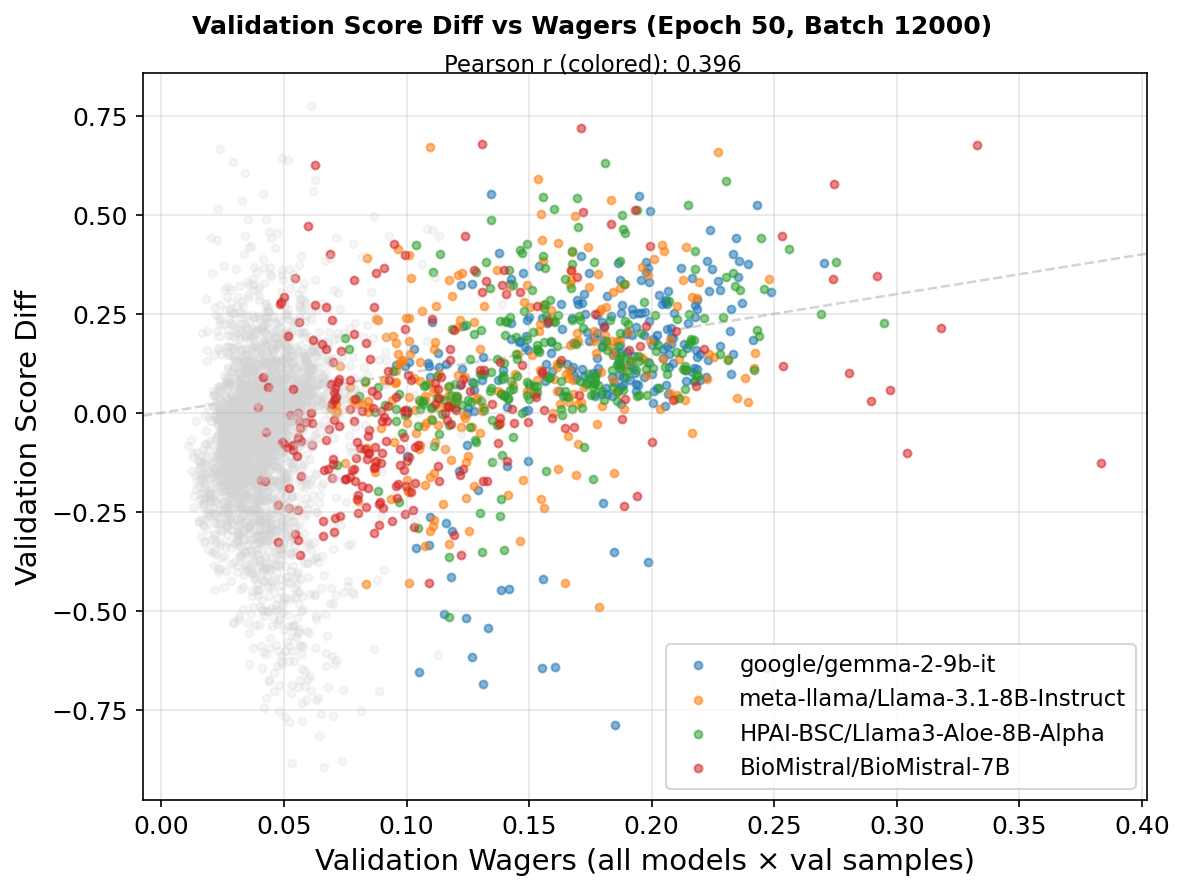}
    \end{subfigure}
    \caption{Advantage--wager alignment for \proj~I. 
    Learned wager ($x$-axis) vs.~realized score differentials $s(p_i,y) - b_{-i}$ ($y$-axis), which equal the hindsight optimal wagers under our parameter choice ($c_1=1, c_2=1/2$, $c_3=1/2$).
    For Figures~\ref{fig:alignment}--\ref{fig:estimated_alignment}, each panel corresponds to one scenario: homogeneous models on PubMedQA (top-left), heterogeneous models on MedMCQA (top-right), MMLU (bottom-left), and PubMedQA (bottom-right).
    Colored dots highlight models with context (PubMedQA) or the best-performing model (other datasets).
    Pearson correlations are reported at the top of each panel. The positive correlations confirm that LLMs learn to wager in proportion to their realized advantage. On PubMedQA(top-left and bottom-right), models successfully differentiate whether they received relevant context, wagering near zero when they did not.}
    \label{fig:alignment}
\end{figure}


\paragraph{Advantage--wager alignment.}
Fig.~\ref{fig:alignment} shows that learned wagers are positively correlated with realized comparative advantages, providing evidence that models learn to wager in proportion to their expected score advantage. 
We conduct further experiments to verify that the hidden states contain sufficient information to separately predict both the model's own score and the baseline score; we find that the difference between these predictions aligns closely with the learned wager (see Figures~\ref{fig:own_score}--\ref{fig:estimated_alignment} in the appendix). 

\begin{table*}[ht]
\centering
\caption{Average wagers and net payouts in Scenario~II when no model is calibrated v.s.~when Gemma2 is calibrated. When Gemma2 is calibrated, it increases net payout on both MedMCQA and MMLU. On MedMCQA, as it does not perform as well as Llama3.1 and Aloe, the learned wagers are reduced; on MMLU, it increases the wagers, as it performs best. Both wagers and net payouts $\times 100$.}\label{tab:single_calibrated}
 \resizebox{0.8\linewidth}{!}{%
\begin{tabular}{l
                    *{4}{c} 
                   *{4}{c}}
\toprule
& \multicolumn{4}{c}{\textbf{MedMCQA}} & \multicolumn{4}{c}{\textbf{MMLU}} \\
\cmidrule(lr){2-5} \cmidrule(lr){6-9}
& \multicolumn{2}{c}{\textbf{All Uncalibrated}} & \multicolumn{2}{c}{\textbf{Gemma2 Calibrated}} & \multicolumn{2}{c}{\textbf{All Uncalibrated}} & \multicolumn{2}{c}{\textbf{Gemma2 Calibrated}} \\
\cmidrule(lr){2-3} \cmidrule(lr){4-5} \cmidrule(lr){6-7} \cmidrule(lr){8-9}
\textbf{Method} & Wagers & Net Payouts & Wagers & Net Payouts & Wagers & Net Payouts & Wagers & Net Payouts\\
\midrule
\textbf{Gemma2} & 5.87\tiny{$\pm$0.29} & -0.93\tiny{$\pm$0.05} &  5.13\tiny{$\pm$0.42} &  -0.45\tiny{$\pm$0.05} & 16.30\tiny{$\pm$8.74} & -0.52\tiny{$\pm$1.19} & 16.62\tiny{$\pm$5.81} & -0.28\tiny{$\pm$0.67} \\
Llama3.1 & 5.87\tiny{$\pm$0.29} & 0.23\tiny{$\pm$0.02} & 8.45\tiny{$\pm$0.26} & 0.09\tiny{$\pm$0.01} & 8.59\tiny{$\pm$3.21}  & -0.30\tiny{$\pm$0.31} & 8.29\tiny{$\pm$4.03} &  -0.39\tiny{$\pm$0.38}\\
Aloe & 10.95\tiny{$\pm$0.07} & 0.26\tiny{$\pm$0.02} & 9.38\tiny{$\pm$0.35} & 0.10\tiny{$\pm$0.02} & 7.96\tiny{$\pm$2.96} & -0.45\tiny{$\pm$0.29} & 7.80\tiny{$\pm$2.98}& -0.52\tiny{$\pm$0.33}  \\
BioMistral & 2.96\tiny{$\pm$0.17} & -0.54\tiny{$\pm$0.03} & 2.36\tiny{$\pm$0.11} & -0.46\tiny{$\pm$0.02} & 4.76\tiny{$\pm$1.41} & -0.69\tiny{$\pm$0.25} & 4.14\tiny{$\pm$1.40} & -0.62\tiny{$\pm$0.24}  \\
\bottomrule

\end{tabular}%
}
\end{table*}

\paragraph{Calibration.}
The mechanism also incentivizes participants to calibrate their predictions: even when a model's belief $\Q_i \neq \mathcal{P}$, calibration aligns expected scores better with realized scores, so that the learned wager better reflects \textit{actual} comparative advantage. 
We conduct calibration via temperature scaling~\citep{Xie2024Calibrating}, which predicts an appropriate temperature conditional on the last-layer hidden state.
In Table~\ref{tab:single_calibrated}, we show that when one model is calibrated (Gemma2) while the others remain unchanged, the calibrated model receives improved net payoffs because its expected score advantage more accurately reflects its realized comparative advantage. 
This suggests that rational participants have an incentive to calibrate---discouraging both overconfidence and underconfidence. 
We also report \proj~I results when all participants are calibrated for Scenarios~II and~III. 
The results show that calibrated aggregation can substantially improve accuracy on in-distribution tasks while maintaining similar performance on out-of-distribution tasks.

\if 1
\newpage
In this section, we first introduce the optimization procedure that participants use to learn their wagering strategies. We then conduct multi-agent simulations and compare the performance of the aggregated predictions with baselines under three scenarios: (1) homogeneous models with private contexts, (2) heterogeneous models without additional contexts, and (3) heterogeneous models with private contexts.
Finally, we conduct additional analyses on alternative aggregation variants, on advantage--wager alignment at convergence, and on calibrating predictions as a way to help align beliefs with ground truth, which can lead to better payout.

Throughout the experiment, we set the constants $c_1,c_2$ and, $c_3=1$ such that $w^*=((s_i-b_{-i})/2)_{+}$. We focus on the Brier score (Def.~\ref{def:brier}). 

\xtw{TODO: table in appendix on all mechanism parameters and learning parameters.}


\subsection{An Optimization Procedure for Participants of \proj}

%
We assume that the predictions and wagers of all participants are revealed along with the outcomes. Hence, the optimal wager in hindsight, $w^*$, can be derived.
%
%
As discussed in Sec.~\ref{sec:learn_wager}, it is more reasonable to use feedback from \proj to learn wagers rather than to improve predictions. Thus, each participant fixes the LLM parameters while updating the wager network.
Each participant extracts the last-layer hidden state after processing all input tokens and passes it to a two-layer perceptron (MLP) to output a scalar wager. Since $w^*$ is bounded in $[0,1]$, we apply a sigmoid transformation to the wager-head output to preserve the feasible range while improving numerical stability. All participants' wager networks are trained simultaneously without assuming that opponents' strategies are fixed. The wager networks are trained using the mean squared error (MSE) between the predicted wager $w$ and the hindsight target $w^*$. This optimization is equivalent to minimizing the regret of playing optimally in each round when $w^*>0$. 
When $w^*=0$, the MSE update toward zero is less aggressive than directly minimizing regret, since the penalty is quadratic rather than linear.

\subsection{Experiment Setup}

We focus on prediction tasks, including multiple-choice question answering and event forecasting. We limit the LLM output to one token and convert token logits into a probability distribution over the outcome space. We include participants from a heterogeneous pool of general models, Gemma-2-9B (Gemma2)~\citep{gemmateam2024gemma2} and Llama3.1-8B (Llama3.1)~\citep{grattafiori2024llama3herdmodels}, as well as domain-fine-tuned models, Llama3-Aloe-8B (Aloe)~\citep{gururajan2024aloe} and BioMistral-7B (BioMistral)~\citep{labrak-etal-2024-BioMistral}.
%
For every dataset, we use an $8{:}1{:}1$ train-validation-test split and fix the batch size to 100.  For learning-based methods, we use the Adam optimizer, tune the learning rate with grid search and adopt early stopping to prevent overfitting.
For baselines, we stop at the epoch that achieves the best validation performance. For \proj, because the training procedure is decentralized, we use the Brier score on the most recent training batches as the early stopping criterion, ensuring that the number of training-batch samples used for early stopping matches the validation-set size. Unless specified otherwise, all methods use linear pooling for aggregation.
 In all experiments, the aggregated predictions are evaluated on a held-out test set, where no further updates to the wagering strategy are allowed. 



\subsubsection{Metrics}

In addition to predictive performance metrics, including AUC and ACC, and expected calibration error (ECE), we evaluate whether the aggregation weights identify high-performing experts on each question. Since \proj is derived from a Brier-score wagering mechanism, our mechanism-specific metrics are defined in terms of the Brier score. In particular, Kendall's Tau, MRR, and dynamic regret all measure whether the learned wagers align with per-question Brier-score performance.

\begin{definition} [Kendall's Tau (K-Tau)] 
We use K-Tau to measure the ordinal association between two orderings: (1) the ordering of models by their per-question Brier-score performance and (2) the ordering of models by their wagers. Formally,
\[
\tau = \frac{\text{(number of concordant pairs) - (number of discordant pairs)}}{\text{(number of pairs)}}.
\]
\end{definition}

\begin{definition} [Mean Reciprocal Rank (MRR) of the Best Experts]
The per-sample reciprocal rank is the inverse rank of the best expert in terms of Brier score. Formally,
\[
\text{MRR} = \frac{1}{N} \sum_{n=1}^N 
\frac{1}{\text{rank}(\arg\min_j \|\delta^{(n)}_y - p^{(n)}_j\|_2^2)},
\]
where $N$ is the number of samples and the rank is induced by the aggregation weights.
\end{definition}

\begin{definition} [Dynamic Regret (D-Regret) in Brier Score]
The D-Regret is defined as
\[
\frac{1}{N} \sum_{n=1}^N 
\|\delta_y^{(n)} - p_{\text{agg}}^{(n)}\|_2^2 
- 
\min_j \|\delta_y^{(n)} - p_j^{(n)}\|_2^2.
\]
It measures the Brier-score gap between the aggregated prediction and the best individual model on each question.
\end{definition}

All reported intervals ($\pm$) in the result tables are $95\%$ confidence intervals over five independent runs.

    





\subsubsection{Baselines}\label{sec:baselines}

We consider seven baselines spanning pre-inference routers that assume black-box LLM access~\citep{chen2024routerdc, song2025irt, ong2025routellm}, learning-free ensembling~\citep{mavromatis2024pack,kang2025scalable,wang2022selfconsistency}, and learning-based ensembling with LLM hidden states~\citep{WOLPERT1992241}. 

Among the learning-free methods, uniform averaging (UniformAvg) performs equally weighted aggregation. SelfCertainty~\citep{kang2025scalable} uses the predicted distribution to estimate confidence, for example via entropy or KL divergence to a uniform distribution, and treats this confidence as the aggregation weight. PackLLM~\citep{mavromatis2024pack} uses the inverse of the LLM-evaluated perplexity on the input tokens as weights, which requires access to model parameters. These aggregations can be effective when the estimated confidence is positively correlated with model performance. Furthermore, all learning-free approaches can be interpreted as decentralized aggregation if we treat the predictions and, for PackLLM, the estimated perplexities as agents' reports. However, they may be sensitive to overconfidence or underconfidence across models, and the agents' reports are not incentive-compatible.

Learning-based methods, including pre-inference routers~\citep{chen2024routerdc,ong2025routellm, song2025irt} and stacked generalization (StackedGen)~\citep{WOLPERT1992241}, require centralized control of the participants to learn routing or aggregation networks. Pre-inference routers learn aggregation weights conditional only on the provided contexts, assuming black-box LLM access.\footnote{While these routers are generally used to select the best LLM by taking the argmax of the weights, we adapt them for soft weighted aggregation to be consistent with other methods.} Their predicted weights may capture a model's general capability conditional on question features such as topic, format, or context availability, but they do not capture each model's internal uncertainty on the specific question. In contrast, StackedGen concatenates the LLM hidden states as input to train a supervised aggregation network.

In comparison, \proj is a decentralized mechanism that incentivizes LLMs to report predictions truthfully (Thm.~\ref{thm:dsic}) and to wager according to their expected Brier-score advantage over the leave-one-out baseline (Thm.~\ref{thm:alignment}).

We provide a more detailed discussion of the baselines in Appendix~\ref{apx:baselines}.

\subsubsection{Datasets}

Our experiments consider six datasets: PubMedQA~\citep{jin-etal-2019-PubMedQA}, MedMCQA~\citep{MedMCQA}, MMLU~\citep{hendrycks2021measuring}, ARC-Challenge~\citep{clark2018think}, ForecastQA~\citep{jin-etal-2021-forecastqa}, and a newly created forecasting dataset called BayesX.

For PubMedQA and BayesX, the input tokens consist of two components: a base question and a private context, such as a publication abstract or an observed signal that may or may not help answer the question. In our experiments, we expose the private context to random subsets of participants.

BayesX is a synthetic Bayesian forecasting dataset where the base question contains a prior distribution for an event to be forecasted, the private context exposes the true and false positive rate with the evidence and the ground truth is the posterior probability distribution derived with Bayes theorem. Thus, BayesX captures a setting in which additional information reduces epistemic uncertainty about the data-generating process while increasing aleatoric uncertainty about the realized outcome.

Further details about the datasets are provided in Appendix~\ref{apx:datasets}.


\subsection{Scenario~I: Homogeneous Models with Private Contexts}


In this setting, all participants use the same LLM. For each question, only one randomly selected model is provided with the relevant context in the form of additional input tokens, which acts as private information. We use Aloe on PubMedQA and Gemma2 on BayesX. 
For BayesX, we use expected loss for all learning-based methods and evaluate KL divergence and TV distance between the predicted distribution and the ground-truth posterior distribution. MRR and D-Regret are evaluated based on expected Brier scores.
We scale the number of participants from 4 to 12.

On PubMedQA, \proj achieves performance similar to RouteLLM and StackedGen in terms of AUC, ACC, MRR, and D-Regret while outperforming other pre-inference routers and learning-free methods (Table~\ref{tab:homo}). Although \proj is trained through decentralized wager heads, the aggregated performance on the training batches increases steadily throughout training (Fig.~\ref{fig:perf_overtime} in the appendix), resembling the optimization behavior of centralized supervised learning. In terms of ECE, while \proj is slightly less calibrated than the strongest baselines on the full dataset, it matches or improves upon baselines on the subset where the model without context predicts incorrectly (Table~\ref{tab:homo_subset} in the appendix). As the number of participants increases, only \proj, StackedGen, and RouteLLM maintain similar performance, while all learning-free methods become less robust to the enlarged participation pool.

On BayesX, \proj is less effective than RouteLLM and StackedGen, but it outperforms the remaining baselines in KLD, TVD, and D-Regret (Table~\ref{tab:bayex}). RouteLLM and PackLLM achieve particularly high performance, including perfect MRR. This suggests that, in this setting, it is relatively easy to infer which model has access to the relevant context from the input tokens alone. However, this advantage may not generalize when all models receive contexts of varying relevance. Indeed, when RouteLLM is evaluated without access to the contexts, its performance becomes similar to UniformAvg. SelfCertainty performs poorly on BayesX because its confidence estimate is inversely correlated with the relevant notion of epistemic certainty: observing the signal makes the posterior more faithful, but also closer to uniform.

\subsection{Scenario~II: Heterogeneous Models with No Context}

In this setting, we use heterogeneous LLMs, including Gemma2, Aloe, Llama3.1, and BioMistral. We evaluate in-distribution performance on MedMCQA and MMLU, and out-of-distribution performance on ARC-Challenge and ForecastQA using MMLU as the training set. We present results with two participants in Tables~\ref{tab:2models} and~\ref{tab:2models_subset} in the appendix, and results with four participants in Tables~\ref{tab:4models} and~\ref{tab:4models_subset}.

Except for Gemma2 on MMLU and ARC-Challenge, the aggregated predictions of the baselines outperform individual LLMs under both the two- and four-participant settings, confirming the value of aggregation. With four participants, StackedGen, RouteLLM, NIRTRouter, and \proj are the strongest-performing methods across datasets.

On MedMCQA, adding two less capable models, Gemma2 and BioMistral, leads to noticeable performance drops for learning-free methods, including UniformAvg, PackLLM, and SelfCertainty, whereas supervised methods and \proj maintain stronger performance. Since the learning-free methods rely solely on self-estimated confidence or context perplexity, systematic overconfidence in weaker models can lead to excessive weight being assigned to subpar predictions. This behavior illustrates a practical limitation of confidence-based aggregation: even without strategic manipulation, miscalibrated confidence can increase a weak model's influence on the aggregate.

\subsection{Scenario~III: Heterogeneous Models with Private Contexts}
This setting combines Scenario~I and Scenario~II. We use all four LLMs for evaluation on PubMedQA. We consider two settings: (1) one randomly selected model is assigned the relevant context, as in Scenario~I; and (2) one randomly selected model is assigned the relevant context while another randomly selected model is assigned an irrelevant context, that is, an abstract randomly drawn from the test set. The results are provided in Table~\ref{tab:hetero} and~\ref{tab:hetero_subset}.

StackedGen and \proj stand out as the best-performing methods. When there is one relevant context and no irrelevant context, RouteLLM with context and PackLLM outperform the remaining baselines. Similar to the BayesX results, this likely occurs because context availability is easy to detect from the input tokens, allowing these methods to assign a high weight to the model that received the context. However, their advantage diminishes when an irrelevant context is also introduced. In this harder setting, detecting the presence of additional context is no longer sufficient: the aggregation method must also distinguish relevant context from misleading or irrelevant context.

\subsection{Additional Analyses}

We evaluate \proj~II and log pooling as alternative variants. Both lead to results similar to \proj~I with linear pooling, suggesting that the empirical performance of \proj is not driven by a single choice of leave-one-out baseline or pooling rule. The similarity between \proj~I and \proj~II is also consistent with the theory: both variants preserve truthful prediction reporting and advantage--wager alignment, while differing mainly in their mechanism-design tradeoff between normality and no-arbitrage. Likewise, the similarity between linear pooling and log pooling suggests that the learned wagers, rather than the precise pooling rule, are the primary driver of the aggregation behavior in our experiments.

\paragraph{Advantage--wager alignment.}
We empirically observe advantage--wager alignment in Fig.~\ref{fig:alignment}. The learned wagers are positively correlated with the comparative advantages revealed in hindsight, and we report the Pearson correlation in the figure. This provides direct evidence for the mechanism's intended behavior: models learn to wager more when they are likely to outperform the leave-one-out baseline.

We further train a predictor for each model's own score (Fig.~\ref{fig:own_score}) and a separate predictor for the leave-one-out baseline score (Fig.~\ref{fig:baseline}) using the same model hidden states. We then plot the difference between the predicted own score and the predicted baseline score against the learned wager (Fig.~\ref{fig:estimated_alignment}). The resulting alignment remains strong even though the wager, own-score predictor, and baseline-score predictor are trained with separate heads. This suggests that the hidden states contain sufficient information to estimate both the model's own expected Brier score and its comparative advantage over the aggregation baseline. In future work, we will explore more powerful predictors, including fine-tuning the LLMs, to further strengthen the alignment between learned wagers and ground-truth score differences.



\paragraph{Calibration.}Calibration can be interpreted as updating a participant's expected score so that it better matches the participant's realized score. We conduct calibration via temperature scaling~\citep{Xie2024Calibrating}. In Table~\ref{tab:single_calibrated}, we show that when one model is calibrated while the other models remain unchanged, the calibrated model receives improved net payoffs because its expected score advantage becomes more faithful to its realized advantage. This suggests that rational participants have an incentive to calibrate their predictions.
We also report \proj~I results when all participants are calibrated for Scenario~II and Scenario~III. The results show that calibrated aggregation can substantially improve accuracy on in-distribution tasks while maintaining similar performance on out-of-distribution tasks.
\fi
\section{Conclusion}
We proposed a family of \emph{advantage-aligned wagering mechanisms for LLM aggregation} (WALLA) that yields three properties: (1)~dominant-strategy incentive compatibility under general belief structures, (2)~advantage--wager alignment, and (3)~the decoupling of prediction and wager learning.
To our knowledge, this is the first mechanism designed for aggregating LLM predictions that simultaneously achieves advantage-weighted aggregation, uncertainty awareness, fully decentralized learning, and incentive-compatibility guarantees. Empirical results suggest that \proj matches centralized baselines in aggregated prediction performance.

So far, we explored learning a neural network for wagers. As state-of-the-art LLMs already exhibit the capability of gameplay~\citep{Bills2025Improving}, the natural next step would be to explore whether learning-in-context can be viable for proprietary LLMs to estimate and report the model's comparative advantage. Future work can also explore mechanisms for more complex collaboration settings that involve multi-step reasoning beyond the aggregation of probabilistic predictions.



\newpage
\bibliographystyle{unsrtnat}
\bibliography{mybib}
\newpage
\appendix

\section{Theoretical Properties of WSWM Using a Scaled Linear Score}\label{apx:warmup}
Consider the repeated setting where question $x^{(t)}$ arrives at time $t$ with outcome $y^{(t)}$.
Suppose all participants are non-strategic: reporting the predictions $p_i^{(t)}$ truthfully and wagering the total cumulative payout, namely, their budget, denoted as $B_i^{(t-1)}$.
Consider a WSWM that adopts a scaled linear score $s_{\text{lin}} \in \mathcal{S}: (\Delta(\mathcal{Y})\times \mathbb{R}_{\ge 0})^M$ such that  \[s_i^{(t)} = \frac{p_i^{(t)}[y]}{\sum_{j=1}^M w_j^{(t)} p_j^{(t)}[y]}\] and the linear pooling (i.e., $p^{(t)}_{\text{agg}} = \frac{\sum_{i=1}^M w_i^{(t)} p_i^{(t)}}{\sum_{i=1}^M w_i^{(t)}}$).
We first demonstrate that such a mechanism inherently performs an exponentiated gradient-descent step on the cross-entropy loss. We then show the mechanism exactly recovers Bayesian Model Averaging (BMA) \citep{hoeting1999bayesian}. Finally, we derive a tighter regret bound than the one shown in \citet{Freeman2020no}, highlighting a tradeoff between this tight bound and incentive compatibility.

For notation simplicity, we remove the timestamp and simplify $p_{\text{agg}}[y]$ as $p_{\text{agg}}$ and $p_{i}[y]$ as $p_{i}$ when the context is clear.




\subsection{WSWM Recovers the Scaled Negative Gradient of Cross-Entropy Loss}

We connect this mechanism to centralized supervised learning by analyzing the cross-entropy loss of the aggregated prediction. 


\begin{theorem}
Suppose WSWM utilizes the scoring function $s_{\text{lin}}$. The net payout function, defined as $\pi_i = w_i \left( s_i - \frac{\sum_{k=1}^M s_k w_k}{\sum_{k=1}^M w_k} \right)$, satisfies $\pi_i = -w_i \frac{\partial L}{\partial w_i}$ where $L = - \log p_{\text{agg}}$, the cross-entropy loss.
\end{theorem}

\begin{proof}

Expanding the logarithm of the quotient yields the loss function:
\[
L = -\log \left( \sum_{i=1}^M w_i p_i \right) + \log \left( \sum_{i=1}^M w_i \right).
\]

Taking the partial derivative of $L$ with respect to a specific weight $w_i$ gives:
\[
\frac{\partial L}{\partial w_i} = -\frac{p_i}{\sum_{k=1}^M w_k p_k} + \frac{1}{\sum_{k=1}^M w_k}.
\]

 Substituting the given scoring rule $s_k$ into $\sum_{k=1}^M s_k w_k$ gives:
\begin{align*}
\sum_{k=1}^M s_k w_k &= \sum_{k=1}^M \left( \frac{p_k}{\sum_{j=1}^M w_j p_j} \right) w_k \\
&= \frac{\sum_{k=1}^M w_k p_k}{\sum_{j=1}^M w_j p_j} \\
&= 1.
\end{align*}

We substitute this result back into the net payout function $\pi_i$:
\begin{align*}
\pi_i &= w_i \left( s_i - \frac{1}{\sum_{k=1}^M w_k} \right) \\
&= w_i \left( \frac{p_i}{\sum_{j=1}^M w_j p_j} - \frac{1}{\sum_{k=1}^M w_k} \right)\\
 &= -w_i \left( -\frac{p_i}{\sum_{j=1}^M w_j p_j} + \frac{1}{\sum_{k=1}^M w_k} \right).
\end{align*}

We observe that the expression inside the parentheses is exactly the partial derivative $\frac{\partial L}{\partial w_i}$ derived previously. Therefore, by direct substitution:
\[
\pi_i = -w_i \frac{\partial L}{\partial w_i}.
\]
\end{proof}

\begin{remark} This relationship demonstrates that the budget update $B_i^{(t)} = w_i^{(t)} + \pi_i^{(t)} = w_i^{(t)} \left(1 - \frac{\partial L^{(t)}}{\partial w_i^{(t)}}\right) = w_i^{(t)} - w_i^{(t)} \frac{\partial L^{(t)}}{\partial w_i^{(t)}}$ has the form of a Multiplicative Weights Update (MWU) algorithm~\citep{v008a006}. It can be interpreted as executing decentralized gradient descent where the learning rate is scaled dynamically by the current weight. It also mirrors the exponentiated gradient step $B_i^{(t)} \approx w_i^{(t)} e^{-\frac{\partial L^{(t)}}{\partial w_i^{(t)}}}$ under a first-order Taylor approximation.
\end{remark}

\subsection{Equivalence to Bayesian Model Averaging}

We now show that the WSWM under this setting is equivalent to Bayesian model averaging (BMA)~\citep{hoeting1999bayesian}. We start with the review of BMA.

\begin{definition}[Bayesian Model Averaging (BMA)] \label{def:bma}
Let $\theta_1, \ldots, \theta_M$ be discrete parameters for the $M$ models. Given a history of observations $D^{(1:t-1)} = \{(x^{(1)}, y^{(1)}), \ldots, (x^{(t-1)}, y^{(t-1)})\}$, the posterior probability of parameter $\theta_i$ is $p(\theta_i|D^{(1:t-1)})$. Assuming $(x^{(t)},y^{(t)})$ are drawn i.i.d., the BMA predictive probability for the true label $y^{(t)}$ is:
\[
p_{\text{BMA}}(y^{(t)}|x^{(t)}, D^{(1:t-1)}) = \sum_{i=1}^M p(y^{(t)}|x^{(t)},\theta_i)p(\theta_i|D^{(1:t-1)}).
\]
\end{definition}

To prove that the mechanism recovers the exact posterior, we assume that the BMA predictive distribution perfectly recovers the true marginal probability of the label.

\begin{theorem} \label{thm:bma_recovery}
Assume the BMA predictive distribution equals the true marginal data distribution:\\ $p_{\text{BMA}}(y^{(t)}|x^{(t)}, D^{(1:t-1)}) = p(y^{(t)}|x^{(t)}, D^{(1:t-1)})$. If the initial budgets are set to the prior probabilities $B_i^{(0)} = p(\theta_i)$, the aggregated prediction satisfies $p^{(t)}_{\text{agg}} = p_{\text{BMA}}(y^{(t)}|x^{(t)}, D^{(1:t-1)})$ for all $t \ge 1$, and the budgets track the exact posterior $B_i^{(t)} = p(\theta_i|D^{(1:t)})$.
\end{theorem}


\begin{proof}
We proceed by induction. The base case holds by initialization $B_i^{(0)} = p(\theta_i)$. Assume the inductive hypothesis $B_i^{(t-1)} = p(\theta_i|D^{(1:t-1)})$. Because $w_i^{(t)} = B_i^{(t-1)}$, substituting the hypothesis into the aggregated prediction yields $p^{(t)}_{\text{agg}} = \sum_{j=1}^M p(\theta_j|D^{(1:t-1)}) p(y^{(t)}|x^{(t)},\theta_j)$, which exactly matches Definition \ref{def:bma}.

Evaluating the budget update step established previously:
\begin{align*}
    B_i^{(t)} &= w_i^{(t)} s_i^{(t)} \quad= w_i^{(t)} \frac{p_i^{(t)}}{\sum_{j=1}^M w_j^{(t)} p_j^{(t)}} \quad= B_i^{(t-1)} \frac{p_i^{(t)}}{p^{(t)}_{\text{agg}}} \\
    &= \frac{p(\theta_i|D^{(1:t-1)}) p(y^{(t)}|x^{(t)},\theta_i) }{p_{\text{BMA}}(y^{(t)}|x^{(t)}, D ^{(1:t-1)})} \\
    &= \frac{p(\theta_i|x^{(t)}, D^{(1:t-1)}) p(y^{(t)}|x^{(t)}, D ^{(1:t-1)},\theta_i) }{p_{\text{BMA}}(y^{(t)}|x^{(t)}, D ^{(1:t-1)})} \quad \text{(conditional independence)} \\
    &= \frac{p(\theta_i|x^{(t)}, D^{(1:t-1)}) p(y^{(t)}|x^{(t)}, D ^{(1:t-1)},\theta_i) } {p(y^{(t)}|x^{(t)}, D ^{(1:t-1)})} \quad \text{(by assumption)} \\
    &= p(\theta_i|y^{(t)}, x^{(t)}, D^{(1:t-1)})  \quad \text{(Bayes' theorem)} \\ 
    &= p(\theta_i|D^{(1:t)}).
\end{align*}
This completes the proof.
\end{proof}

As $p_{\text{BMA}}$ cannot recover the true marginal exactly, its effectiveness is typically quantified with the notion of regret. We provide the regret bound in the next section.

\subsection{Regret Bound}

We continue to assume the initial budgets sum to one and derive a regret bound for the WSWM with non-strategic agents under $s_{lin}$ (i.e., $s_i^{(t)} = \frac{p_i^{(t)}}{\sum_{j=1}^M w_j^{(t)} p_j^{(t)}}$). 

\begin{theorem}[Regret Bound]
Assume $p_i^{(t)} > 0$ for all $i, t$. The regret of the WSWM mechanism over $T$ rounds with respect to any optimal fixed distribution of models $\mathbf{u} \in \Delta(\mathcal{M})$ is bounded by:
\[
\sum_{t=1}^T L^{(t)}(\mathbf{w}^{(t)}) - \sum_{t=1}^T L^{(t)}(\mathbf{u}) \le D_{\text{KL}}(\mathbf{u} \parallel \mathbf{w}^{(0)}) + H(\mathbf{u}),
\]
where $D_{\text{KL}}$ is the Kullback-Leibler divergence and $H(\mathbf{u})$ is the Shannon entropy.
\end{theorem}

\begin{proof}

Because $\sum_{k=1}^M w_k^{(t)} = 1$ and $\sum_{k=1}^M s_k^{(t)}w_k^{(t)} = \frac{\sum_{k=1}^M w_k^{(t)} p_k^{(t)}}{\sum_{j=1}^M w_j^{(t)} p_j^{(t)}} = 1$, the net payout simplifies to:
\[
\pi_i^{(t)} = w_i^{(t)}(s_i^{(t)} - 1).
\]

Let $Z_t = \sum_{j=1}^M w_j^{(t)} p_j^{(t)}$.
A agent's updated budget for the next round is her wager plus her net payout: $w_i^{(t+1)} = B_i^{(t)} = w_i^{(t)} + \pi_i^{(t)} = w_i^{(t)} s_i^{(t)} = \frac{w_i^{(t)} p_i^{(t)}}{Z_t}$.
We expand the terminal weight $w_i^{(T+1)}$ recursively:
\[
w_i^{(T+1)} = w_i^{(0)} \prod_{t=1}^T \frac{p_i^{(t)}}{Z_t}.
\]
Taking the natural logarithm of both sides:
\[
\log w_i^{(T+1)} = \log w_i^{(0)} + \sum_{t=1}^T \log p_i^{(t)} - \sum_{t=1}^T \log Z_t.
\]
Define the instantaneous loss as $L^{(t)}(\mathbf{w}^{(t)}) = -\log Z_t$ and the loss for a fixed model $i$ as $L^{(t)}(\mathbf{e}_i) = -\log p_i^{(t)}$. 
\[
\log w_i^{(T+1)} = \log w_i^{(0)} - \sum_{t=1}^T L^{(t)}(\mathbf{e}_i) + \sum_{t=1}^T L^{(t)}(\mathbf{w}^{(t)}).
\]
Because $\mathbf{w}^{(T+1)} \in \Delta(\mathcal{M})$, we know $\log w_i^{(T+1)} \le 0$. Rearranging gives:
\[
\sum_{t=1}^T L^{(t)}(\mathbf{w}^{(t)}) - \sum_{t=1}^T L^{(t)}(\mathbf{e}_i) \le -\log w_i^{(0)}.
\]

Taking the expectation over $i \sim \mathbf{u}$:
\[
\sum_{t=1}^T L^{(t)}(\mathbf{w}^{(t)}) - \sum_{i=1}^M u_i \sum_{t=1}^T L^{(t)}(\mathbf{e}_i) \le -\sum_{i=1}^M u_i \log w_i^{(0)}.
\]
By Jensen's inequality, $-\sum_{i=1}^M u_i \log p_i^{(t)} \ge -\log \left( \sum_{i=1}^M u_i p_i^{(t)} \right) = L^{(t)}(\mathbf{u})$. Substituting this lower bound and isolating the information-theoretic quantities produces:
\[
\sum_{t=1}^T L^{(t)}(\mathbf{w}^{(t)}) - \sum_{t=1}^T L^{(t)}(\mathbf{u}) \le \sum_{i=1}^M u_i \log \frac{u_i}{w_i^{(0)}} - \sum_{i=1}^M u_i \log u_i,
\]
which simplifies to $D_{\text{KL}}(\mathbf{u} \parallel \mathbf{w}^{(0)}) + H(\mathbf{u})$.
\end{proof}

\begin{remark} The constant regret bound $D_{\text{KL}}(\mathbf{u} \parallel \mathbf{w}^{(0)}) + H(\mathbf{u})$
achieved by WSWM is tighter than the $\mathcal{O}(\sqrt{T \ln M})$ sublinear bound established by \citet{Freeman2020no}. This difference highlights a tradeoff between a tight regret bound and IC.\footnote{Although \citet{Freeman2020no} is incentive-compatible to agents who strategically report predictions, they assume a fixed wagering strategy that always wagers a fraction $\eta$ of the cumulative payout.} 
\end{remark}


\section{Deferred Content from Sec.~\ref{sec:mechanism}}
\label{app:D}

\subsection{Proof of BNE under Common Prior}\label{proof:BNE}

\bne*
 
\begin{proof}
Fix opponents' equilibrium strategies $(\mathbf{p}_{-i}^*, \mathbf{w}_{-i}^*)$.
By Theorem~\ref{thm:dsic}, $p_i^* = \Q(Y = \cdot \mid \mathcal{F}_i)$ uniquely maximizes $\mathbb{E}_{\Q}[s(p_i, Y) \mid \mathcal{F}_i]$ over all $p_i$, and hence maximizes $A_i = \mathbb{E}_{\Q}[s(p_i, Y) - b_{-i} \mid \mathcal{F}_i]$. Given $p_i^*$, Theorem~\ref{thm:swa} gives the best-response wager $w_i^*(p^*_i) = (A_i/2c_3)^+$. When $A_i > 0$, the pair $(p_i^*, w_i^*(p^*_i))$ is the unique maximizer of the expected net payout. When $A_i \leq 0$, $w_i^*(p^*_i) = 0$ and the net payout is zero regardless of $p_i$; the agent is excluded from aggregation. A fixed point of the wager best-response mapping exists by Brouwer's theorem for bounded $s$, as the mapping is continuous on $[0, (\overline{s} - \underline{s})/(2c_3)]^M$.
\end{proof}

\subsection{Proof of NE under Heterogenenous Beliefs}\label{proof:NE}

\nashe*
 
\begin{proof}
Fix opponents' equilibrium strategies $(\mathbf{p}_{-i}^*, \mathbf{w}_{-i}^*)$. 
By Theorem~\ref{thm:dsic}, $p_i^*$ maximizes $\mathbb{E}_{\Q_i}[s(p_i, Y) \mid \mathcal{F}_i]$ and hence maximizes $A_i(p_i) = \mathbb{E}_{\Q_i}[s(p_i, Y) - b_{-i} \mid \mathcal{F}_i]$ (since $b_{-i}$ depends on opponents' strategies, not on $p_i$). By Theorem~\ref{thm:swa}, for any fixed $p_i$, the net payoff is maximized at $w_i^*(p_i^*) = (A_i/2c_3)^+$. 
Since $p_i^*$ maximizes $A_i$, the pair $(p_i^*, w_i^*(p^*_i))$ maximizes the expected net payoff over all $(\tilde{p}_i, \tilde{w}_i)$ jointly. When $A_i > 0$, the pair $(p_i^*, w_i^*(p^*_i))$ is the unique maximizer of the expected net payout. A fixed point exists by Brouwer's theorem for bounded $s$.
The NE does not require information about others' beliefs because $(p^*_i, w^*_i(p_i^*))$ maximizes the expected net payoff with respect to any opponents' beliefs and profile $(\mathbf{p}_{-i}, \mathbf{w}_{-i})$.
\end{proof}

\subsection{Proofs of Properties of the Two Variants}
\label{proof:variants}

\normality*
\begin{proof}
By strict properness, $p_i^*:= \Q_i(Y = \cdot \mid \mathcal{F}_i)$ uniquely maximizes $\mathbb{E}[s(p, Y) \mid \mathcal{F}_i]$ when $w_i \ge 0$. With $w_j > 0$, replacing $p_j$ with $p'_j$ a prediction closer to $p_i^*$ increases $\mathbb{E}[s(p_j, Y) \mid \mathcal{F}_i]$, which raises the baseline $\mathbb{E}[b_{-i}^{\mathrm{I}} \mid \mathcal{F}_i]$ and hence reduces $A_i$ and the net payout. When $w_i = 0$ or $w_j=0$, the expected net payoff remains the same regardless of player $j$'s prediction. 
\end{proof}


 \arbone*
\begin{proof}
By Theorem 3.3 of~\citet{Chen2014Removing}, when at least two predictions differ ($p_j \neq p_k$ for some $j, k \neq i$), there exists $\hat{p}_i$ with $s(\hat{p}_i, y) > b_{-i}^{\mathrm{I}}$ for all $y$. Let $\delta(y) := s(\hat{p}_i, y) - b_{-i}^{\mathrm{I}} > 0$ and $\delta_{\min} := \min_y \delta(y) > 0$. Then $\pi_i = w_i(\delta(y) - c_3 w_i)$. Choosing any $w_i \in (0,\, \delta_{\min}/c_3)$ ensures $c_3 w_i < \delta(y)$ for all $y$, hence $\pi_i > 0$ for all $y$.
\end{proof}

 \noarb*
\begin{proof}
By Theorem~4.1 of~\citet{Chen2014Removing}, the base net payout $\tilde{\pi}_i := w_i(s(p_i, y) - s(q_{-i}, y))$ has the form of a no-arbitrage wagering mechanism: there exists $y^*$ with $\tilde{\pi}_i(y^*) \leq 0$ for any $\mathbf{p}$ and $\mathbf{w}$. 
Since $\pi_i = \tilde{\pi}_i - c_3 w_i^2 \leq \tilde{\pi}_i$, it follows that $\pi_i(y^*) \leq 0$.
\end{proof}
 
\normtwo*
\begin{proof}
We construct a  counterexample with the Brier score
$s(p,y)=c_1-c_2\|p-\delta_y\|_2^2$ (Def.~\ref{def:brier}).

Recall that normality (Def.~\ref{def:normality}) is asserted for every \emph{fixed}
prediction profile $\mathbf{p}\in\Delta(\mathcal{Y})^M$ and wager profile
$\mathbf{w}\in\mathbb{R}_{\ge0}^M$; these are treated as deterministic
parameters, and the conditional expectation $\mathbb{E}_{\Q_i}[\cdot\mid\mathcal{F}_i]$
is taken solely over the randomness in $Y$.  Consequently, for any fixed
$\mathbf{p}$ and $\mathbf{w}$, the pooled baseline
$q_{-i}=\sum_{j\ne i}w_jp_j/W_{-i}$ is a fixed vector in $\Delta(\mathcal{Y})$ and $p_1=\mathbb{E}_{\Q_i}[\delta_Y\mid\mathcal{F}_1]$
is $\mathcal{F}_1$-measurable,
\[
  \mathbb{E}_{\Q_i}\!\left[\|q_{-1}-\delta_Y\|_2^2\mid\mathcal{F}_1\right]
  = \|q_{-1}-p_1\|_2^2
    + \mathbb{E}_{\Q_i}\!\left[\|p_1-\delta_Y\|_2^2\mid\mathcal{F}_1\right],
\]
where the cross-term $2\langle q_{-1}-p_1,\,\mathbb{E}_{\Q_i}[p_1-\delta_Y\mid\mathcal{F}_1]\rangle
= 2\langle q_{-1}-p_1,\,p_1-p_1\rangle=0$ vanishes.

Hence
\begin{equation}
\label{eq:baseline_score}
  \mathbb{E}_{\Q_i}[s(q_{-1},Y)\mid\mathcal{F}_1]
  = c_1 - c_2\|q_{-1}-p_1\|_2^2 - c_2 \mathbb{E}_{\Q_i}\!\left[\|p_1-\delta_Y\|_2^2\mid\mathcal{F}_1\right].
\end{equation}

Let $\mathcal{Y}=\{1,2,3\}$, $M=3$, and consider normality for player $i=1$.
Fix the following deterministic profiles:
\[
  p_1=(0.5,\,0.5,\,0),\quad
  p_2=(1,\,0,\,0),\quad
  p_3=(0,\,1,\,0),\quad
  w_2=1,\quad w_3=2,
\]
so $W_{-1}=3$.

The \proj~II baseline for player~1 is
\[
  q_{-1}
  = \frac{w_2 p_2+w_3 p_3}{W_{-1}}
  = \tfrac{1}{3}(1,0,0)+\tfrac{2}{3}(0,1,0)
  = \!\left(\tfrac{1}{3},\tfrac{2}{3},0\right).
\]

The squared distance is
\[
  \|q_{-1}-p_1\|_2^2
  = \left\|\!\left(-\tfrac{1}{6},\tfrac{1}{6},0\right)\right\|_2^2
  = \frac{1}{36}+\frac{1}{36}
  = \frac{1}{18}.
\]

Now let $\mathbf{p}'$ be the profile obtained by setting
$p'_2=p_1=(0.5,0.5,0)$ (player~2's prediction improves to match player~1's
true belief) while keeping $p'_3=p_3$ and all wagers unchanged.  

The new pooled baseline is
\[
  q'_{-1}
  = \tfrac{1}{3}(0.5,0.5,0)+\tfrac{2}{3}(0,1,0)
  = \!\left(\tfrac{1}{6},\tfrac{5}{6},0\right),
\]
and the new squared distance to $p_1$ is
\[
  \|q'_{-1}-p_1\|_2^2
  = \left\|\!\left(-\tfrac{1}{3},\tfrac{1}{3},0\right)\right\|_2^2
  = \frac{1}{9}+\frac{1}{9}
  = \frac{2}{9}
  > \frac{1}{18}
  = \|q_{-1}-p_1\|_2^2.
\]


By~\eqref{eq:baseline_score},
\[
  \mathbb{E}_{\Q_i}[s(q'_{-1},Y)\mid\mathcal{F}_1]
  < \mathbb{E}_{\Q_i}[s(q_{-1},Y)\mid\mathcal{F}_1].
\]
Because the net payout subtracts the baseline score,
and $w_1>0$, the expected payoff of player~1 strictly increases:
\[
  \mathbb{E}_{\Q_i}[\pi_1^\mathrm{II}(\mathbf{w},\mathbf{p}',Y)\mid\mathcal{F}_1]
  > \mathbb{E}_{\Q_i}[\pi_1^\mathrm{II}(\mathbf{w},\mathbf{p},Y)\mid\mathcal{F}_1].
\]
This is a contradiction to Def.~\ref{def:normality} and hence \proj~II fails
normality.
\end{proof}

\budget*
\begin{proof}
We first consider \emph{\proj~I}. 
Write $s_i := s(p_i, y)$, $\bar{s}_w := \frac{\sum_j w_j s_j}{W}$ and $W_{-i} = \sum_{j\ne i}w_j$. From $W\bar{s}_w = w_i s_i + W_{-i} b_{-i}^{\mathrm{I}}$, we have:
\begin{equation}\label{eq:loo-global}
s_i - b_{-i}^{\mathrm{I}} = \frac{W(s_i - \bar{s}_w)}{W_{-i}}, \qquad s_i - \bar{s}_w = \frac{W_{-i}}{W}(s_i - b_{-i}^{\mathrm{I}}).
\end{equation}
Writing $\frac{W}{W_{-i}} = 1 + \frac{w_i}{W_{-i}}$ and using $\sum_i w_i(s_i - \bar{s}_w) = 0$, we write our net payout function~\eqref{eq:ourpayout} as the difference of the leave-one-out (LOO) inflation term and the regularization term:

\begin{equation}
\label{eq:exact-budget}
\sum_{i=1}^M \pi_i = \underbrace{\sum_{i=1}^M \frac{w_i^2(s_i - \bar{s}_w)}{W_{-i}}}_{\text{LOO inflation}} \;-\; \underbrace{c_3\sum_{i=1}^M w_i^2}_{\text{regularization}}.
\end{equation}

Using~\eqref{eq:loo-global} and $s_i - b_{-i}^{\mathrm{I}} \leq \overline{s} - \underline{s}$, we have 
$$\frac{w_i^2(s_i - \bar{s}_w)}{W_{-i}} = \frac{w_i^2(s_i - b_{-i}^{\mathrm{I}})}{W} \leq \frac{w_i^2(\overline{s} - \underline{s})}{W}.$$
Therefore, the LOO inflation is at most $\frac{\overline{s} - \underline{s}}{W}\sum_i w_i^2$, giving the first inequality in~\eqref{eq:budget-bound}. 

For the second: since $\sum_i w_i^2 / W \leq w_{\max}$ and $\sum_i w_i^2 \geq w_{\max}^2$,
\[
\frac{(\overline{s}-\underline{s})\sum_i w_i^2}{W} - c_3\sum_i w_i^2 \;\leq\; (\overline{s}-\underline{s})\, w_{\max} - c_3\, w_{\max}^2 \;\leq\; \frac{(\overline{s}-\underline{s})^2}{4c_3},
\]
where the last step maximized at $(\overline{s}-\underline{s})/(2c_3)$ with value $(\overline{s}-\underline{s})^2/(4c_3)$. 
 
\emph{\proj~II.} When $s(\cdot, y)$ is concave in $p$, Jensen's inequality gives $b_{-i}^{\mathrm{II}} \geq b_{-i}^{\mathrm{I}}$ and thus $\pi_{-i}^{\mathrm{II}} \leq \pi_{-i}^{\mathrm{I}}$ for all prediction and wager profiles, so \proj~I's bound applies.
\end{proof}

\section{Deferred Content from Sec.~\ref{sec:experiment}}

\subsection{Datasets}
\label{apx:datasets}

Table~\ref{tab:datasets} shows basic information about each dataset.

\begin{table}[ht]
\centering
\caption{Dataset attributes.}\label{tab:datasets}
 \resizebox{1.0\linewidth}{!}{%
\begin{tabular}{l
                    *{6}{c}}
\toprule
  & PubMedQA~\citep{jin-etal-2019-PubMedQA} & BayesX & MedMCQA~\citep{MedMCQA} & MMLU~\citep{hendrycks2021measuring} & ARC-Challenge~\citep{clark2018think} \\
\midrule
 Size & 37.82K  & 10K & 183K & 116K & 1.17K\\
 Task & Binary Classification & Distribution Matching & Multiple Choice & Multiple Choice & Multiple Choice  \\
 Private Context & \cmark & \cmark & \xmark & \xmark & \xmark \\
\bottomrule

\end{tabular}%
}
\end{table}

PubMedQA~\citep{jin-etal-2019-PubMedQA}\footnote{We use the publicly available PubMedQA dataset on Hugging Face: \url{https://huggingface.co/datasets/qiaojin/PubMedQA}.} is a specialized biomedical research dataset designed for answering questions based on medical abstracts.  The primary task is to provide a ``yes,'' ``no,'' or ``maybe'' response to a clinical question by synthesizing information from an associated PubMed abstract. 
For our experiment, we convert this into a binary classification by focusing on ``yes'' and ``no'' questions. We also balance the dataset such that there is a similar number of  ``yes'' and ``no'' questions. 
We extract the abstract per question and randomly assign relevant or irrelevant abstracts to a subset of participating models.

MedMCQA~\citep{MedMCQA}\footnote{We use the publicly available MedMCQA dataset on Hugging Face: \url{https://huggingface.co/datasets/openlifescienceai/medmcqa}.} is a large-scale, multiple-choice question-answering dataset specifically curated from Indian medical entrance exams (AIIMS and NEET-PG). It covers 21 diverse medical subjects, such as anatomy, pathology, and pharmacology. Because these questions are designed to test human medical students, they require deep domain knowledge, making MedMCQA a standard tool for measuring the clinical competency of general and medical-purpose LLMs.

MMLU~\citep{hendrycks2021measuring}\footnote{We use the publicly available MMLU dataset on Hugging Face: \url{https://huggingface.co/datasets/cais/mmlu}.} is one of the most comprehensive benchmarks for evaluating the general knowledge and problem-solving capabilities of language models. It spans 57 subjects across STEM, the humanities, social sciences, and more, testing both world knowledge and analytical skill. The questions are multiple-choice and range in difficulty from elementary levels to advanced professional standards (such as law or medicine). MMLU has become the industry standard for reporting a model's foundational ``intelligence'' and its ability to handle a vast breadth of human knowledge.

ARC-Challenge~\citep{clark2018think}\footnote{We use the publicly available ARC-Challenge dataset on Hugging Face: \url{https://huggingface.co/datasets/allenai/ai2_arc}.}
 is a high-difficulty subset of the AI2 Reasoning Challenge, consisting of grade-school science questions that are particularly difficult for automated systems to solve. The Challenge set contains questions that cannot be answered by simple keyword matching or statistical co-occurrence alone; they often require a degree of logical reasoning, common sense, or a basic understanding of scientific processes. It remains a key metric for determining how well a model can move beyond rote memorization toward actual conceptual understanding.

BayesX is a new dataset we constructed to capture a setting in which additional information reduces epistemic uncertainty about the world while revealing greater aleatoric uncertainty in the outcome-generating process. 
Each example describes a binary cluster-saturation forecasting problem. The prior probability is either $0.1$ or $0.9$, and the context contains a binary warning signal with sampled true-positive and false-positive rates. The evidence is selected so that receiving additional information can move the posterior closer to a uniform distribution. 
The ground truth is therefore the posterior probability rather than only a hard binary label.\footnote{We include the BayesX dataset in the repo \url{https://github.com/chailab-rutgers/WALLA}.}


\subsection{Formal Definitions of Metrics}\label{apx:metrics}

We formally define Kendall's Tau, MRR, and dynamic regret in terms of the Brier score as follows.

\begin{definition} [Kendall's Tau (K-Tau)] 
We use K-Tau to measure the ordinal association between two orderings: (1) the ordering of models by their per-question Brier-score performance and (2) the ordering of models by their wagers. Formally,
\[
\tau = \frac{\text{(number of concordant pairs) - (number of discordant pairs)}}{\text{(number of pairs)}}.
\]
\end{definition}

\begin{definition} [Mean Reciprocal Rank (MRR) of the Best Experts]
The per-sample reciprocal rank is the inverse rank of the best expert in terms of the Brier score. Formally,
\[
\text{MRR} = \frac{1}{N} \sum_{n=1}^N 
\frac{1}{\text{rank}(\arg\min_j \|\delta^{(n)}_y - p^{(n)}_j\|_2^2)},
\]
where $N$ is the number of samples and the rank is induced by the aggregation weights.
\end{definition}

\begin{definition} [Dynamic Regret (D-Regret) in the Brier Score]
The D-Regret is defined as
\[
\frac{1}{N} \sum_{n=1}^N 
\|\delta_y^{(n)} - p_{\text{agg}}^{(n)}\|_2^2 
- 
\min_j \|\delta_y^{(n)} - p_j^{(n)}\|_2^2.
\]
It measures the Brier-score gap between the aggregated prediction and the best individual model on each question.
\end{definition}

\subsection{Baselines}\label{apx:baselines}

Uniform averaging (UniformAvg) is a simple method that averages the predicted probability distributions with equal weights (i.e., $p_{\text{agg}} = \frac{\sum_{i=1}^Mp_i}{M}$).



SelfCertainty~\citep{kang2025scalable} is a learning-free ensembling that uses model confidence as weight for averaging. We use the variant that quantifies confidence through the Kullback-Leibler (KL) divergence between the model's predicted token probability distribution and a uniform distribution, as it performs best.\footnote{We use the public implementation of SelfCertainty from \url{https://github.com/backprop07/Self-Certainty}.}

PackLLM~\citep{mavromatis2024pack} is a learning-free ensembling method that aims to minimize the perplexity of the input prompt for the ensemble of all models.
As PackLLM relies on the perplexity of the input tokens for each model, it requires open-weight models.  In a decentralized setting where different providers control different models, providers can misreport perplexities to gain disproportional weights.\footnote{We use the public implementation of PackLLM from \url{https://github.com/cmavro/PackLLM}.}

RouterDC~\citep{chen2024routerdc} is a pre-inference router that finetunes a DeBERTaV3-base~\citep{he2023debertav} language model via dual-contrastive learning. It aligns a query's embedding with the top-K most capable LLMs for that specific task while distancing it from the last-K models.\footnote{We use the public implementation of RouterDC from \url{https://github.com/shuhao02/RouterDC}. We use the pretrained DeBERTaV3 from \url{https://huggingface.co/microsoft/deberta-v3-base}.}

\citet{song2025irt} proposes a pre-inference router inspired by Item Response Theory (IRT) called IRTRouter. It treats user queries as ``test items'' and LLMs as ``respondents,'' and explicitly models the relationship between a specific query's difficulty and an individual model's latent capabilities. It uses a frozen BERT model to retrieve the query embedding. We adopt NIRTRouter, a variant that trains neural networks for the IRT implementation.\footnote{We use the public implementation of NIRTRouter from \url{https://github.com/Mercidaiha/IRT-Router}. We adopt the pretrained BERT from \url{https://huggingface.co/google-bert/bert-base-uncased}.}

RouteLLM~\citep{ong2025routellm} was designed to choose between a more capable but expensive model and a weaker and cheaper model using human pairwise preference as feedback. Its base router is a BERT model that is finetuned using preference data.  We extend this framework for routing on more LLMs using cross-entropy loss and pairwise ranking loss.\footnote{We adapt the public implementation of RouteLLM from \url{https://github.com/lm-sys/RouteLLM.git}. We adopt the pretrained BERT from \url{https://huggingface.co/google-bert/bert-base-uncased}.}


Stacked generalization (StackedGen)~\citep{WOLPERT1992241} is a classical ensembling technique that learns a neural network to combine diverse predictive models to achieve greater overall accuracy than any single model could produce alone. We train a two-layer perceptron with cross-entropy loss. We use the last-layer hidden states of LLMs as input to the perceptron as \proj to be consistent.




\begin{table}[ht]
\centering
\caption{The mechanism parameters and learning parameters applied to all learning-based methods.}\label{tab:params}
 \resizebox{1.0\linewidth}{!}{%
\begin{tabular}{l
                    *{10}{c}}
\toprule
\multicolumn{6}{c}{Mechanism Parameters}  &  \multicolumn{4}{c}{Learning Parameters} \\
\cmidrule(lr){1-6} \cmidrule(lr){7-10}
$c_1$  & $c_2$  & $c_3$ & Aggregation & Scoring rule & Variant & Batch size & Dataset split & MLP hidden sizes & Optimizer\\
1 & 1/2 & 1/2 & Linear pooling & Brier score & \proj~I & 100 & 8:1:1 & 512, 256 & Adam\\ 
\bottomrule

\end{tabular}%
}
\end{table}

\subsection{Additional Experiment Setup Details}\label{apx:exp_setup}
The learning parameters provided in Table~\ref{tab:params} are common to all methods.
For every dataset, we use an $8{:}1{:}1$ train-validation-test split and fix the batch size to 100.  For learning-based methods, we use the Adam optimizer, tune the learning rate with grid search and adopt early stopping to prevent overfitting.
For baselines, we stop at the epoch that achieves the best validation performance. For \proj, because the training procedure is decentralized, we use the Brier score on the most recent training batches as the early stopping criterion, ensuring that the number of training-batch samples used for early stopping matches the validation-set size. Unless specified otherwise, all methods use linear pooling for aggregation.
 In all experiments, the aggregated predictions are evaluated on a held-out test set, where no further updates to the wagering strategy are allowed.
 
Experiments are run on 4 NVIDIA RTX PRO 6000 Blackwell Server Edition GPUs (96 GiB VRAM each) and a dual-socket AMD EPYC 9354 setup (64 CPU cores total). It provides about 1.1 TiB system RAM (plus 33GiB swap).




    

\begin{table*}[ht]
\centering
\begin{minipage}{\textwidth}
\caption{Homogeneous models (Aloe) on PubMedQA with private contexts. Questions that all models can answer correctly are filtered out. For each question, one of $M \in \{4, 8, 12\}$ models is randomly selected to receive a relevant abstract; the rest receive none.}\label{tab:homo_subset}
 \resizebox{1.0\linewidth}{!}{%
\begin{tabular}{l
                    *{5}{c} 
                   *{5}{c} *{5}{c} }
\toprule
& \multicolumn{5}{c}{\textbf{4} } & \multicolumn{5}{c}{\textbf{8} } & \multicolumn{5}{c}{\textbf{12}} \\
\cmidrule(lr){2-6} \cmidrule(lr){7-11} \cmidrule(lr){12-16}
\textbf{Method} &AUC $\uparrow$ & ACC $\uparrow$ & ECE $\downarrow$  & MRR $\uparrow$ & DRegret $\downarrow$ &AUC $\uparrow$ & ACC $\uparrow$ & ECE $\downarrow$  & MRR $\uparrow$ & DRegret $\downarrow$ &AUC $\uparrow$ & ACC $\uparrow$ & ECE $\downarrow$  & MRR $\uparrow$ & DRegret $\downarrow$\\
\midrule
 UniformAvg &  26.05 &  30.44 &  30.89 &  54.57 &  39.27 &  19.29 &  20.36 & 43.32 &  39.10&    48.38 &  19.92 &  17.03  & 47.68 &  31.92 &    51.66 \\
SelfCertainty &  54.42  & 48.94 &  23.12 &  67.94  &    31.05 &  51.38 & 43.15 &   27.94 &  62.45 &   36.41&  49.71 & 39.82 &  30.81 &  60.65&    39.30 \\
PackLLM & 78.23 &  63.72 &  4.72 &  89.58 &    15.18 &  48.75 & 55.15  &  9.35 &  88.83  &  22.87 &  41.43 &  48.42 &   41.43  & 88.45 &   28.00\\
\midrule
RouterDC w/ context & 24.58\tiny{$\pm$27.74}&  31.57\tiny{$\pm$61.88}  &30.69\tiny{$\pm$59.00}&  61.57\tiny{$\pm$366.04}  & 39.32\tiny{$\pm$39.64} & 19.86\tiny{$\pm$2.61}  & 20.32\tiny{$\pm$3.61} & 43.30\tiny{$\pm$4.81} &48.12\tiny{$\pm$146.06}& 48.17\tiny{$\pm$4.22} &20.22\tiny{$\pm$6.03} & 17.03\tiny{$\pm$1.03}&  47.61\tiny{$\pm$0.21} & 41.98\tiny{$\pm$186.21}   & 51.47\tiny{$\pm$2.50}  \\
NIRTRouter w/ context &  25.81\tiny{$\pm$3.03} & 30.56\tiny{$\pm$1.55} &30.77\tiny{$\pm$1.54} &52.02\tiny{$\pm$8.94}& 39.27\tiny{$\pm$0.04} &19.29\tiny{$\pm$0.00}&  20.36\tiny{$\pm$0.00} & 43.32\tiny{$\pm$0.00}&  33.89\tiny{$\pm$17.09} &  48.38\tiny{$\pm$0.00} & 19.92\tiny{$\pm$0.03}& 17.03\tiny{$\pm$0.00}&  47.67\tiny{$\pm$0.00} & 24.50\tiny{$\pm$1.64} &  51.66\tiny{$\pm$0.01}  \\
RouteLLM w/ context &  72.54\tiny{$\pm$0.15} & 69.27\tiny{$\pm$0.15}  & 4.89\tiny{$\pm$0.09}  & 89.30\tiny{$\pm$0.76}  & 6.47\tiny{$\pm$0.14}  & 72.55\tiny{$\pm$0.07} &  72.55\tiny{$\pm$0.07} &  4.99\tiny{$\pm$0.26} & 87.92\tiny{$\pm$0.79}  &  6.53\tiny{$\pm$0.05} & 73.11\tiny{$\pm$0.11} & 69.08\tiny{$\pm$0.08} & 5.15\tiny{$\pm$0.08}& 
86.59\tiny{$\pm$0.89} &  6.55\tiny{$\pm$0.04} \\
RouteLLM w/o context & 23.03\tiny{$\pm$0.59} & 29.69\tiny{$\pm$0.49}  & 31.82\tiny{$\pm$0.50}  & 52.08\tiny{$\pm$0.42}  & 39.24\tiny{$\pm$0.01} & 19.52\tiny{$\pm$0.83} &  20.00\tiny{$\pm$0.46} &  43.68\tiny{$\pm$0.46} & 33.43\tiny{$\pm$0.90}  & 48.52\tiny{$\pm$0.02} & 22.00\tiny{$\pm$0.81}  & 16.26\tiny{$\pm$0.41}  & 48.47\tiny{$\pm$0.40}  & 
27.78\tiny{$\pm$2.65} & 51.64\tiny{$\pm$0.06}  \\
StackedGen & 71.57\tiny{$\pm$0.36} & 69.25\tiny{$\pm$0.27} & 4.27\tiny{$\pm$0.20} & 89.18\tiny{$\pm$0.44}  & 6.26\tiny{$\pm$0.34} &  73.38\tiny{$\pm$1.09} &  69.15\tiny{$\pm$0.45} &  5.02\tiny{$\pm$0.58} &  87.10\tiny{$\pm$1.11}&    6.48\tiny{$\pm$0.61} &  72.68\tiny{$\pm$0.42} & 69.05\tiny{$\pm$0.70}  &  4.70\tiny{$\pm$0.50} &  85.99\tiny{$\pm$1.10} &   6.90\tiny{$\pm$0.65} \\
\midrule
 \proj~I &  67.60\tiny{$\pm$1.05} & 67.34\tiny{$\pm$1.19}  & 3.08\tiny{$\pm$0.37}  &  89.14\tiny{$\pm$0.55} &   8.81\tiny{$\pm$0.89}  &  70.91\tiny{$\pm$0.82} &  68.28\tiny{$\pm$0.34} &   7.55\tiny{$\pm$0.65} &  79.48\tiny{$\pm$1.08}  &   16.81\tiny{$\pm$0.36} & 64.71\tiny{$\pm$0.82} & 67.33\tiny{$\pm$0.78} &   7.22\tiny{$\pm$1.77}  & 77.34\tiny{$\pm$1.87}  & 20.45\tiny{$\pm$0.79}  \\
\proj~II & 66.07\tiny{$\pm$0.41}  & 66.55\tiny{$\pm$0.27}  & 2.45\tiny{$\pm$0.39}  &  89.62\tiny{$\pm$0.21}  &  9.95\tiny{$\pm$0.10}  & 63.16\tiny{$\pm$0.62}  &65.50\tiny{$\pm$0.96} & 2.88\tiny{$\pm$0.48}  &87.30\tiny{$\pm$1.40}   & 12.60\tiny{$\pm$0.62} &  59.42\tiny{$\pm$1.30}  & 63.68\tiny{$\pm$1.10}  & 2.84\tiny{$\pm$0.91} &   86.18\tiny{$\pm$0.46}  & 14.96\tiny{$\pm$0.40}
\\
\bottomrule

\end{tabular}%
}
\end{minipage}
\end{table*}

\begin{table*}[ht]
\centering
\begin{minipage}{\textwidth}
\caption{2 Heterogeneous models (Llama3.1 and Aloe) without context on three datasets.
}\label{tab:2models}
 \resizebox{1.0\linewidth}{!}{%
\begin{tabular}{l
                    *{5}{c} 
                   *{5}{c} *{5}{c}}
\toprule
& \multicolumn{5}{c}{\textbf{MedMCQA (i.d.)}} & \multicolumn{5}{c}{\textbf{MMLU (i.d.)}} & \multicolumn{5}{c}{\textbf{ARC-Challenge (o.o.d.)}} \\
\cmidrule(lr){2-6} \cmidrule(lr){7-11} \cmidrule(lr){12-16}
\textbf{Method} & ACC $\uparrow$ & ECE $\downarrow$  & MRR $\uparrow$ & K-Tau $\uparrow$ & DRegret $\downarrow$ & ACC $\uparrow$ & ECE $\downarrow$  & MRR $\uparrow$ & K-Tau $\uparrow$ & DRegret $\downarrow$ & ACC $\uparrow$ & ECE $\downarrow$  & MRR $\uparrow$ & K-Tau $\uparrow$ & DRegret $\downarrow$ \\
\midrule
Llama3.1 & 78.97 & 2.64& 69.51 & -21.95 & 10.62 & 81.99 & 2.52 & 68.51 & -25.96 & 7.37 &  81.30 & 3.06 & 68.39& -26.43 & 8.89\\
Aloe & 79.77 & 3.18 & 80.49 & 21.95 & 10.42 &  80.23 & 8.77 & 81.49  & 25.96 & 11.10 & 80.35 &8.58 & 81.61 & 26.43 & 11.31  \\
\midrule
UniformAvg & 81.76 & 4.60 & 69.51 & 0.00 & 7.58 & 82.22 & 2.47 & 68.51 & 0.00 & 6.66 & 82.87 & 3.05 & 68.39 & 0.00 &  7.05  \\
SelfCertainty & 81.73 & 2.24 & 88.64 & 54.56 & 7.32 & 82.15 & 5.07 & 88.57 & 54.29 & 7.56 & 82.87 & 3.94 & 87.74 & 50.96 & 7.34 \\
PackLLM & 81.54 & 4.18 & 79.04 & 16.17  & 7.62 & 82.27& 2.49 & 72.53&  -9.90  & 6.65 & 82.70 & 2.80  & 75.70 &  2.78 & 7.15 \\
\midrule
RouterDC & 82.01\tiny{$\pm$3.58} & 4.36\tiny{$\pm$0.62} & 80.18\tiny{$\pm$3.91} & 20.71\tiny{$\pm$15.64}  & 7.49\tiny{$\pm$1.11} &  82.21\tiny{$\pm$0.04} &  2.54\tiny{$\pm$0.13}  & 76.30\tiny{$\pm$8.83} & 5.19\tiny{$\pm$35.31}  &  6.68\tiny{$\pm$0.03}  &  82.89\tiny{$\pm$0.05} &  2.86\tiny{$\pm$0.45}  & 76.32\tiny{$\pm$8.99} &  5.29\tiny{$\pm$35.96} & 7.07\tiny{$\pm$0.03}      \\
NIRTRouter & 81.98\tiny{$\pm$3.61} & 4.27\tiny{$\pm$0.98} & 80.56\tiny{$\pm$0.94} & 22.24\tiny{$\pm$3.75} & 7.49\tiny{$\pm$1.13} & 82.36\tiny{$\pm$0.57}  & 2.23\tiny{$\pm$0.23}  & 68.51\tiny{$\pm$0.03}  & -25.97\tiny{$\pm$0.13}  &  6.54\tiny{$\pm$0.55}  & 82.78\tiny{$\pm$0.00} & 3.42\tiny{$\pm$0.00} & 68.39\tiny{$\pm$0.00} & -26.43\tiny{$\pm$0.00}  & 7.00\tiny{$\pm$0.00}   \\
RouteLLM & 81.82\tiny{$\pm$2.71} & 3.89\tiny{$\pm$1.74}  & 80.56\tiny{$\pm$0.94} & 22.24\tiny{$\pm$3.75} & 7.51\tiny{$\pm$1.14} & 82.56\tiny{$\pm$0.19} & 1.86\tiny{$\pm$1.28} & 68.51\tiny{$\pm$0.03} & -25.97\tiny{$\pm$0.13} & 6.37\tiny{$\pm$0.07} & 83.13\tiny{$\pm$0.00}  & 2.13\tiny{$\pm$0.00} & 68.39\tiny{$\pm$0.00} & -26.43\tiny{$\pm$0.00} & 6.94\tiny{$\pm$0.00} \\
StackedGen & 82.61\tiny{$\pm$0.28} & 2.16\tiny{$\pm$0.32} & 85.51\tiny{$\pm$0.39} & 42.04\tiny{$\pm$1.57} & 6.35\tiny{$\pm$0.11}  & 82.52\tiny{$\pm$0.21} & 2.29\tiny{$\pm$0.21} & 77.47\tiny{$\pm$1.03} & 9.88\tiny{$\pm$4.12} & 6.21\tiny{$\pm$0.19} & 83.33\tiny{$\pm$0.21}  & 3.07\tiny{$\pm$0.35} & 75.22\tiny{$\pm$0.37} & 0.87\tiny{$\pm$1.46}  & 6.61\tiny{$\pm$0.03}   \\
\midrule
 \proj~I & 82.02\tiny{$\pm$0.32} & 4.53\tiny{$\pm$0.28} & 73.69\tiny{$\pm$3.26} & -5.25\tiny{$\pm$13.06} & 7.57\tiny{$\pm$0.12} & 82.41\tiny{$\pm$0.17} & 1.84\tiny{$\pm$0.09} & 69.07\tiny{$\pm$1.01} & -23.72\tiny{$\pm$4.03} & 6.49\tiny{$\pm$0.11} &  83.00\tiny{$\pm$0.14} & 2.44\tiny{$\pm$0.17} & 68.30\tiny{$\pm$0.20} & -26.78\tiny{$\pm$0.81} & 7.01\tiny{$\pm$0.04}  \\
\proj~II & 82.02\tiny{$\pm$0.23} & 4.48\tiny{$\pm$0.32} & 74.07\tiny{$\pm$6.44} & -3.73\tiny{$\pm$25.78}  & 7.56\tiny{$\pm$0.10}  & 82.35\tiny{$\pm$0.33}  & 1.90\tiny{$\pm$0.23}  & 68.69\tiny{$\pm$0.67} &  -25.23\tiny{$\pm$2.67}  & 6.51\tiny{$\pm$0.12} &  82.91\tiny{$\pm$0.18} &  2.34\tiny{$\pm$0.37}  & 68.30\tiny{$\pm$0.28}  & -26.78\tiny{$\pm$1.13}  & 7.01\tiny{$\pm$0.05}  \\
\proj~I (log-pool) &  81.73\tiny{$\pm$0.17} &  2.70\tiny{$\pm$0.66}  & 73.96\tiny{$\pm$5.91} &  -4.17\tiny{$\pm$23.66}&   7.19\tiny{$\pm$0.08}  & 82.33\tiny{$\pm$0.08} &  3.09\tiny{$\pm$0.43} &  70.08\tiny{$\pm$0.37}&   -19.67\tiny{$\pm$1.50}&   6.55\tiny{$\pm$0.09} &   82.65\tiny{$\pm$0.29}  & 3.02\tiny{$\pm$0.80} &  68.02\tiny{$\pm$0.49}&  -27.91\tiny{$\pm$1.95}&  6.98\tiny{$\pm$0.37}  \\
\proj~I (calibrated) &  82.55\tiny{$\pm$0.16}  & 5.20\tiny{$\pm$0.17} &  75.03\tiny{$\pm$4.71} &  0.10\tiny{$\pm$18.83}&   6.78\tiny{$\pm$0.02}  &  82.76\tiny{$\pm$0.05}  & 2.45\tiny{$\pm$0.08} &   79.11\tiny{$\pm$1.47}  & 16.44\tiny{$\pm$5.88}&   5.79\tiny{$\pm$0.03} &  82.87\tiny{$\pm$0.11} &  2.88\tiny{$\pm$0.44} &  81.67\tiny{$\pm$0.17}&  26.70\tiny{$\pm$0.70} & 6.64\tiny{$\pm$0.05}  \\
\bottomrule

\end{tabular}%
}
\end{minipage}
\end{table*}

\begin{table*}[ht]
\centering
\begin{minipage}{\textwidth}
\caption{2 Heterogeneous models (Llama3.1 and Aloe) without context on three datasets, with questions that both models can answer correctly filtered out.
}\label{tab:2models_subset}
 \resizebox{1.0\linewidth}{!}{%
\begin{tabular}{l
                    *{5}{c} 
                   *{5}{c} *{5}{c} }
\toprule
& \multicolumn{5}{c}{\textbf{MedMCQA (i.d.)}}  & \multicolumn{5}{c}{\textbf{MMLU (i.d.)}} & \multicolumn{5}{c}{\textbf{ARC-Challenge (o.o.d.)}} \\
\cmidrule(lr){2-6} \cmidrule(lr){7-11} \cmidrule(lr){12-16}
\textbf{Method} & ACC $\uparrow$ & ECE $\downarrow$  & MRR $\uparrow$ & K-Tau $\uparrow$ & DRegret $\downarrow$ & ACC $\uparrow$ & ECE $\downarrow$  & MRR $\uparrow$ & K-Tau $\uparrow$ & DRegret $\downarrow$ & ACC $\uparrow$ & ECE $\downarrow$  & MRR $\uparrow$ & K-Tau $\uparrow$ & DRegret $\downarrow$ \\
\midrule
Llama3.1 & 26.38 & 30.11 & 77.69 & 10.78 & 24.86 & 26.94 & 38.89 & 81.74 & 26.94 & 21.41 & 27.85 & 38.59 & 80.37 & 21.48 & 24.50 \\
Aloe & 29.20 & 36.31 & 72.31 & -10.78 & 32.14 & 19.79 & 54.53 & 68.26 & -26.94 & 41.29 & 24.16 & 50.18 & 69.63 & -21.48 & 39.75 \\
\midrule
UniformAvg & 36.17 & 17.01 & 77.69 & 0.00 & 20.68 & 27.87 & 34.11 & 81.74 & 0.00 & 22.88 & 33.89  &31.10  & 80.37 & 0.00 & 22.62 \\
SelfCertainty & 36.03  & 23.67 & 72.61 & -9.55 & 22.65 & 27.59 & 40.88 & 67.33 & -30.68 & 28.38 & 33.89 &  34.01 & 69.46 & -22.15 & 25.80 \\
PackLLM & 35.40 & 18.32 & 73.39 & -6.45 & 21.08 & 28.08 & 33.83  & 78.34 & 13.37 & 22.79  & 33.22 & 31.28 & 73.32 & -6.71  & 23.04  \\
\midrule
RouterDC & 36.44\tiny{$\pm$5.23} & 16.94\tiny{$\pm$6.40} & 72.11\tiny{$\pm$2.45} & -11.55\tiny{$\pm$9.82} & 20.66\tiny{$\pm$1.91} & 27.81\tiny{$\pm$0.17}  &  34.39\tiny{$\pm$0.46}   &  73.65\tiny{$\pm$9.16}  & -5.39\tiny{$\pm$36.64}  &  22.97\tiny{$\pm$0.17} & 33.96\tiny{$\pm$0.19} & 29.69\tiny{$\pm$0.15}   &  73.93\tiny{$\pm$7.30} & -4.30\tiny{$\pm$29.21} & 22.71\tiny{$\pm$0.17}  \\
NIRTRouter & 36.34\tiny{$\pm$5.34} & 17.10\tiny{$\pm$6.61}  & 72.04\tiny{$\pm$3.38} & -11.84\tiny{$\pm$13.53} & 20.73\tiny{$\pm$1.92} & 28.68\tiny{$\pm$5.62} & 33.36\tiny{$\pm$1.67}  & 81.28\tiny{$\pm$5.82}  & 25.11\tiny{$\pm$23.28} & 22.06\tiny{$\pm$4.98}  &33.56\tiny{$\pm$0.00} & 28.53\tiny{$\pm$0.01}  & 80.37\tiny{$\pm$0.00}  & 21.48\tiny{$\pm$0.00}  &22.30\tiny{$\pm$0.02} \\
RouteLLM & 35.78\tiny{$\pm$2.09} & 18.07\tiny{$\pm$2.44} & 72.04\tiny{$\pm$3.38} &  -11.84\tiny{$\pm$13.53} & 21.03\tiny{$\pm$1.65} & 29.49\tiny{$\pm$2.49} & 32.43\tiny{$\pm$1.57} & 81.28\tiny{$\pm$5.82}& 25.11\tiny{$\pm$23.28} & 20.73\tiny{$\pm$2.98} & 34.90\tiny{$\pm$0.00} & 26.61\tiny{$\pm$0.02} & 80.37\tiny{$\pm$0.00}& 21.48\tiny{$\pm$0.00} & 21.33\tiny{$\pm$0.04}\\
StackedGen & 38.65\tiny{$\pm$1.16} & 20.23\tiny{$\pm$1.00} & 78.26\tiny{$\pm$0.29}  & 13.02\tiny{$\pm$1.15} & 13.02\tiny{$\pm$1.15} & 29.80\tiny{$\pm$0.60}  & 32.88\tiny{$\pm$0.88}  & 79.54\tiny{$\pm$0.40}  & 18.17\tiny{$\pm$1.61} & 20.44\tiny{$\pm$0.67} & 35.65\tiny{$\pm$0.80} & 29.35\tiny{$\pm$1.20} & 80.29\tiny{$\pm$0.15} & 21.14\tiny{$\pm$0.62} & 20.42\tiny{$\pm$0.10}  \\
\midrule
 \proj~I & 36.56\tiny{$\pm$1.20} & 16.28\tiny{$\pm$1.50} & 75.88\tiny{$\pm$0.87} & 3.50\tiny{$\pm$3.47} & 20.64\tiny{$\pm$0.19} & 29.37\tiny{$\pm$0.81} & 31.47\tiny{$\pm$1.36} & 81.04\tiny{$\pm$0.99} & 24.17\tiny{$\pm$3.96} & 20.85\tiny{$\pm$0.66} & 34.40\tiny{$\pm$0.53} & 26.95\tiny{$\pm$0.46} & 80.16\tiny{$\pm$0.26} & 20.64\tiny{$\pm$1.02} & 21.00\tiny{$\pm$0.03}   \\
\proj~II & 36.55\tiny{$\pm$0.88} & 16.34\tiny{$\pm$0.47} & 75.67\tiny{$\pm$2.61} & 2.68\tiny{$\pm$10.45} & 20.64\tiny{$\pm$0.28}  & 29.12\tiny{$\pm$0.75}   & 31.71\tiny{$\pm$0.88} &  81.01\tiny{$\pm$0.84} &  24.03\tiny{$\pm$3.37}  &  21.00\tiny{$\pm$0.85}   &  34.06\tiny{$\pm$0.69}  &  26.87\tiny{$\pm$0.74}  &  80.20\tiny{$\pm$0.38}  &  20.81\tiny{$\pm$1.51}  & 21.03\tiny{$\pm$0.09}  \\
\proj~I (log-pooling) &  36.06\tiny{$\pm$0.58}  & 19.84\tiny{$\pm$1.13}  & 75.57\tiny{$\pm$3.00}  & 2.27\tiny{$\pm$11.99}  & 21.32\tiny{$\pm$0.62} & 28.31\tiny{$\pm$0.31}  & 34.79\tiny{$\pm$0.69}  & 80.67\tiny{$\pm$0.64}& 
22.67\tiny{$\pm$2.56}  & 22.92\tiny{$\pm$0.97}  & 33.05\tiny{$\pm$1.11}  & 31.67\tiny{$\pm$2.59} &  79.70\tiny{$\pm$0.58}& 
18.79\tiny{$\pm$2.31} &  22.31\tiny{$\pm$0.29}  \\
\proj~I (calibrated) & 38.92\tiny{$\pm$0.55} &  11.68\tiny{$\pm$1.05}&   74.77\tiny{$\pm$1.09}&  -0.91\tiny{$\pm$4.38} &  18.42\tiny{$\pm$0.07}  &  30.05\tiny{$\pm$0.20} &  24.90\tiny{$\pm$0.29} &  76.88\tiny{$\pm$0.21}  & 7.54\tiny{$\pm$0.85} & 18.15\tiny{$\pm$0.09}  & 33.89\tiny{$\pm$0.44}&  26.35\tiny{$\pm$0.62} &  76.26\tiny{$\pm$0.27}&   5.03\tiny{$\pm$1.07} &  19.98\tiny{$\pm$0.08}\\
\bottomrule

\end{tabular}%
}
\end{minipage}
\end{table*}

\begin{table*}[ht]
\centering
\begin{minipage}{\textwidth}
\caption{Heterogeneous models (Gemma2, Aloe, Llama3.1, and BioMistral) without context on three datasets, with questions that all models can answer correctly filtered out. 
}\label{tab:4models_subset}
 \resizebox{1.0\linewidth}{!}{%
\begin{tabular}{l
                    *{5}{c} 
                   *{5}{c} *{5}{c} }
\toprule
& \multicolumn{5}{c}{\textbf{MedMCQA (i.d.)}} & \multicolumn{5}{c}{\textbf{MMLU (i.d.)}} & \multicolumn{5}{c}{\textbf{ARC-Challenge (o.o.d.)}}  \\
\cmidrule(lr){2-6} \cmidrule(lr){7-11} \cmidrule(lr){12-16} 
\textbf{Method} & ACC $\uparrow$ & ECE $\downarrow$  & MRR $\uparrow$ & K-Tau $\uparrow$ & DRegret $\downarrow$ & ACC $\uparrow$ & ECE $\downarrow$  & MRR $\uparrow$ & K-Tau $\uparrow$ & DRegret $\downarrow$ & ACC $\uparrow$ & ECE $\downarrow$  & MRR $\uparrow$ & K-Tau $\uparrow$ & DRegret $\downarrow$ \\
\midrule
Gemma2 &  43.25  & 40.85  & 56.83  & -8.50  & 71.61  & 70.20  & 21.74  & 76.62  & 21.61 &  29.71  & 78.04  & 17.82  & 83.15  & 28.84  & 24.36   \\
Aloe &  69.05   & 7.86  &  60.32  & 16.12 &   23.26  &  54.05  &  26.10  &   50.61  &  -3.01 &   47.75 &   54.89  &  24.89  &  48.85  &  -5.09  &  52.09 \\
Llama3.1 &   67.82  &  2.57  &  56.19  &  11.84  &  22.25   & 58.15  & 15.59   & 43.17  &  0.69  &  37.34 &   57.09   & 15.80  & 39.87 &   -2.76  & 43.73  \\
BioMistral &  19.64   &  38.15   &  41.48 &   -19.47  &  78.88  &  21.09  &   45.46  &   38.26  &   -19.29  &   82.59  &   23.35  &   42.36  &   35.84   &  -20.99   &  84.41   \\ 
\midrule
UniformAvg & 63.97 & 8.20  & 58.06 & 0.00 & 30.79 & 65.48 &  9.12 & 76.33 & 0.00 &  30.32 &  71.06 & 12.69 &  82.58 &  0.00 &  30.86    \\
SelfCertainty & 60.87   & 9.05 &   71.25  & 19.89  &   33.34  &  68.81  &  7.37  & 78.04  &  27.88  & 27.03   & 75.25  & 7.46  &  83.53  & 36.66 &  24.16  \\
PackLLM &  67.87  &  10.75  &  55.85  &  21.00  &  27.51  &  64.99  &  8.92  &  41.64  &  -15.83  &  31.16  &  68.46  &  11.31  &  38.52  &  -16.10  &  32.40 \\
\midrule
RouterDC    &  67.42\tiny{$\pm$3.61}  &  9.29\tiny{$\pm$0.91}  & 59.79\tiny{$\pm$0.83}   & 30.02\tiny{$\pm$1.96}   & 27.16\tiny{$\pm$8.06} &  65.78\tiny{$\pm$0.51}   & 8.28\tiny{$\pm$0.51}  &   59.41\tiny{$\pm$18.55}  & 21.66\tiny{$\pm$13.68}  &  29.83\tiny{$\pm$0.32}  & 71.26\tiny{$\pm$0.43}  & 12.43\tiny{$\pm$0.47}  &  60.92\tiny{$\pm$25.76}  &  22.91\tiny{$\pm$19.13} &  30.36\tiny{$\pm$0.38}  \\
NIRTRouter &  70.82\tiny{$\pm$3.04} &  6.77\tiny{$\pm$0.10}  &  59.79\tiny{$\pm$0.81}  &  30.03\tiny{$\pm$1.91}  &  21.47\tiny{$\pm$2.98} &   70.08\tiny{$\pm$7.19}  &  4.81\tiny{$\pm$0.11}  & 76.51\tiny{$\pm$1.37} &   32.92\tiny{$\pm$8.11}  &  25.38\tiny{$\pm$7.51}  &  77.45\tiny{$\pm$5.07}  &  8.87\tiny{$\pm$0.99}  &  83.15\tiny{$\pm$0.00}  &  39.79\tiny{$\pm$0.00}  &  23.58\tiny{$\pm$6.45}  \\
RouteLLM &  71.67\tiny{$\pm$1.79}  &  2.09\tiny{$\pm$4.02}  &  59.79\tiny{$\pm$0.81}   & 30.03\tiny{$\pm$1.91}  &  18.77\tiny{$\pm$0.39}  &  69.95\tiny{$\pm$8.91}  &  8.92\tiny{$\pm$4.87}  &  76.51\tiny{$\pm$1.37} &   32.92\tiny{$\pm$8.11}   & 24.56\tiny{$\pm$7.84}  &  78.44\tiny{$\pm$0.00}  &  6.28\tiny{$\pm$3.05}  &  83.15\tiny{$\pm$0.00}  &  39.79\tiny{$\pm$0.00}  & 20.90\tiny{$\pm$3.11}  \\
StackedGen &  73.13\tiny{$\pm$0.30}  &  1.47\tiny{$\pm$0.59}  &  64.16\tiny{$\pm$0.20}  &  33.91\tiny{$\pm$0.07} & 17.09\tiny{$\pm$0.16}   & 70.59\tiny{$\pm$1.23}   & 8.42\tiny{$\pm$1.15} &  76.20\tiny{$\pm$0.58}  & 32.87\tiny{$\pm$1.56}  & 23.62\tiny{$\pm$1.43}   & 77.84\tiny{$\pm$0.26}  &  7.41\tiny{$\pm$1.70}  &  82.96\tiny{$\pm$0.31}   & 39.40\tiny{$\pm$0.18}  & 19.97\tiny{$\pm$0.10}   \\
\midrule
 \proj~I &  70.82\tiny{$\pm$0.24}  &  7.38\tiny{$\pm$0.43}  &  59.90\tiny{$\pm$0.60}  &  28.92\tiny{$\pm$0.64}   & 21.93\tiny{$\pm$0.41} &  69.88\tiny{$\pm$1.10}   & 6.51\tiny{$\pm$1.01}  &  76.32\tiny{$\pm$0.35}  &  32.16\tiny{$\pm$0.77}   & 26.00\tiny{$\pm$0.55}  &  76.80\tiny{$\pm$0.70}  &  10.87\tiny{$\pm$1.65}  &  80.47\tiny{$\pm$0.97}   & 38.59\tiny{$\pm$0.85}   & 24.24\tiny{$\pm$0.28}  \\
\proj~II &  70.56\tiny{$\pm$0.35}  &  7.62\tiny{$\pm$0.78}  &  59.10\tiny{$\pm$0.61}   & 26.11\tiny{$\pm$0.83}   & 22.65\tiny{$\pm$0.42}  &  70.32\tiny{$\pm$1.11}  &  6.20\tiny{$\pm$0.94}  &  76.41\tiny{$\pm$0.56}  &  29.48\tiny{$\pm$4.68}  &  25.63\tiny{$\pm$0.70}   & 77.69\tiny{$\pm$0.54}  &  11.16\tiny{$\pm$2.10} &   81.94\tiny{$\pm$0.85}  &  36.48\tiny{$\pm$5.60}  &  23.26\tiny{$\pm$0.32}  \\
\proj~I (log-pooling) &  69.89\tiny{$\pm$0.16}  & 1.26\tiny{$\pm$0.14}  & 59.32\tiny{$\pm$0.08}  & 28.42\tiny{$\pm$0.36}  & 21.81\tiny{$\pm$0.17} & 70.63\tiny{$\pm$0.11}  & 10.02\tiny{$\pm$0.62}  & 70.92\tiny{$\pm$1.48}  & 23.15\tiny{$\pm$2.32}  & 24.03\tiny{$\pm$0.02}  & 78.54\tiny{$\pm$0.18} &  8.91\tiny{$\pm$0.49}  & 75.94\tiny{$\pm$1.27} &  28.08\tiny{$\pm$2.08} &  20.90\tiny{$\pm$0.07} \\
\proj~I (calibration) &  73.27\tiny{$\pm$0.07} & 10.17\tiny{$\pm$0.14} &  67.35\tiny{$\pm$0.52} & 34.12\tiny{$\pm$0.52} &  16.57\tiny{$\pm$0.10}   &71.36\tiny{$\pm$0.19} &  9.27\tiny{$\pm$0.17} & 70.23\tiny{$\pm$0.10}  & 21.01\tiny{$\pm$0.13}&  31.58\tiny{$\pm$0.36}  & 76.10\tiny{$\pm$0.16}&  13.39\tiny{$\pm$0.76}  & 76.92\tiny{$\pm$0.20} &  38.81\tiny{$\pm$0.48} &  21.06\tiny{$\pm$0.10}\\
\bottomrule

\end{tabular}%
}
\end{minipage}
\end{table*}

\begin{table*}[ht]
\centering
\begin{minipage}{\textwidth}
\caption{Heterogeneous models (Gemma2, Aloe, Llama3.1, and BioMistral) with private contexts on the PubMedQA dataset, with questions that all models can answer correctly filtered out. For each question, we randomly assign one model with a relevant context (left columns) or two models, each with a relevant and an irrelevant context (right columns).}\label{tab:hetero_subset}
 \resizebox{1.0\linewidth}{!}{%
\begin{tabular}{l
                    *{6}{c} 
                   *{6}{c}}
\toprule
& \multicolumn{6}{c}{\textbf{1 Relevant Context}} & \multicolumn{6}{c}{\textbf{1 Relevant Context \& 1 Irrelavent Context}} \\
\cmidrule(lr){2-7} \cmidrule(lr){8-13} 
\textbf{Method} & AUC $\uparrow$ & ACC $\uparrow$ & ECE $\downarrow$  & MRR $\uparrow$ & K-Tau $\uparrow$ & DRegret $\downarrow$ & AUC $\uparrow$ & ACC $\uparrow$ & ECE $\downarrow$  & MRR $\uparrow$ & K-Tau $\uparrow$ & DRegret $\downarrow$ \\
\midrule
UniformAvg &  62.26  &   67.04 &  5.93   & 61.04  &  0.00  &   29.64   &  65.29  &  68.43 &   7.61   & 65.47   & 0.00   & 31.42\\
SelfCertainty &  64.02  &  68.15  &  9.67  &  80.56  &  25.57   & 29.57  &  63.05 &   64.45  & 16.27 &   79.63  &  25.50  &  38.47 \\
PackLLM & 73.28  & 77.04 &11.34 & 69.89 &22.08& 20.68 & 66.20  &  71.43 &  9.66  & 59.43  & 7.04  & 29.60 \\
\midrule
RouterDC w/ context & 62.16\tiny{$\pm$1.03} & 67.18\tiny{$\pm$1.76} & 6.04\tiny{$\pm$2.34} & 54.05\tiny{$\pm$21.81} & 7.18\tiny{$\pm$13.04} & 29.58\tiny{$\pm$0.83} &  65.03\tiny{$\pm$1.62}  & 68.80\tiny{$\pm$3.66} &  8.11\tiny{$\pm$4.30}   & 50.04\tiny{$\pm$1.33} &   -0.76\tiny{$\pm$14.62}  &  31.39\tiny{$\pm$0.78} \\
NIRTRouter w/ context  &  62.06\tiny{$\pm$2.79}  &  67.69\tiny{$\pm$1.18}  &  6.56\tiny{$\pm$8.74}   & 46.88\tiny{$\pm$40.59}  &  0.07\tiny{$\pm$83.43} &   29.39\tiny{$\pm$2.96}  & 65.08\tiny{$\pm$1.26}  & 69.42\tiny{$\pm$0.00}  & 8.77\tiny{$\pm$3.77}  & 45.02\tiny{$\pm$68.41}  & -2.09\tiny{$\pm$81.12}  & 31.17\tiny{$\pm$1.96}  \\
RouteLLM w/ context  & 70.38\tiny{$\pm$0.62} & 70.75\tiny{$\pm$0.82}  & 2.97\tiny{$\pm$0.68} & 62.28\tiny{$\pm$1.39}  & 19.27\tiny{$\pm$0.41} & 23.45\tiny{$\pm$0.67} & 70.39\tiny{$\pm$0.59} & 72.52\tiny{$\pm$0.80}& 
3.13\tiny{$\pm$0.81} & 63.46\tiny{$\pm$2.28} & 17.10\tiny{$\pm$1.38}& 
25.22\tiny{$\pm$0.57} 
\\
RouteLLM w/o context &  62.54\tiny{$\pm$11.96} &  66.78\tiny{$\pm$0.29} & 2.74\tiny{$\pm$1.75}  &54.72\tiny{$\pm$32.92} & 7.50\tiny{$\pm$6.67}  &28.87\tiny{$\pm$3.44} &  64.62\tiny{$\pm$21.30}  & 69.93\tiny{$\pm$4.44} &  2.90\tiny{$\pm$1.45}  & 61.40\tiny{$\pm$4.40}  & 12.26\tiny{$\pm$14.45} & 28.98\tiny{$\pm$10.12} \\
StackedGen &  78.87\tiny{$\pm$0.39}  & 78.94\tiny{$\pm$0.25}    &  4.54\tiny{$\pm$0.31}    &  73.21\tiny{$\pm$0.22} & 29.97\tiny{$\pm$0.57}   & 14.51\tiny{$\pm$0.06}  &    77.62\tiny{$\pm$0.76}    &   79.44\tiny{$\pm$2.26}    &  4.14\tiny{$\pm$1.87}  &   72.15\tiny{$\pm$1.23}  & 25.05\tiny{$\pm$2.25}   & 16.65\tiny{$\pm$1.21}   \\
\midrule
 \proj~I &   73.49\tiny{$\pm$0.87}   &  74.95\tiny{$\pm$0.26}   &  7.64\tiny{$\pm$0.59}   & 73.65\tiny{$\pm$0.30}  &  24.48\tiny{$\pm$0.63}   &   20.75\tiny{$\pm$0.20}   &    73.50\tiny{$\pm$0.43}  &    76.51\tiny{$\pm$0.29}   &   9.52\tiny{$\pm$0.38}  &    75.21\tiny{$\pm$0.55}   &   25.33\tiny{$\pm$0.41}  &  23.09\tiny{$\pm$0.06}  \\
\proj~II &  73.68\tiny{$\pm$0.31}  &   75.07\tiny{$\pm$0.43}  &    7.31\tiny{$\pm$0.68}     &  73.43\tiny{$\pm$0.07} &   23.53\tiny{$\pm$0.64}   &   20.45\tiny{$\pm$0.14}   &   72.93\tiny{$\pm$0.05}   &   76.18\tiny{$\pm$0.29}  &   8.45\tiny{$\pm$0.47}   &   75.29\tiny{$\pm$0.59}    &   22.84\tiny{$\pm$0.58}   &  23.05\tiny{$\pm$0.08}     \\
\proj~I (log-pooling) & 72.53\tiny{$\pm$0.76}  & 74.24\tiny{$\pm$0.23}  & 3.36\tiny{$\pm$0.86}  & 73.59\tiny{$\pm$0.27}  & 24.61\tiny{$\pm$0.72}  & 20.32\tiny{$\pm$0.36}  &  73.30\tiny{$\pm$0.48}  &   76.57\tiny{$\pm$0.23}  &       9.60\tiny{$\pm$0.39} &  75.22\tiny{$\pm$0.59}  & 25.37\tiny{$\pm$0.24}  &  23.12\tiny{$\pm$0.13} \\
\proj~I (calibrated) &  83.43\tiny{$\pm$0.32}  & 80.65\tiny{$\pm$0.20} &  10.01\tiny{$\pm$0.24}  & 78.36\tiny{$\pm$0.10} &  30.74\tiny{$\pm$0.30} &  15.18\tiny{$\pm$0.08} &  83.86\tiny{$\pm$0.28} &  84.08\tiny{$\pm$0.06} & 10.56\tiny{$\pm$0.13}  & 80.52\tiny{$\pm$0.19} &  32.92\tiny{$\pm$0.31} & 13.40\tiny{$\pm$0.10} \\
\bottomrule

\end{tabular}%
}
\end{minipage}
\end{table*}

\begin{figure}[]
    \centering
    \begin{subfigure}[b]{0.49\textwidth}
        \centering
        \includegraphics[width=\textwidth]{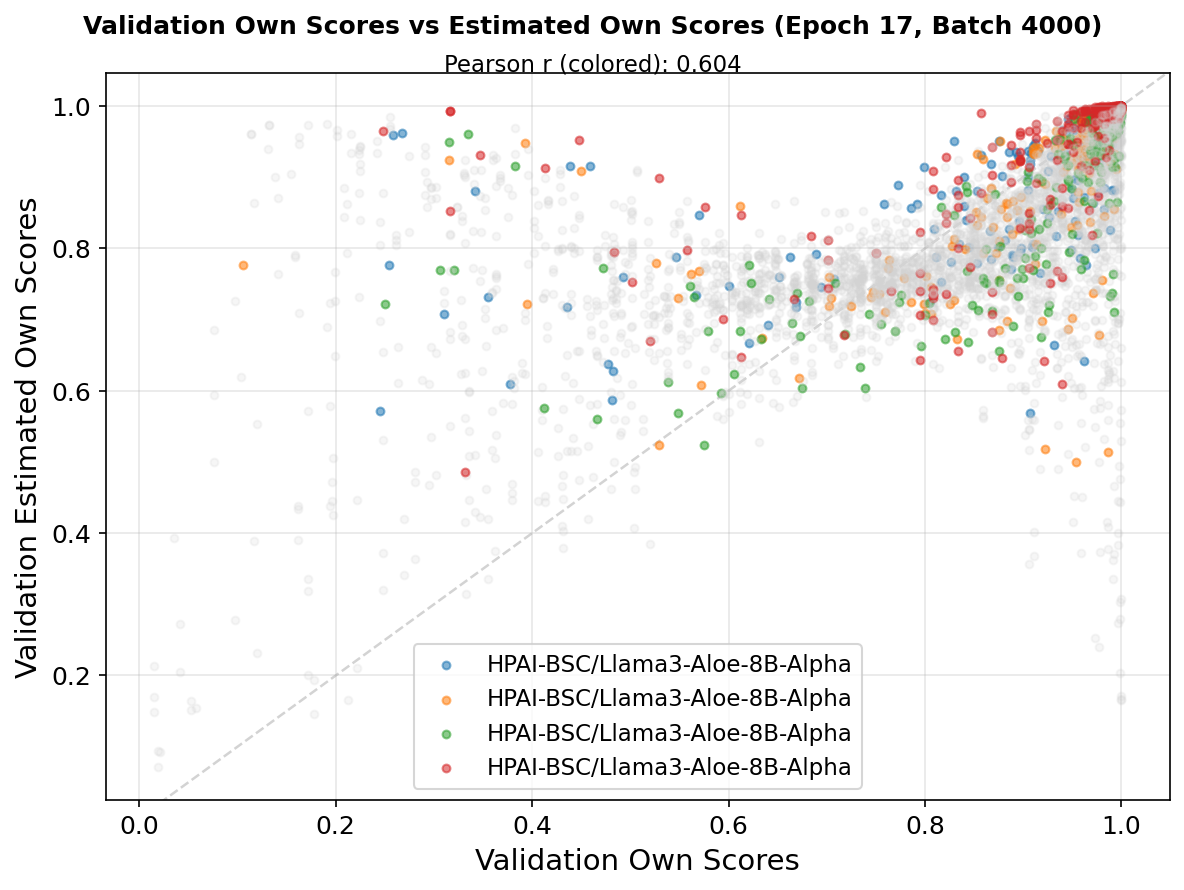}
    \end{subfigure}
    \hfill 
    \begin{subfigure}[b]{0.49\textwidth}
        \centering
        \includegraphics[width=\textwidth]{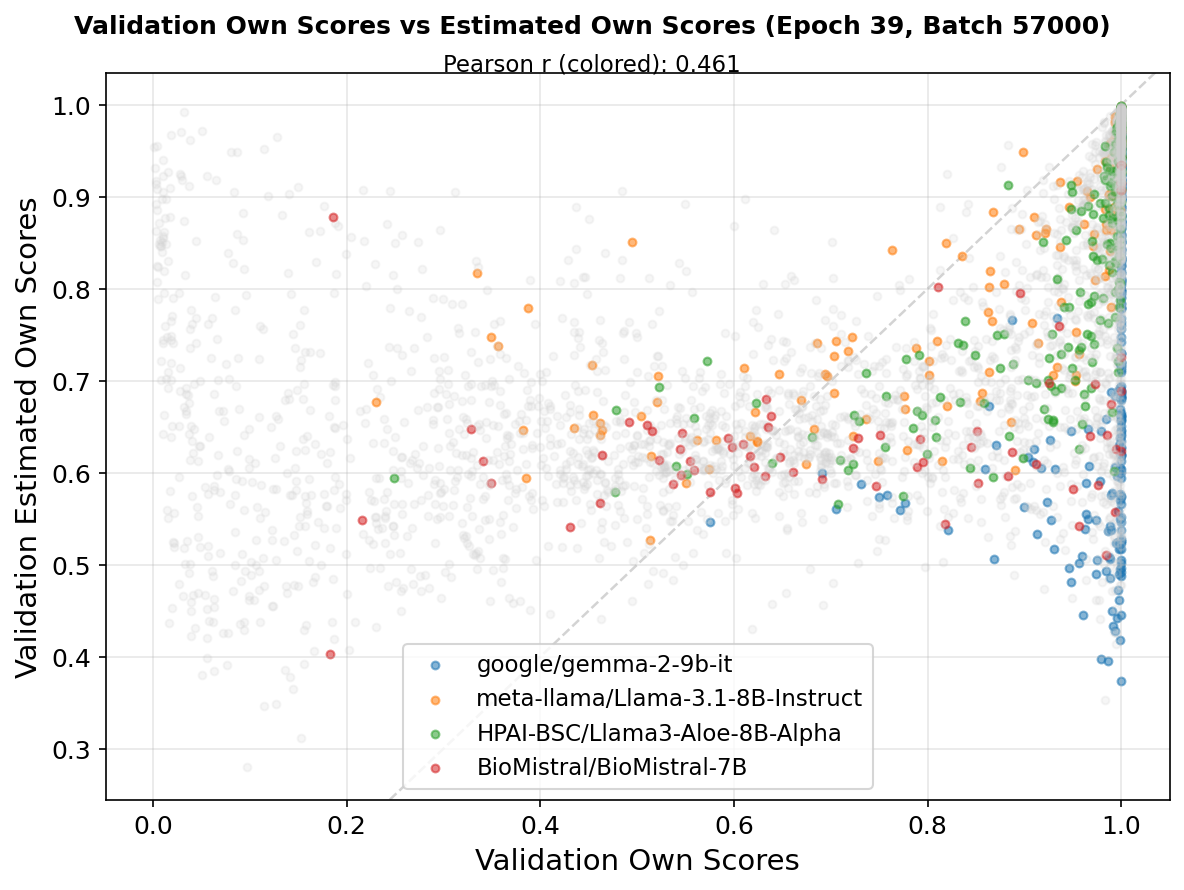}
    \end{subfigure}
    \hfill
    \begin{subfigure}[b]{0.49\textwidth}
        \centering
        \includegraphics[width=\textwidth]{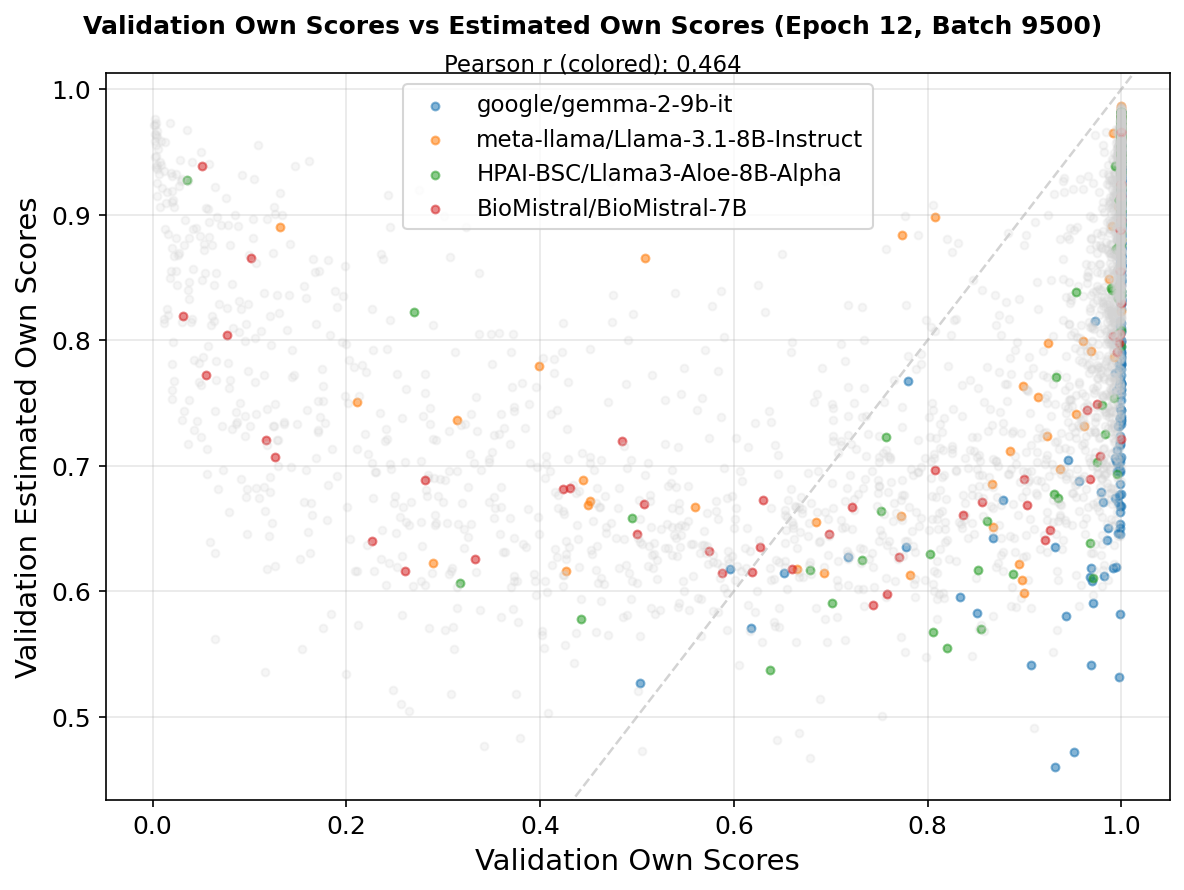}
    \end{subfigure}
    \hfill
    \begin{subfigure}[b]{0.49\textwidth}
        \centering
        \includegraphics[width=\textwidth]{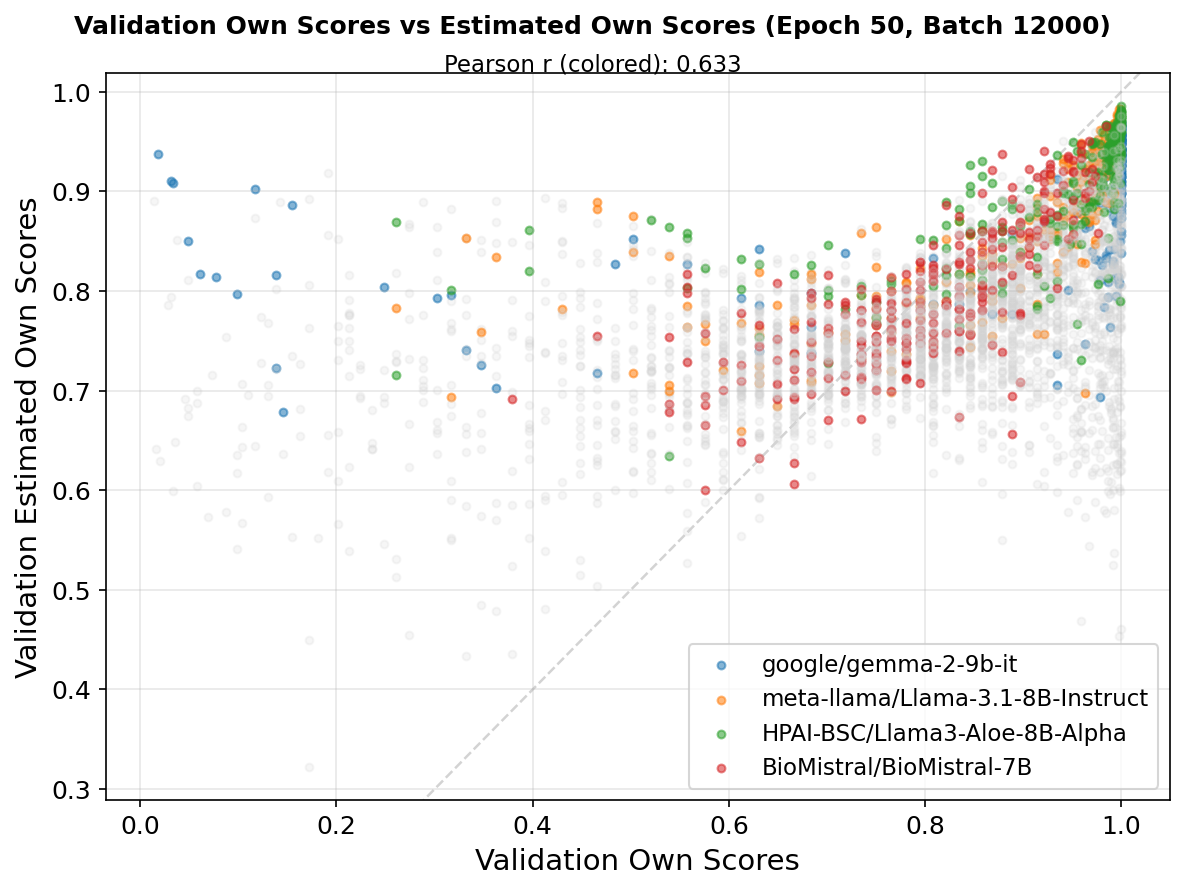}
    \end{subfigure}
    \caption{The alignment between the actual Brier scores ($x$-axis) and the predicted Brier scores trained on last-layer hidden states.
    }
    \label{fig:own_score}
\end{figure}

\begin{figure}[]
    \centering
    \begin{subfigure}[b]{0.49\textwidth}
        \centering
        \includegraphics[width=\textwidth]{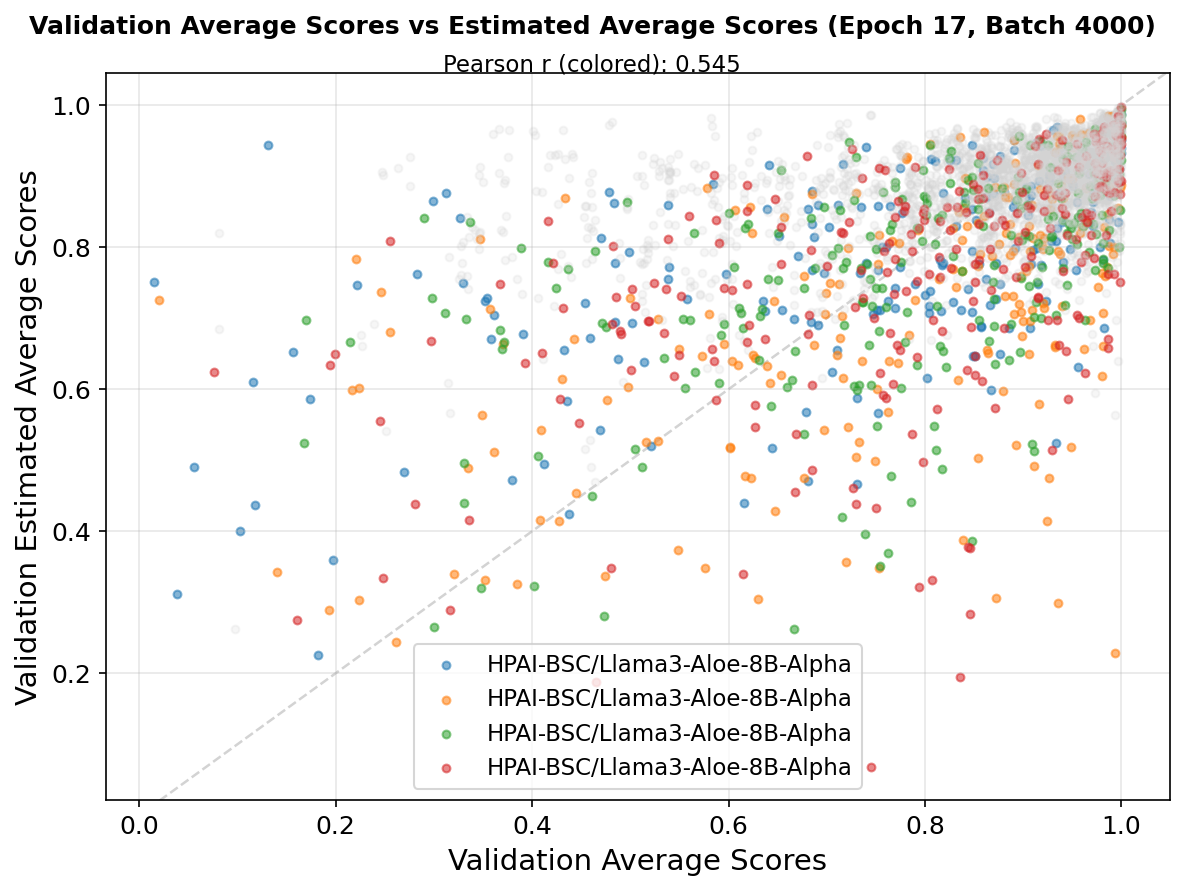}
    \end{subfigure}
    \hfill 
    \begin{subfigure}[b]{0.49\textwidth}
        \centering
        \includegraphics[width=\textwidth]{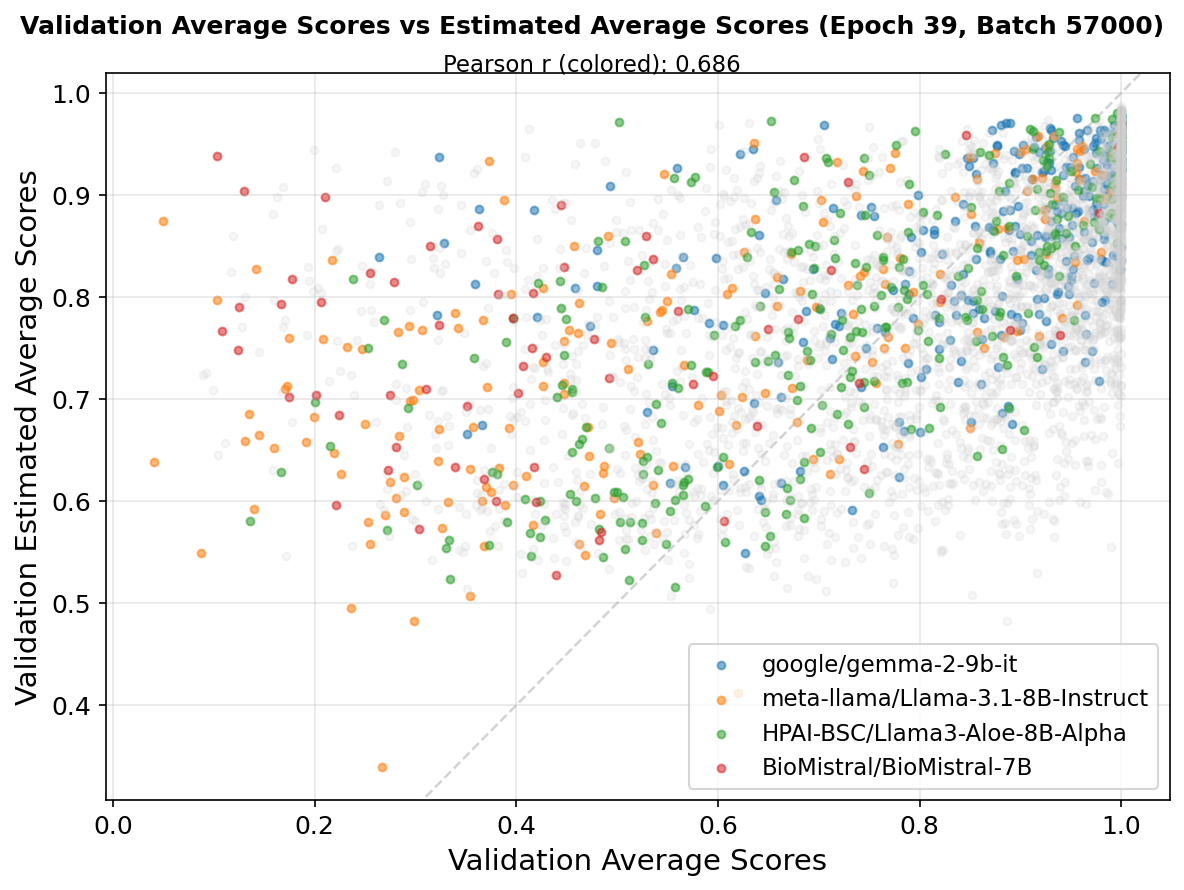}
    \end{subfigure}
    \hfill
    \begin{subfigure}[b]{0.49\textwidth}
        \centering
        \includegraphics[width=\textwidth]{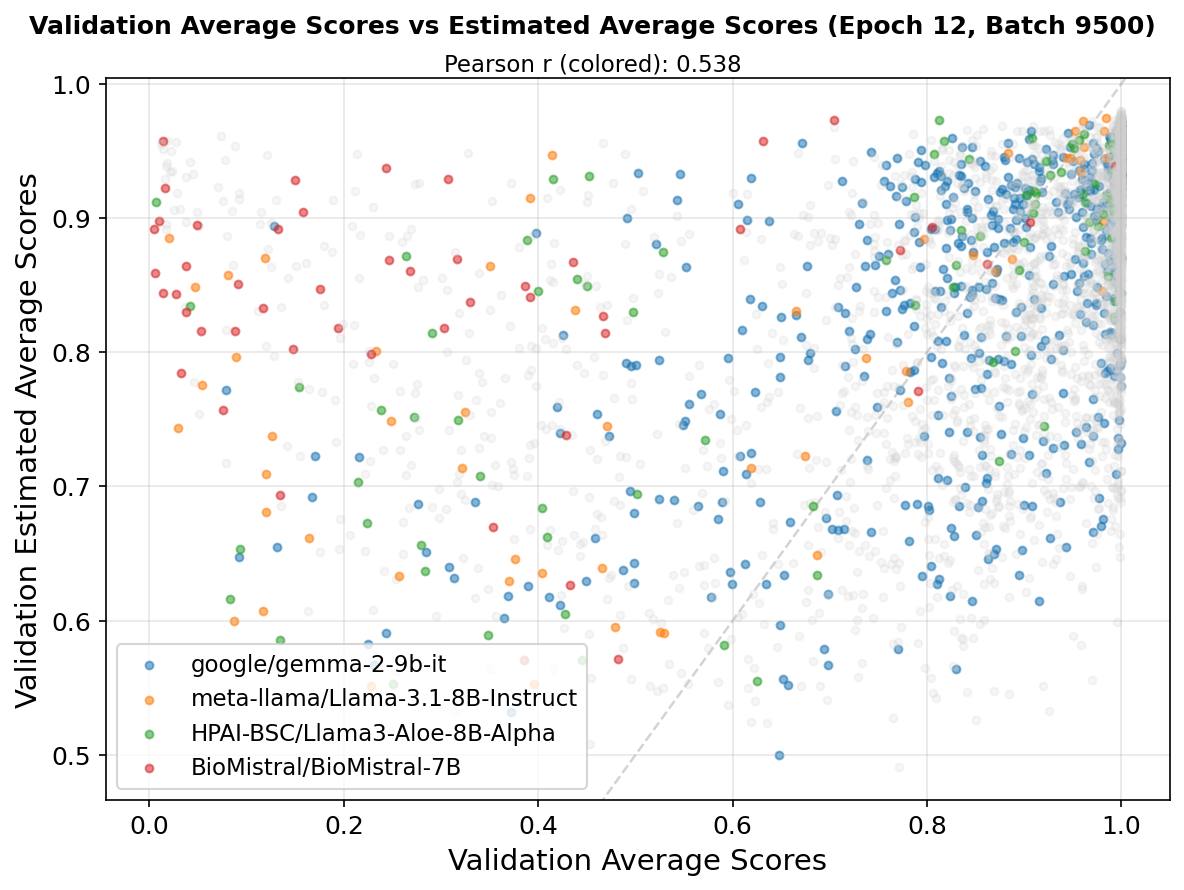}
    \end{subfigure}
    \hfill
    \begin{subfigure}[b]{0.49\textwidth}
        \centering
        \includegraphics[width=\textwidth]{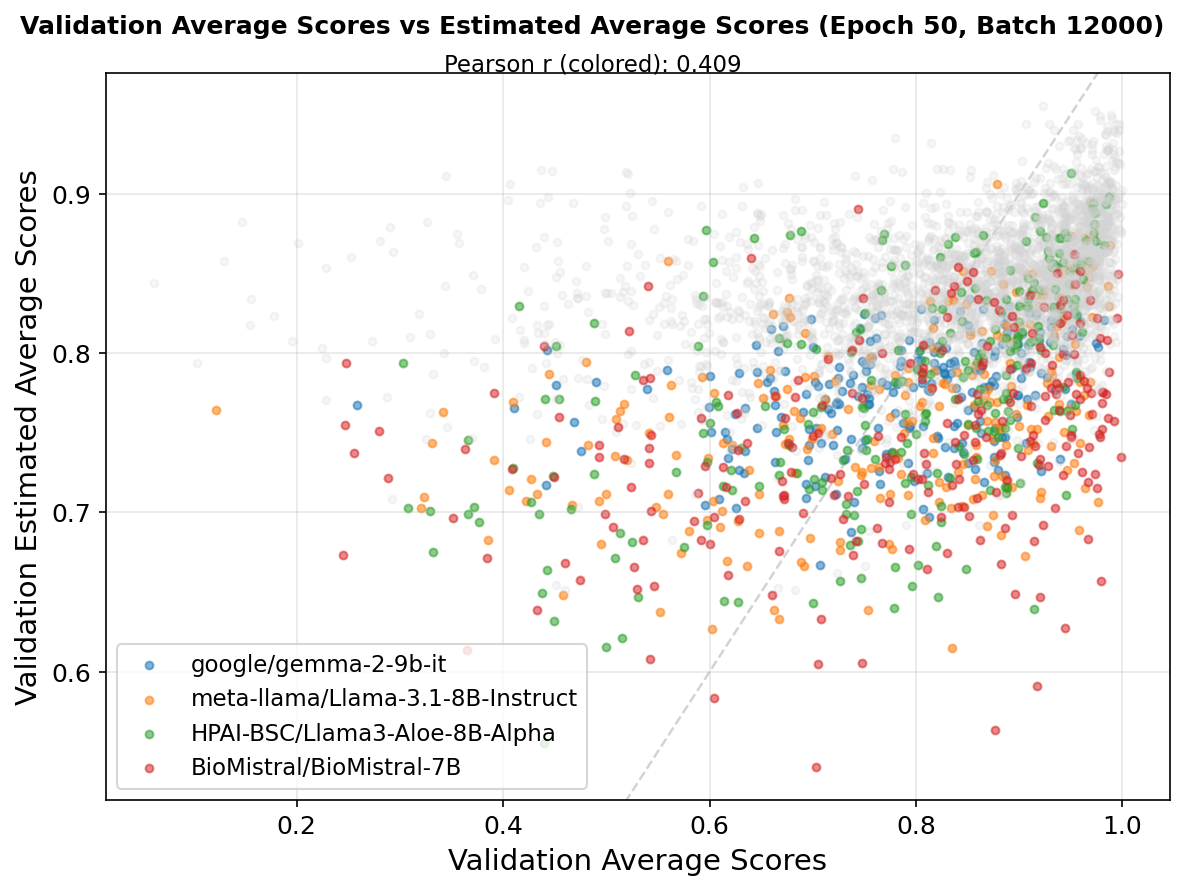}
    \end{subfigure}
    \caption{The alignment between the actual baseline average scores ($b_{-i}$ ($x$-axis) and the predicted baseline average scores trained on last-layer hidden states.}
    \label{fig:baseline}
\end{figure}

\begin{figure}[]
    \centering
    \begin{subfigure}[b]{0.49\textwidth}
        \centering
        \includegraphics[width=\textwidth]{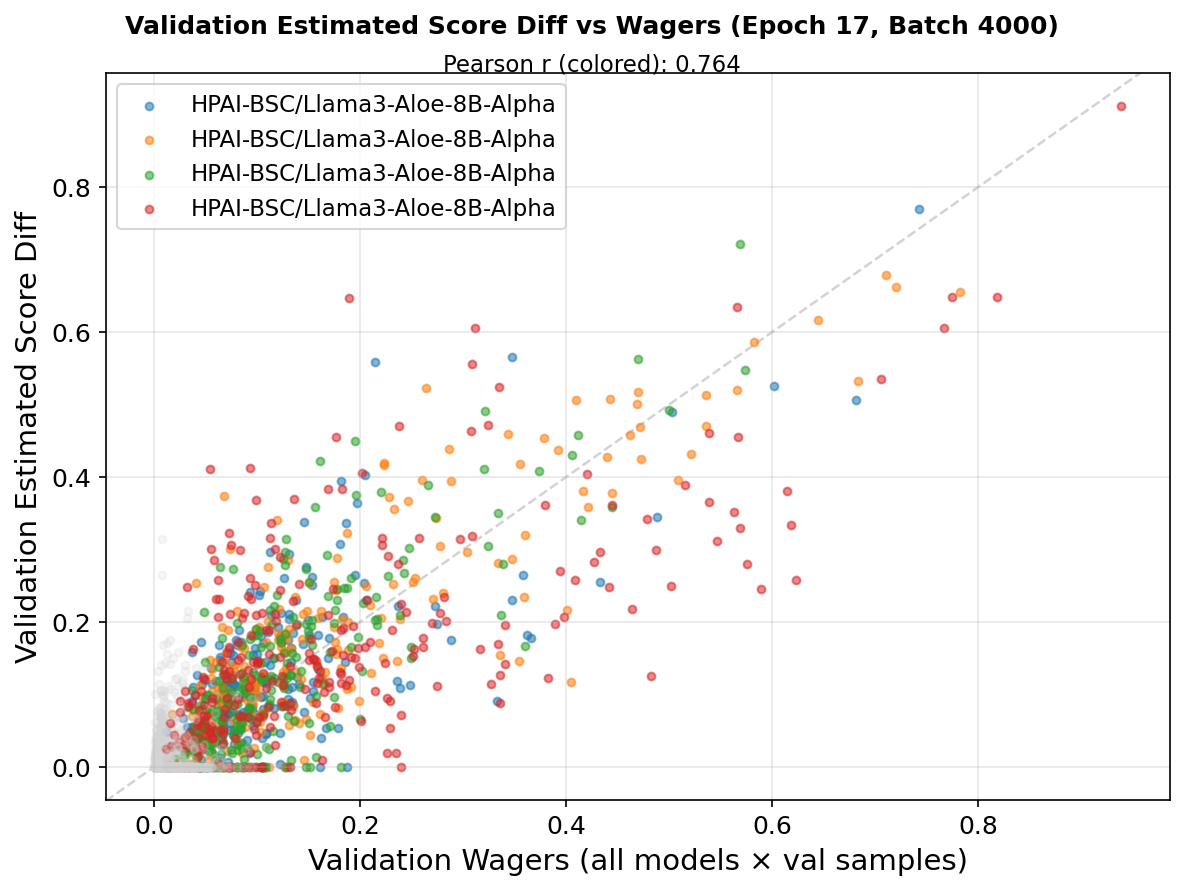}
    \end{subfigure}
    \hfill 
    \begin{subfigure}[b]{0.49\textwidth}
        \centering
        \includegraphics[width=\textwidth]{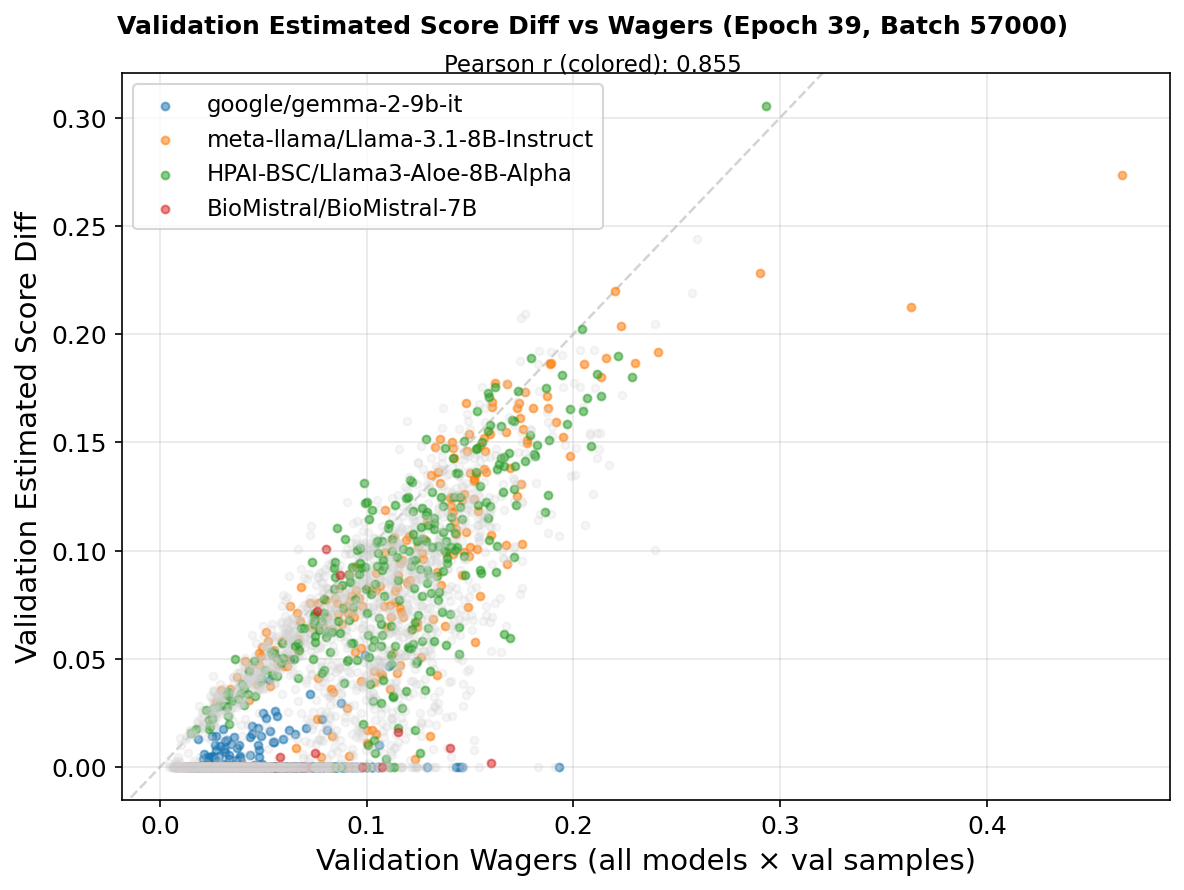}
    \end{subfigure}
    \hfill
    \begin{subfigure}[b]{0.49\textwidth}
        \centering
        \includegraphics[width=\textwidth]{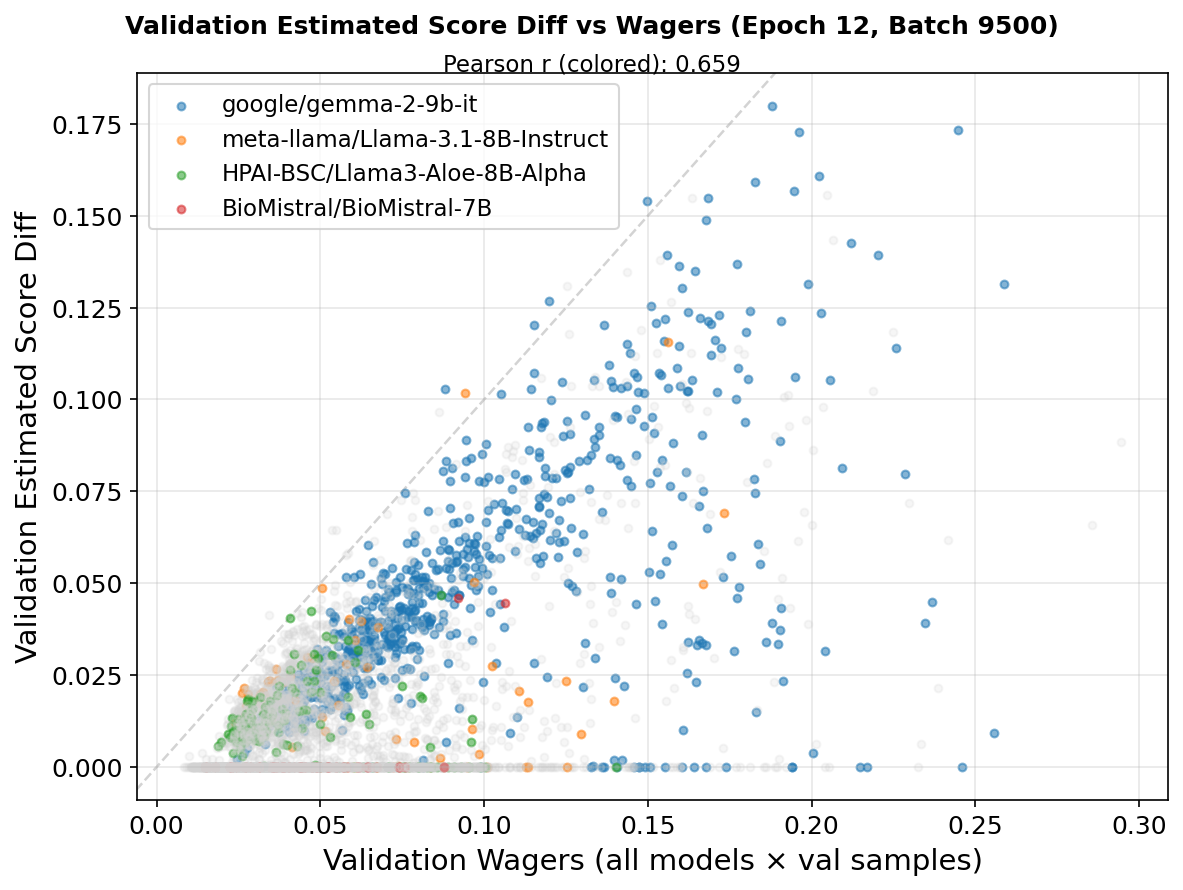}
    \end{subfigure}
    \hfill
    \begin{subfigure}[b]{0.49\textwidth}
        \centering
        \includegraphics[width=\textwidth]{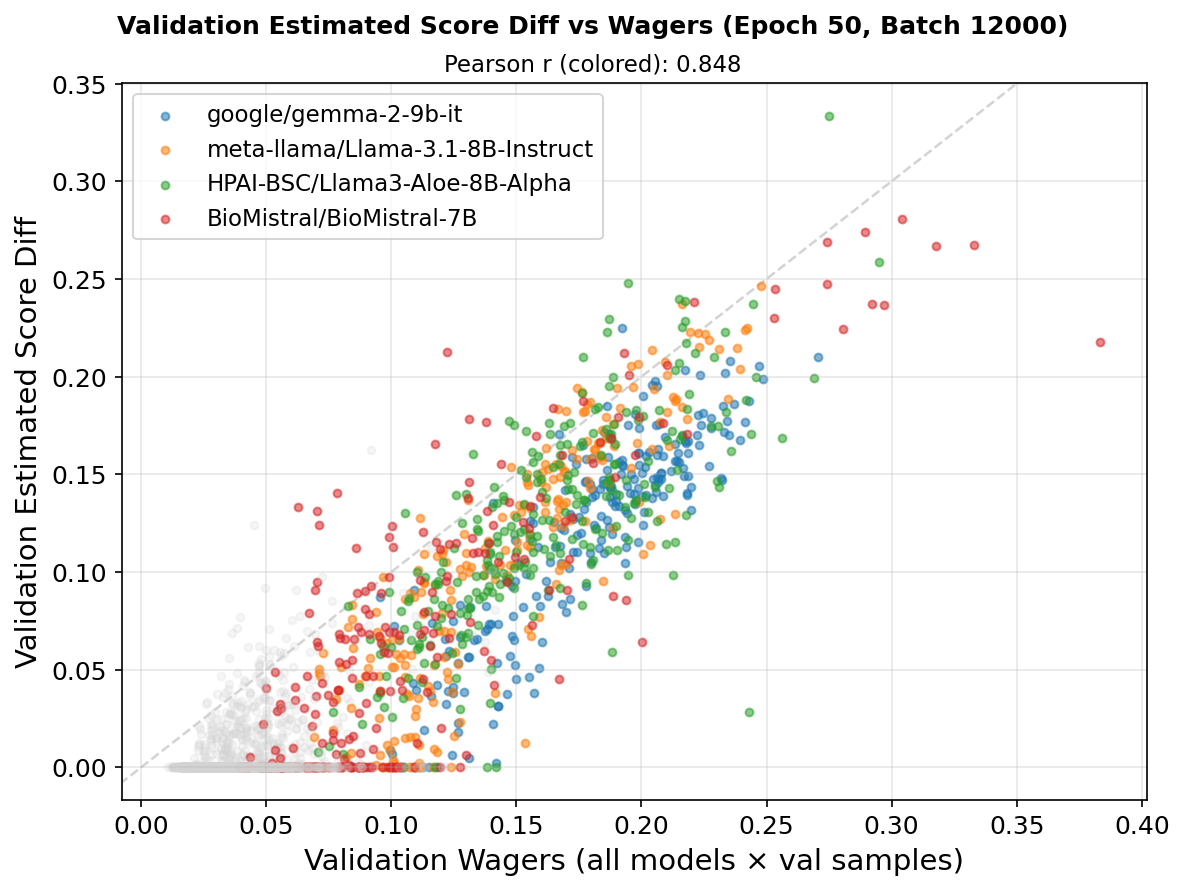}
    \end{subfigure}
    \caption{The alignment between the wagers ($x$-axis) and the differences between the predicted own scores and the predicted baseline average scores ($y$-axis), which can be interpreted as the best-response wager under one's belief. 
      The alignment is strong even though the wagers, the predicted own scores, and the predicted baseline average scores are trained with separate heads.}
    \label{fig:estimated_alignment}
\end{figure}

\begin{figure}[htbp]
    \centering
    \begin{subfigure}[b]{0.32\textwidth}
        \centering
        \includegraphics[width=\textwidth]{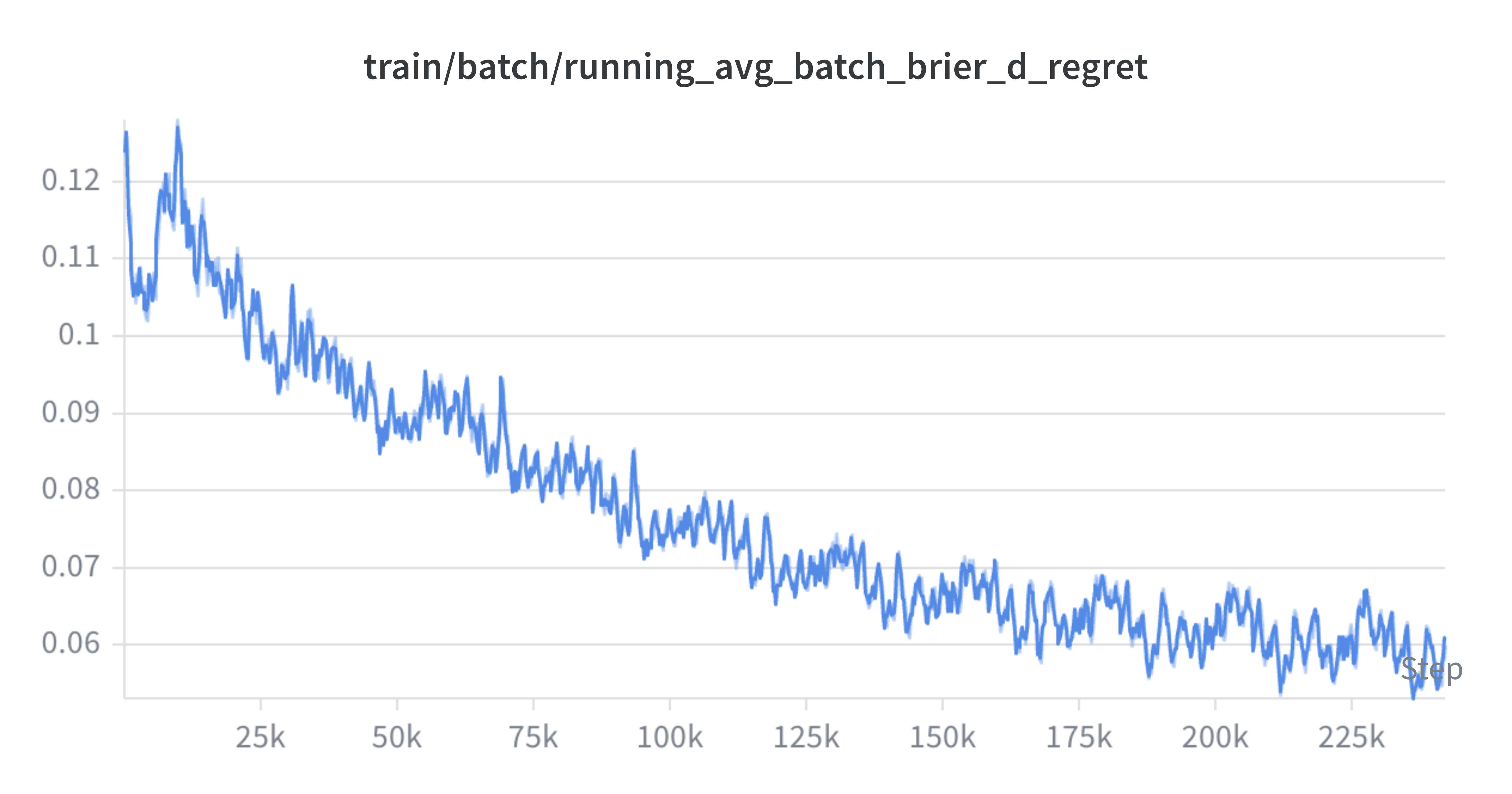}
    \end{subfigure}
    \hfill 
    \begin{subfigure}[b]{0.32\textwidth}
        \centering
        \includegraphics[width=\textwidth]{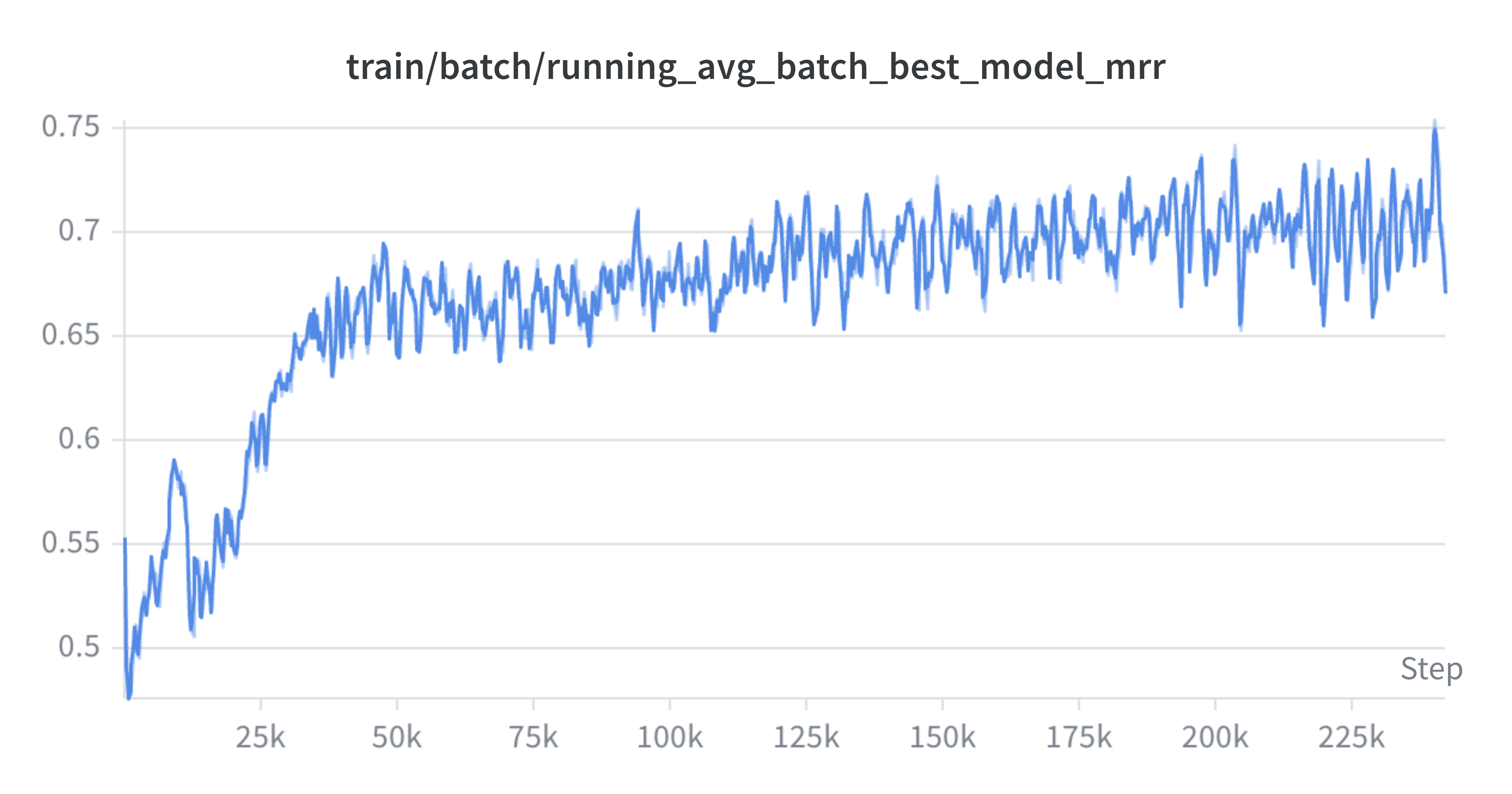}
    \end{subfigure}
    \hfill
    \hfill
    \begin{subfigure}[b]{0.32\textwidth}
        \centering
        \includegraphics[width=\textwidth]{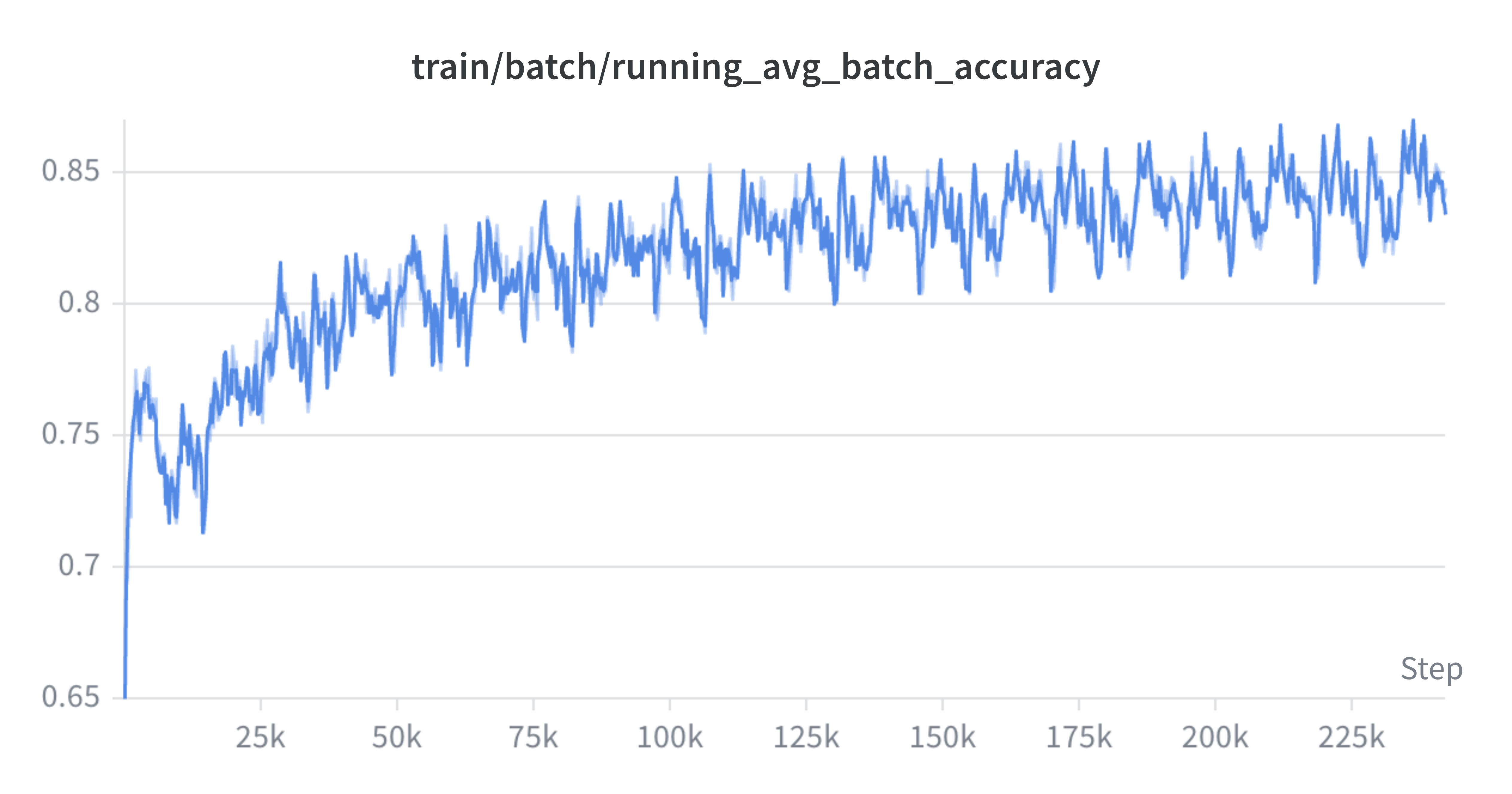}
    \end{subfigure}
    
    \caption{Dynamic regret, mean reciprocal rank of the best expert (MRR), and accuracy (ACC) over time of \proj under Scenario I on PubMedQA. They demonstrate that \proj's performance gradually improves over time, though in a decentralized learning environment. A similar trend can be observed in other scenarios or datasets.}
    \label{fig:perf_overtime}
\end{figure}


\end{document}